%% file: arxiv_v2.tex
\documentclass[a4paper,11pt]{article}
\input{assets/import_packages}
\input{assets/settings_custom}

\usepackage[toc,page,header]{appendix}
\usepackage{minitoc}

\title{Generalization error in high-dimensional perceptrons:\\
  Approaching Bayes error with convex optimization}
\author{Benjamin Aubin$^{\dagger}$, Florent Krzakala$^{\star}$, Yue M. Lu$^{\circ}$, Lenka Zdeborov{\'a}$^{*}$
}
\date{
$^\dagger$ \textit{Universit\'e Paris-Saclay, CNRS, CEA,\\
Institut de physique th\'eorique, 91191, Gif-sur-Yvette, France.}\spacecase
$^\star$ \textit{IdePHICS laboratory, {\'E}cole Polytechnique F{\'e}d{\'e}rale de Lausanne\\
1015, Lausanne, Switzerland }\spacecase
$^\circ$ \textit{John A. Paulson School of Engineering and Applied Sciences,\\ Harvard University, Cambridge, MA 02138,  USA}\spacecase
$^*$ \textit{SPOC laboratory, {\'E}cole Polytechnique F{\'e}d{\'e}rale de Lausanne\\
1015, Lausanne, Switzerland}\spacecase
}

\newcommand{\ndim}{d}
\newcommand{\nsamples}{n}
\newcommand{\bayes}{{\rm b}}

\renewcommand \partname{}

\begin{document}
\maketitle
\makeatletter
\def\l@subsubsection#1#2{}
\makeatother

\doparttoc 
\faketableofcontents 

\begin{abstract}
  We consider a commonly studied supervised classification of
  a synthetic dataset whose labels are generated by feeding a one-layer neural network with random \iid inputs. 
We study the generalization performances of standard
  classifiers in the high-dimensional regime
  where $\alpha={\nsamples}/{\ndim}$ is kept finite in the limit
  of a high dimension $\ndim$ and number of samples $\nsamples$. 
Our contribution is three-fold: First, we prove
  a formula for the generalization error achieved by $\rL_2$ regularized
  classifiers that minimize a convex loss.
  This formula was first
  obtained by the heuristic replica
  method of statistical physics.
 Secondly, focussing on commonly used loss functions and optimizing
 the $\rL_2$ regularization strength, we
  observe that while ridge regression performance is
  poor, logistic and hinge regression are surprisingly able to
  approach the Bayes-optimal generalization error extremely closely. 
  As $\alpha \to \infty$ they lead to Bayes-optimal rates, a fact that does
  not follow from predictions of margin-based generalization error bounds. 
Third, we design an optimal loss and regularizer that provably leads to Bayes-optimal
  generalization error. 
\end{abstract}

\newpage

\section{Introduction}
\input{files/introduction.tex}

\section{Main technical results}
\input{files/technical.tex}

\section{Generalization errors}
\input{files/applications.tex}

\section{Reaching Bayes optimality}
\input{files/optimality.tex}

\newpage
\section*{Acknowledgments}
\input{files/acknowledgment.tex}

\bibliographystyle{unsrt}
\bibliography{assets/refs}

\newpage
\appendix
\input{appendix.tex}

\end{document}

%% file: assets/import_packages.tex
\usepackage{authblk}
\usepackage[british]{babel}
\usepackage[mono=false]{libertine}
\usepackage{setspace}

\usepackage{cite}
\usepackage[utf8]{inputenc} 
\usepackage[T1]{fontenc}    
\usepackage{url}            
\usepackage{booktabs}       
\usepackage{amsfonts}       
\usepackage{nicefrac}       
\usepackage{microtype}      
\usepackage{mathtools}
\usepackage{amsmath,amsfonts,amssymb,amsthm,mathrsfs,bbm}
\usepackage{latexsym,amscd,amsbsy,dsfont}
\usepackage{mathtools}
\usepackage{dsfont}
\usepackage{algorithmic}
\usepackage{algorithm2e}
\usepackage{color}
\usepackage[usenames, dvipsnames]{xcolor}
\usepackage{graphicx}
\usepackage{epstopdf}
\usepackage{float}
\usepackage{cases}
\usepackage{geometry}
\usepackage{caption}
\usepackage{physics}
\usepackage{tensor}
\usepackage{tikz}
\usepackage{subcaption}
\usepackage{xr-hyper}
\usepackage[pdfpagemode=UseNone,bookmarksopen=false,colorlinks=true,urlcolor=RoyalBlue,citecolor=YellowOrange,citebordercolor=RoyalBlue,linkcolor=RoyalBlue]{hyperref}
\usetikzlibrary{arrows}
\usetikzlibrary{decorations.markings}
\usetikzlibrary{fit}
\usepackage[most]{tcolorbox}
\tcbuselibrary{theorems}
\usepackage{empheq}
\usetikzlibrary{bayesnet}
\usepackage{bookmark}

%% file: assets/settings_custom.tex
\def \({\left(}
\def \){\right)}
\def \[{\left[}
\def \]{\right]}

\newcommand{\bw}{{\textbf {w}}}

\newcommand{\bc}{{\textbf {c}}}

\newcommand{\bgamma}{{\boldsymbol{\gamma}}}
\newcommand{\bomega}{{\boldsymbol{\omega}}}

\renewcommand{\d}{\text{d}}
\newcommand{\e}{\text {e}}

\newcommand{\eq}{\text{ eq}}

\newcommand{\be}{\begin{equation}}
\newcommand{\ee}{\end{equation}}
\newcommand{\beqa}{\begin{eqnarray}}
\newcommand{\eeqa}{\end{eqnarray}}

\newcommand{\bea}{\begin{align}}
\newcommand{\eea}{\end{align}}

\newtheorem{theorem}{Theorem}[section]

\newtheorem{remark}[theorem]{\textbf{Remark}}
\newtheorem{proposition}[theorem]{\textbf{Proposition}}
\newtheorem{corollary}[theorem]{\textbf{Corollary}}

\DeclareMathAlphabet{\varmathbb}{U}{bbold}{m}{n}
\newcommand{\id}{\mathds{1}}
\newcommand{\EE}{\mathbb{E}}
\newcommand{\bbR}{\mathbb{R}}
\newcommand{\bbP}{\mathbb{P}}

\newcommand{\bbN}{\mathbb{N}}

\renewcommand{\d}{{\rm d}}

\newcommand{\mZ}{\mathcal{Z}}

\newcommand{\mI}{\mathcal{I}}
\newcommand{\mM}{\mathcal{M}}
\newcommand{\mN}{\mathcal{N}}

\newcommand{\mD}{\mathcal{D}}

\newcommand{\mL}{\mathcal{L}}

\newcommand{\mR}{\mathcal{R}}

\newcommand{\mP}{\mathcal{P}}	
	
\renewcommand{\tr}[1]{\textrm{Tr}\[#1\]}
\newcommand{\td}[1]{{\tilde{#1}}}

\newcommand{\spacecase}[0]{\vspace{0.3cm} \\}
\newcommand{\hhspace}[0]{\hspace{0.3cm}}

\renewcommand{\subparagraph}[1]{\vspace{0.5cm} $\bullet$ \textit{#1} \\ \vspace{0.2cm} }
\newcommand{\andcase}[0]{\hspace{ 0.2cm }\textrm{ and }\hspace{ 0.2cm }}

\definecolor{green}{RGB}{0, 153, 0}
\definecolor{light_blue}{RGB}{51, 153, 255}
\definecolor{orange}{RGB}{255, 204, 0}
\definecolor{bg}{RGB}{0, 153, 153}
\definecolor{blue}{RGB}{0, 102, 204}
\definecolor{red}{RGB}{204, 0, 0}
\definecolor{lg}{RGB}{214, 214, 214}

\renewcommand{\vec}[1]{{\textbf{#1}}}
\newcommand{\mat}[1]{{\textrm{#1}}}
\newcommand{\tbf}[1]{{\bold{#1}}}

\newcommand{\sign}{{\textrm{sign}}}

\newcommand{\out}[0]{{\rm out}}
\newcommand{\w}[0]{{\rm{w}}}

\newcommand{\iid}[0]{\rm iid~}

\newcommand{\rI}[0]{{\rm{I}}}
\newcommand{\rL}[0]{{\ell}}

\newcommand{\underlim}[2]{\underset{#1 \to #2}{\longrightarrow}}

\newcommand{\extr}{{\textbf{extr}}}

\newcommand{\argmin}{{\rm{argmin}}}

\definecolor{codegreen}{rgb}{0,0.6,0}
\definecolor{codegray}{rgb}{0.5,0.5,0.5}
\definecolor{codepurple}{rgb}{0.58,0,0.82}
\definecolor{backcolour}{rgb}{0.95,0.95,0.92}
\lstdefinestyle{mystyle}{
    backgroundcolor=\color{backcolour},   
    stringstyle=\color{codepurple},
    commentstyle=\color{codegreen},
    numberstyle=\tiny\color{codegray},
    basicstyle=\ttfamily\footnotesize,
    numbers=left,      
}

%% file: files/introduction.tex
\label{sec:introduction}
High-dimensional statistics, where the ratio $\alpha={\nsamples}/{\ndim}$ is kept finite while the dimensionality $\ndim$ and the number of samples $\nsamples$ grow, often display interesting non-intuitive features. Asymptotic generalization performances for such problems in the so-called \emph{teacher-student} setting, with synthetic data, have been the subject of intense investigations spanning many decades \cite{seung1992statistical,watkin1993statistical,engel2001statistical,bayati2011lasso,el2013robust,donoho2016high}.
To understand the effectiveness of modern machine learning techniques, and also the limitations of the classical statistical learning approaches \cite{zhang2016understanding,belkin2019reconciling}, it is of interest to revisit this line of research. Indeed, this direction is currently the subject to a renewal of interests, as testified by some very recent, yet already rather influential papers \cite{Candes18,Barbier5451,Hastie19,belkin2019two,mei2019generalization}. The present paper subscribes to this line of work and studies high-dimensional classification within one of the simplest models considered in statistics and machine learning: convex linear estimation with data generated by a teacher \emph{perceptron} \cite{gardner1989three}. We will focus on the generalization abilities in this problem, and compare the performances of Bayes-optimal estimation to the more standard \emph{Empirical Risk Minimization} (ERM). We then compare the results with the prediction of standard generalization bounds that illustrate in particular their limitation even in this simple, yet non-trivial, setting.

\paragraph{Synthetic data model ---} We consider a supervised machine learning task, whose dataset is generated by a single layer neural network, often named a \emph{teacher} \cite{seung1992statistical,watkin1993statistical,engel2001statistical}, that belongs to the Generalized Linear Model (GLM) class. Therefore, we assume the $\nsamples$ samples are generated according to
\begin{align}
	\vec{y} = \varphi_{\out^\star}\(\frac{1}{\sqrt{\ndim}} \mat{X} \vec{w}^\star \) \qquad \Leftrightarrow \qquad \vec{y} \sim P_{\out^\star} \(. \big\vert \frac{1}{\sqrt{\ndim}} \mat{X} \vec{w}^\star \) \,,
\label{main:teacher_model}	
\end{align}
where $\vec{w}^\star \in \bbR^\ndim$ denotes the ground truth vector drawn from a probability distribution $P_{\w^\star}$ with second moment $\rho_{\ndim,\w^\star}\equiv \frac{1}{\ndim}~\EE\[ \|\vec{w}^\star\|_2^2 \]$ and $\varphi_{\out^\star}:\bbR \mapsto \bbR$ represents a component-wise deterministic or stochastic activation function equivalently associated to a distribution $P_{\out^\star}$. The input data matrix $\mat{X}=\(\vec{x}_\mu\)_{\mu=1}^\nsamples \in \bbR^{\nsamples \times \ndim}$ contains \iid Gaussian vectors, i.e $\forall \mu \in [1:\nsamples],~\vec{x}_\mu \sim \mN\(\vec{0},\mat{I}_\ndim\)$. Even though the framework we use and the theorems and results we derive are valid for a rather generic channel in eq.~\eqref{main:teacher_model} --- including regression problems ---
we will mainly focus the presentation on the commonly considered perceptron case: a binary classification task with data given by a $\sign$ activation function $\varphi_{\out^\star}\(\vec{z}\) = \sign\(\vec{z}\)$, with a Gaussian weight distribution $P_{\w^\star}(\vec{w}^\star)= \mN_{\vec{w}^\star}\(\vec{0}, \rho_{\ndim,\w^\star} \rI_{\ndim}\)$. The $\pm 1$ labels are thus generated as
\begin{align}
\vec{y} &= \sign\(\frac{1}{\sqrt{\ndim}} \mat{X}\vec{w}^\star \)\,, \hhspace \text{with} \hhspace \vec{w}^\star \sim \mN_{\vec{w}^\star}\(\vec{0}, \rho_{\ndim,\w^\star} \rI_{\ndim}\) .
\label{main:teacher_sign}	
\end{align}
This particular setting was extensively studied in the past \cite{seung1992statistical,opper1996statistical}
and is interesting in the sense it does not show a computational-to-statistical gap. 
Yet, our analysis and the set of equations presented in SM \ref{appendix:proof:equivalence_gordon_replicas_formulation_l2} \eq.~\eqref{appendix:fixed_point_replicas} are valid more generically to any other ground truth distributions $P_{\out^\star}$ and $P_{\w^\star}$. 
Finally, the isotropic Gaussian hypothesis of the input vectors $\mat{X}$ can be relaxed to non-isotropic Gaussian. 

\paragraph{Empirical Risk Minimization ---}
The workhorse of machine learning is Empirical Risk Minimization (ERM), where one minimizes a \emph{loss function} in the corresponding high-dimensional parameter space
$\bbR^{\ndim}$. To avoid overfitting of the training set one often adds a \emph{regularization term} $r$. 
ERM then corresponds to estimating $\hat{\vec{w}}_{\rm erm} =  \argmin_{\vec{w}} \[ \mL\(\vec{w}; \vec{y}, \mat{X}\) \]$ where the regularized training loss $\mL$ is defined by, using the notation $z_\mu \(\vec{w}, \vec{x}_\mu\) \equiv \frac{1}{\sqrt{\ndim}} \vec{x}_\mu^\intercal\vec{w}$,
\begin{align}
	 \mL\(\vec{w}; \vec{y}, \mat{X}\) = \sum_{\mu=1}^\nsamples l\(y_\mu, z_\mu\(\vec{w}, \vec{x}_\mu\) \) +   r\(\vec{w}\) \,.
	\label{main:training_loss}
\end{align}
The goal of the present paper is to discuss the generalization performance of these estimators for the classification task (\ref{main:teacher_sign}) in the high-dimensional limit.
We focus our analysis on commonly used loss functions $l$, namely the square $l^{\rm square}(y,z)=\frac{1}{2}(y-z)^2$, logistic $l^{\rm logistic}(y, z)=\log(1+\exp(-y z))$ and hinge losses $l^{\rm hinge}(y, z)=\max\(0,1-yz\)$.
We will mainly illustrate our results for the $\rL_2$ regularization $r\(\vec{w}\) = {\lambda} \|\vec{w}\|_2^2/2$, where we introduced a regularization strength hyper-parameter $\lambda$, even though a similar rigorous analysis can be conducted for any separable and convex regularizer. 

\paragraph{Related works ---} The above learning problem has been extensively studied in the statistical physics community using the heuristic replica method \cite{gardner1989three,seung1992statistical,watkin1993statistical,opper1996statistical,engel2001statistical}. Due to the interest in high-dimensional statistics, they have experienced a resurgence in popularity in recent years. In particular, rigorous works on related problems are much more recent. 
The authors of \cite{Barbier5451} established rigorously the replica-theory predictions for the Bayes-optimal generalization error. Here we focus on standard ERM estimation and compare it to the information theoretic baseline results obtained in \cite{Barbier5451}. 
Authors of \cite{Thrampoulidis16} analyzed rigorously M-estimators for the regression case where data are generated by a linear-activation teacher. Here we analyze classification with a more general and non-linear teacher, focusing in particular on the sign-teacher.  
The case of max-margin loss was studied in \cite{Montanari19} with a technically closely related proof, but with a focus on the over-parametrized regime, thus not addressing the questions that we focus on. A range of unregularized losses was also analyzed for a sigmoid teacher (that is very similar to a sign-teacher) again in the context of the double-descent behavior in \cite{deng2019model,kini2020analytic}. Here we focus instead on the regularized case as it drastically improves generalization performances of the ERM and that allows us to compare with the Bayes-optimal estimation as well as to standard generalization bounds.  
Our proof, as in the above mentioned works and \cite{Mignacco2020}, is based on Gordon's Gaussian Min-max inequalities \cite{gordon1985some,Thrampoulidis16}, including in particular the effect of the regularization. 
\paragraph{Main contributions ---} 
Our first main contribution is to provide, in Sec.~\ref{sec:fixed_point}, the rigorous high-dimensional asymptotics of the classification generalization performances of ERM with the loss given by (\ref{main:training_loss}), for any convex loss $l$ and a $\ell_2$ regularization. 
Note that for the sake of conciseness, we focus on this latter case, but the proof is performed in the more general regression case and can be easily extended to any convex separable regularization and to non-isotropic Gaussian inputs.
Additionally, we provide a proof of the equivalence between the results of our paper and the ones initially obtained by the replica method, which is of additional interest given the wide range of application of these heuristics statistical-physics techniques in machine learning and computer science \cite{mezard2009information,lenkanature}. In particular, the replica predictions in \cite{Opper1990, Opper1991, opper1996statistical, Kinzel96} follow from our results.
Another approach that originated in physics are the so-called TAP equations \cite{mezard1989space,kabashima2003cdma,kabashima2004bp} that lead to the so-called Approximate Message Passing algorithm for solving linear and generalized linear problems with Gaussian matrices \cite{donoho2009message,Rangan2010}. This algorithm can be analyzed with the so-called \emph{state evolution} method \cite{bayati2011dynamics}, and it is widely believed (and in fact proven for linear problems \cite{bayati2011lasso,gerbelot2020asymptotic}) that the fixed-point of the state evolution gives the optimal error in high-dimensional convex optimization problems. 
The state evolution equations are in fact equivalent to the one given by the replica theory and therefore our results vindicate this approach as well. We also demonstrate numerically that these asymptotic results are very accurate even for moderate system sizes, and they have been performed with the \texttt{scikit-learn} library \cite{scikit-learn}.

Secondly, and more importantly, we provide in Sec.~\ref{sec:applications} a detailed analysis of the generalization error for standard losses such as square, hinge (or equivalently support vector machine) and logistic, as a function of the regularization strength $\lambda$ and the number of samples per dimension $\alpha$. We observe, in particular, that while the ridge regression never closely approaches the Bayes-optimal performance, the logistic regression with optimized $\rL_2$ regularization gets extremely close to optimal. And so does, to a lesser extent, the hinge regression and the max-margin estimator to which the unregularized logistic and hinge converge \cite{Rosset04}. It is quite remarkable that these canonical losses are able to approach the error of the Bayes-optimal estimator for which, in principle, the marginals of a high-dimensional probability distribution need to be evaluated. Notably, all the later losses give ---for a \emph{good choice} of the regularization strength $\lambda$--- generalization errors scaling as $\Theta\(\alpha^{-1}\)$ for large $\alpha$, just as the Bayes-optimal generalization error~\cite{Barbier5451}. This is found to be at variance with the prediction of Rademacher and max-margin-based bounds that predict instead a $\Theta\(\alpha^{-1/2}\)$ rate \cite{vapnik2006estimation,shalev2014understanding}, which therefore appear to be vacuous in the high-dimensional regime. Notice that we reproduce the Rademacher complexity results of \cite{Abbara2019}, which deal exactly with the same setting, only to bring to light interesting conclusions on the ERM estimation.

Third, in Sec.~\ref{sec:optimality}, we design a custom (non-convex) loss and regularizer from the knowledge of the ground truth distributions $P_{\out^\star}, P_{\w^\star}$ that provably gives a plug-in estimator that efficiently achieves Bayes-optimal performances, including the optimal $\Theta\(\alpha^{-1}\)$ rate for the generalization error. Our construction is related to the one discussed in \cite{gribonval2011should,NIPS2013_4868,Advani2016}, but is not restricted to convex losses.

%% file: files/technical.tex
\label{sec:fixed_point}

In the formulas that arise for this statistical estimation problem, the correlations between
the estimator $\hat{\vec{w}}$ and the ground truth vector
$\vec{w}^\star$ play a fundamental role and we thus define two scalar overlap parameters to measure the statistical reconstruction:
\begin{align}
	m &\equiv \frac{1}{\ndim}~\EE_{\vec{y},\mat{X}}~\[\hat{\vec{w}}^\intercal\vec{w}^\star\]\,, 
	&& q \equiv \frac{1}{\ndim}~\EE_{\vec{y},\mat{X}}~\[\|\hat{\vec{w}}\|_2\]^2\,. 
\end{align}
In particular, the generalization error of the estimator $\hat{\vec{w}}(\alpha)
\in \bbR^{\ndim}$ obtained by performing Empirical Risk Minimization
(ERM) on the training loss $\mL$ in eq.~\eqref{main:training_loss}
with $\nsamples = \alpha \ndim$ samples
\begin{align}
e_{\rm g}^{\rm erm}(\alpha) \equiv \EE_{y, \vec{x}}  \id\[ y \ne \hat{y}\(\hat{\vec{w}}(\alpha); \vec{x} \) \]\,,
\end{align}
where $\hat{y}\(\hat{\vec{w}}(\alpha); \vec{x}\)$ denotes the
predicted label, has both at finite $\ndim$ and in the asymptotic limit an explicit expression depending only on the above overlaps $m$ and $q$:
\begin{proposition}[Generalization error of classification]
\label{main:thm:generalization_errors}
In our synthetic binary classification task, the 
generalization error of ERM (or equivalently the test error) is given by
\begin{align}
	e_{\rm g}^{\rm erm}\(\alpha\) &= \frac{1}{\pi} \textrm{acos}\( \sqrt{\eta} \)\,, 
	&& \text{with~~} \eta \equiv \frac{m^2}{\rho_{\ndim,\w^\star}~q} 
	\qquad \text{and}\qquad
    \rho_{\ndim,\w^\star} \equiv \frac{1}{\ndim}~\EE\[ \|\vec{w}^\star\|_2^2 \].
	\label{main:generalization_errors}	
	\end{align}
\end{proposition}
\begin{proof}
	The proof, shown in SM.~\ref{appendix:generalization_error}, is a simple computation based on Gaussian integration.
\end{proof}
To obtain the generalization performances, it thus remains to obtain the asymptotic values of $m$, $q$ (and thus of $\eta$), in the limit $d\to \infty$. With the $\ell_2$ regularization, these values are
characterized by a set of fixed point equations given by the next
theorems.
For any $\tau > 0$, let us first recall the definitions of the Moreau-Yosida regularization and the proximal operator of a convex loss function $(y,z) \mapsto \ell(y \cdot z)$:
\begin{align}
\label{eq:MY_reg}
	\mM_\tau(z) &=  \min_x \Big\{\ell(x) + \frac{(x-z)^2}{2\tau} \Big\}\,, && \mP_\tau(z) =  \argmin_x \Big\{\ell(x) + \frac{(x-z)^2}{2\tau} \Big\}\,.
\end{align}
\begin{theorem}[Gordon's min-max fixed point - Binary classification with $\rL_2$ regularization]
\label{main:thm:gordon_fixed_points:classification}
As $\nsamples, \ndim \to \infty$ with $\nsamples/ \ndim = \alpha = \Theta(1)$, the overlap parameters $m, q$ and the prior's second moment $\rho_{\ndim,\w^\star}$ concentrate to
\begin{align}
m & \underlim{\ndim}{\infty} \sqrt{\rho_{\w^\star}} \mu^\ast\,, && q \underlim{\ndim}{\infty} (\mu^\ast)^2 + (\delta^\ast)^2\,, &&  \rho_{\ndim,\w^\star} \underlim{\ndim}{\infty} \rho_{\w^\star}\,,
\label{main:fixed_point_gordon}
\end{align}
where parameters $\mu^\ast$ and $\delta^\ast$ are solutions of 
\begin{equation}\label{eq:pot_func}
(\mu^\ast, \delta^\ast) = \underset{\mu, \delta \ge 0}{\arg\min} \ \sup_{\tau > 0} \left\{\frac{\lambda(\mu^2 + \delta^2)}{2} - \frac{\delta^2}{2\tau} + \alpha \EE_{g, s} \mM_\tau[\delta g + \mu s \varphi_{\out^\star}(\sqrt{\rho_{\w^\star}} s)]\right\},
\end{equation}
and $g, s$ are two \iid standard normal random variables. The solutions $(\mu^\ast, \delta^\ast)$ of \eqref{eq:pot_func} can be reformulated as a set of fixed point equations 
\begin{align}
\begin{aligned}
	\mu^\ast &= \frac{\alpha}{\lambda \tau^\ast + \alpha } \EE_{g,s} [s \cdot \varphi_{\out^\star}(\sqrt{\rho_{\w^\star}} s) \cdot \mP_{\tau^\ast}(\delta^\ast g+ \mu^\ast ~ s ~ \varphi_{\out^\star}(\sqrt{\rho_{\w^\star}} s))]\,, \spacecase
	\delta^\ast &= \frac{\alpha}{\lambda \tau^\ast + \alpha -1} \EE_{g,s} [g \cdot \mP_{\tau^\ast}(\delta^\ast g+ \mu^\ast ~ s ~ \varphi_{\out^\star}(\sqrt{\rho_{\w^\star}} s))]\,,\spacecase
	(\delta^\ast)^2 &= \alpha \EE_{g,s} [\((\delta^\ast g + \mu^\ast ~ s ~\varphi_{\out^\star}(\sqrt{\rho_{\w^\star}} s)) - \mP_{\tau^\ast}(\delta^\ast g + \mu^\ast~s ~\varphi_{\out^\star}(\sqrt{\rho_{\w^\star}} s)) \)^2] \,.
\label{main:fixed_point_equations_gordon}
\end{aligned}
\end{align}
\end{theorem}

	The proof, shown in SM.~\ref{appendix:proof:gordon} for regression, and consequently valid for the classification particular case, is an application of the Gordon's comparison inequalities.
Let us mention that the theorem focuses on the case of classification for \iid isotropic Gaussian input data and $\ell_2$ regularization as this case was extensively studied in past works. However, similar techniques can be generalized to handle convex and separable regularization functions (see, \emph{e.g.}, \cite{Salehi:19}), and Gaussian inputs with more general covariance matrices \cite{DhifallahL:20}.

This set of fixed point equations can be finally mapped to the ones
obtained equivalently by the heuristic \emph{replica} method from statistical
physics (whose heuristic derivation is shown in
SM.~\ref{appendix:replicas}) as well as the state evolution of the Approximate Message-Passing (AMP) algorithm \cite{kabashima2003cdma,Rangan2010,zdeborova2016statistical}. 
Notice that the main reason why we rely on Convex Gaussian Min-lax Theory (CGMT) to make this set of equation rigorous is that we do not know how to prove that the AMP state evolution corresponds to the solution of ERM. Only after having the CGMT proof in hand, it follows that the SE of AMP gives the same equations than the CGMT.
As a result, their validity for this convex estimation problem is rigorously established by the following theorem:
\begin{corollary}[Equivalence Gordon-replicas]
\label{main:corollary:equivalence_gordon_replicas_formulation_l2}
As $\nsamples, \ndim \to \infty$ with $\nsamples/ \ndim = \alpha =
\Theta(1)$, the overlap parameters $m, q$ concentrate to the
fixed point of the following set of equations:
\begin{align}
		m &= \alpha~\Sigma~\rho_{\w^\star} \cdot \EE_{y, \xi } \[ \mZ_{\out^\star} \times f_{\out^\star} \(y,  \sqrt{\rho_{\w^\star} \eta} \xi, \rho_{\w^\star}\(1 - \eta\)\) \cdot f_{\out} \( y,  q^{1/2}\xi, \Sigma \)    \]\,,  \nonumber \\
		q &= m^2/\rho_{\w^\star} + \alpha \Sigma^2 \cdot \EE_{y, \xi } \[ \mZ_{\out^\star} \( y,  \sqrt{\rho_{\w^\star}\eta} \xi, \rho_{\w^\star}\(1 - \eta\)  \)  \cdot f_{\out} \( y,  q^{1/2}\xi, \Sigma \)^2    \]  \,, \label{main:fixed_point_equations_replicas} \\
		\Sigma &=   \(\lambda - \alpha \cdot \EE_{y, \xi } \[ \mZ_{\out^\star} \( y,  \sqrt{\rho_{\w^\star} \eta} \xi, \rho_{\w^\star}\(1 - \eta\)  \) \cdot  \partial_\omega f_{\out} \( y,  q^{1/2}\xi, \Sigma \)    \] \)^{-1} \,, \nonumber
\end{align}
\begin{align}
\text{with~~~} \eta &\equiv \frac{m^2}{\rho_{\w^\star} ~ q}\,, 	&&f_{\out} \( y, \omega, V  \) \equiv V^{-1} (\mP_{V}[l(y,.)](\omega) - \omega)\,,\nonumber \\
 \mZ_{\out^\star}\( y, \omega, V  \) &= \EE_{z}\[ P_{\out^\star}\(y|\sqrt{V}z +\omega\) \]\,, &&f_{\out^\star} (y,\omega,V) \equiv \partial_\omega \log \( \mZ_{\out^\star}\(y,\omega,V\) \)\,,
\label{main:proximal_replicas}
\end{align}
where $\xi, z$ denote two \iid standard Gaussian normal random variables, and $\EE_y$ the continuous or discrete sum over all possible values $y$ according to $P_{\out^\star}$.
\end{corollary}
For clarity, the proof is again left in SM.~\ref{appendix:proof:equivalence_gordon_replicas_formulation_l2}.
Notice that the equivalent sets of equations \eqref{main:fixed_point_equations_gordon}-\eqref{main:fixed_point_equations_replicas} have been made rigorous only for binary classification and regression in SM.~\ref{appendix:proof:gordon} with $\ell_2$ regularization.
However, the replica's prediction of the fixed point equations for the whole GLM class (classification and regression) are provided for generic convex and separable loss and regularizer (different than $\rL_2$) in SM.~\ref{appendix:proof:replicas_formulation} and contain instead six equations. 
These heuristic equations are nonetheless believed to hold true and the generic Gordon’s min-max framework can be easily generalized to this case.

\paragraph{Bayes optimal baseline ---} Finally, we shall 
compare the ERM performances to the Bayes-optimal generalization
error. Being the information-theoretically best possible estimator, we will use it as a reference baseline for comparison.
The expression of the Bayes-optimal generalization was derived in
\cite{Opper1991} and  proven in \cite{Barbier5451} and we recall here the result:
\begin{theorem}[Bayes Asymptotic performance, from \cite{Barbier5451}]
\label{main:thm:fixed_point_equations_bayes}
For the model \eqref{main:teacher_model} with $P_{\w^\star}(\vec{w}^\star)= \mN_{\vec{w}^\star}\(\vec{0}, \rho_{\ndim,\w^\star} \rI_{\ndim}\)$ with $\rho_{\ndim,\w^\star} \underlim{\ndim}{\infty} \rho_{\w^\star}$, the Bayes-optimal generalization error is quantified by two scalar parameters $q_\bayes$ and $\hat{q}_\bayes$ that verify asymptotically the set of fixed point equations
\vspace{-0.2cm}
	\begin{align}
		q_\bayes = \frac{\hat{q}_\bayes}{1+\hat{q}_\bayes}
		\,, 
		&& \hat{q}_\bayes = \alpha \EE_{y, \xi } \[ \mZ_{\out^\star } \( y,  q_\bayes^{1/2}\xi, \rho_{\w^\star} - q_\bayes \) \cdot  f_{\out^\star } \( y,  q_\bayes^{1/2}\xi, \rho_{\w^\star} - q_\bayes \)^2    \]  \,,
	\label{main:fixed_point_equations_bayes}
	\end{align}
\vspace{-0.5cm}
and reads
	\begin{align}
        e_{\rm g}^{\rm bayes}\(\alpha\) = \frac{1}{\pi}
          \textrm{acos}\(\sqrt{\eta_\bayes}\) \qquad \text{with} \qquad \eta_\bayes = \frac{q_\bayes}{\rho_{\w^\star}}\,.
          \label{main:generalization_error_bayes}
	\end{align}
\end{theorem}

%% file: files/applications.tex
\label{sec:applications}
We now move to the core of the paper and analyze the set of fixed point equations \eqref{main:fixed_point_equations_gordon}, or equivalently \eqref{main:fixed_point_equations_replicas}, leading to the generalization performances given by \eqref{main:generalization_errors},
 for common classifiers on our synthetic binary classification task. As already stressed, even though the results are valid for a wide range of separable convex loss and regularizers, we focus on estimators based on ERM with $\rL_2$ regularization $r(\vec{w}) = \lambda \|\vec{w}\|_2^2/2$, and with square loss (ridge regression) $l^{\rm square}(y,z)=\frac{1}{2}(y-z)^2$, logistic loss (logistic regression) $l^{\rm logistic}(y, z) = \log(1+\exp(-y z))$ or hinge loss (SVM) $l^{\rm hinge}(y, z) = \max\(0,1-yz\)$. In particular, we study the influence of the hyper-parameter $\lambda$ on the generalization performances and the different large $\alpha$ behavior generalization rates in the high-dimensional regime, and compare with the Bayes results.
 We show the solutions of the set of fixed point equations eqs.~\eqref{main:fixed_point_equations_replicas} in Figs.~\ref{fig:gen_error_ridge},~\ref{fig:gen_error_hinge},~\ref{fig:gen_error_logistic} respectively for ridge, hinge and logistic $\rL_2$ regressions. 
  Let us mention that the analytic solution of the set of equations \eqref{main:fixed_point_equations_gordon} is provided only in the ridge case, for which its quadratic loss allows to derive and fully solve the equations (see SM.~\ref{appendix:applications:ridge}), and also for logistic and hinge losses in the regime of vanishing $\lambda \to 0$ (see SM.~\ref{appendix:applications:hinge}).  
 Unfortunately, in general the set of equations has no analytical closed form expression and needs therefore to be evaluated numerically. It is in particular the case for logistic and hinge for finite $\lambda$, whose Moreau-Yosida regularization eq.~\eqref{eq:MY_reg} is, yet, analytical. 
 However, note that \eqref{main:fixed_point_equations_gordon} are fixed point equations on scalar variables so that their numerical resolution posed no problem. The non-trivial part of the rigorous analysis, performed in Thm.~\ref{main:thm:gordon_fixed_points:classification}, is the reduction of the high-dimensional problem to the scalar fixed point equations.  

 First, to highlight the accuracy of the theoretical predictions, we compare in Figs.~\ref{fig:gen_error_ridge}-\ref{fig:gen_error_logistic} the ERM asymptotic ($\ndim\to \infty$) generalization error with the performances of numerical simulations ($\ndim = 10^{3}$, averaged over $n_s=20$ samples) of ERM of the training loss eq.~\eqref{main:training_loss}. Presented for a wide range of number of samples $\alpha$ and of regularization strength $\lambda$, we observe a perfect match between theoretical predictions and numerical simulations so that the error bars are barely visible and have been therefore removed. This shows that the asymptotic predictions are valid even with very moderate sizes. As an information theoretical baseline, we also show the Bayes-optimal performances (black) given by the solution of eq.~\eqref{main:fixed_point_equations_bayes}.

 \paragraph{Ridge estimation---}  As we might expect the square loss gives the worst performances. For low values of the generalization, it leads to an interpolation-peak at $\alpha=1$. The limit of vanishing regularization $\lambda \to 0$ leads to the  \emph{least-norm} or \emph{pseudo-inverse} estimator $\hat{\vec{w}}_{\rm pseudo} = \(\mat{X}^\intercal\mat{X}\)^{-1}\mat{X}^\intercal \vec{y}$. The corresponding generalization error presents the largest interpolation-peak and achieves a maximal generalization error $e_{\rm g}=0.5$. These are well known observations, discussed  as early as in \cite{Kinzel96, Opper1990, krogh1992simple}, that are object of a renewal of interest under the name \emph{double descent}, following a recent series of papers \cite{Hastie19,Geiger19,geiger2020scaling,Belkin19,Mitra2019,Mei19,gerace2020generalisation,d2020double}. This double descent behavior for the pseudo-inverse is shown in Fig.~\ref{fig:gen_error_ridge} with a yellow line. On the contrary, larger regularization strengths do not suffer this peak at $\alpha=1$, but their generalization error performance is significantly worse than the Bayes-optimal baseline for larger values of $\alpha$. Indeed, as we might expect, for a large number of samples, a large regularization biases wrongly the training. However, even with optimized regularizations, performances of the ridge estimator remains far away from the Bayes-optimal performance.

 \paragraph{Hinge and logistic estimation---}  Both these losses, which are the classical ones used in classification problems, improve drastically the generalization error. First of all, let us notice that they do not display a double-descent behavior. This is due to the fact that our results are illustrated in the noiseless case and that our synthetic dataset is always linearly separable. Optimizing the regularization, our results in Fig.~\ref{fig:gen_error_hinge}-\ref{fig:gen_error_logistic} show both hinge and logistic ERM-based classification approach very closely the Bayes error. This might be an interesting message for practitioners, even though showing it in more realistic settings would be preferable.
 To offset these results, note that performances of logistic regression on non-linearly separable data are however very poor, as illustrated by our analysis of a \emph{rectangle door} teacher (see SM.~\ref{appendix:applications:rectangle}).
  
 \paragraph{Max-margin estimation---} As discussed in \cite{Rosset04}, both the logistic and hinge  estimator converge, for vanishing regularization $\lambda\to0$, to the \emph{max-margin} solution.  Taking the $\lambda \to 0$ limit in our equations, we thus obtain the \emph{max-margin} estimator performances. While this is not what gives the best generalization error (as can be seen in Fig.\ref{fig:gen_error_logistic} the logistic with an optimized $\lambda$ has a lower error), the max-margin estimator gives very good results, and gets very close to the Bayes-error.

\paragraph{Optimal regularization---} Defining the regularization value that optimizes the generalization as
\begin{align}
	\lambda^{\rm opt}\(\alpha\) &= \argmin_{\lambda} e_{\rm g}^{\rm erm}\(\alpha, \lambda\)\,.
\end{align} 
we show in Figs.~\ref{fig:gen_error_hinge}-\ref{fig:gen_error_logistic} that both optimal values $\lambda^{\rm opt}\(\alpha\)$ (dashed-dotted orange) for logistic and hinge regression decrease to $0$ as $\alpha$ grows and more data are given. Somehow surprisingly, we observe in particular that  the generalization performances of logistic regression with optimal regularization are \emph{extremely close} to the Bayes performances. The difference with the optimized logistic generalization error is barely visible by eye, so that we explicitly plotted the difference, which is roughly of order $10^{-5}$. 

Ridge regression Fig.~\ref{fig:gen_error_ridge} shows a singular behavior: there exists an  optimal value (purple) which is moreover independent of $\alpha$ achieved for $\lambda^{\rm opt} \simeq 0.5708$. This value was first found numerically and confirmed afterwards semi-analytically in SM.~\ref{appendix:applications:ridge}.

\begin{figure}
    \centering
    \begin{subfigure}[b]{\textwidth}
		\centering
		\includegraphics[scale=0.44]{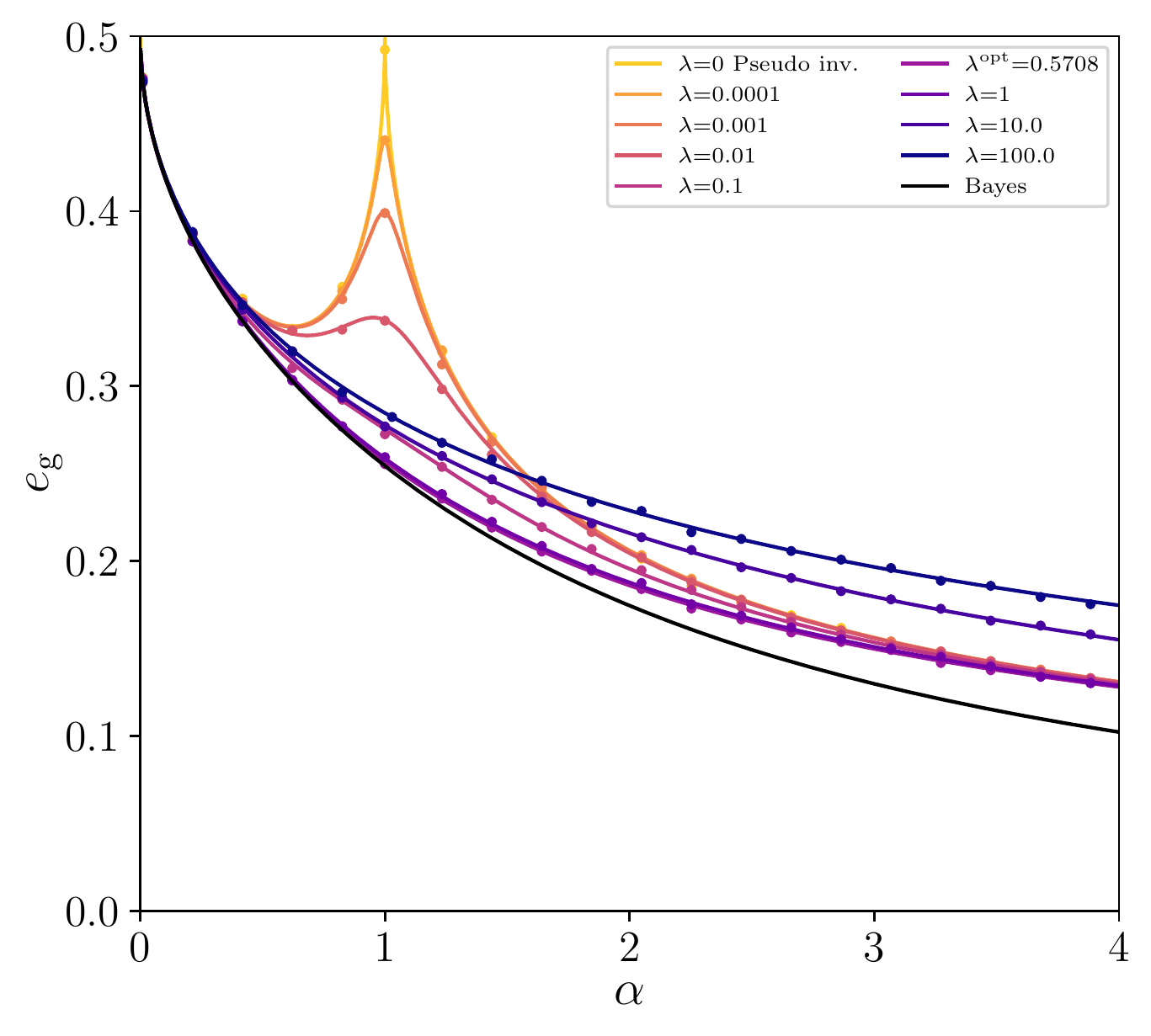}
		\hfill
		\includegraphics[scale=0.44]{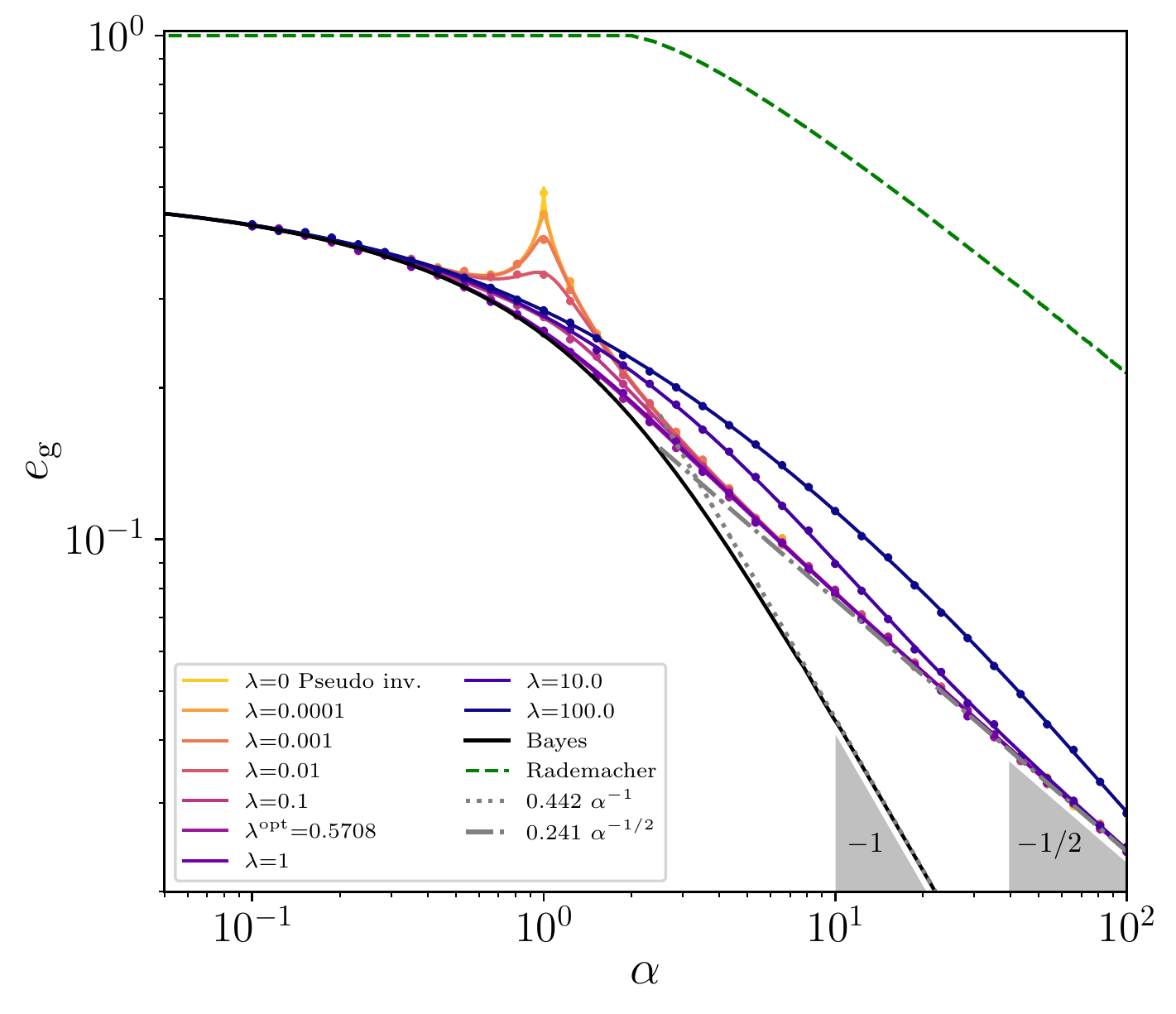}
		\caption{Ridge regression: square loss with $\rL_2$ regularization. Interpolation-peak, at $\alpha=1$, is maximal  for the pseudo-inverse estimator $\lambda=0$ (yellow line) that reaches $e_{\rm g}=0.5$. 
		}
		\label{fig:gen_error_ridge}	
    \end{subfigure}
    \vskip\baselineskip
    \begin{subfigure}[b]{\textwidth}
		\centering
	\includegraphics[scale=0.44]{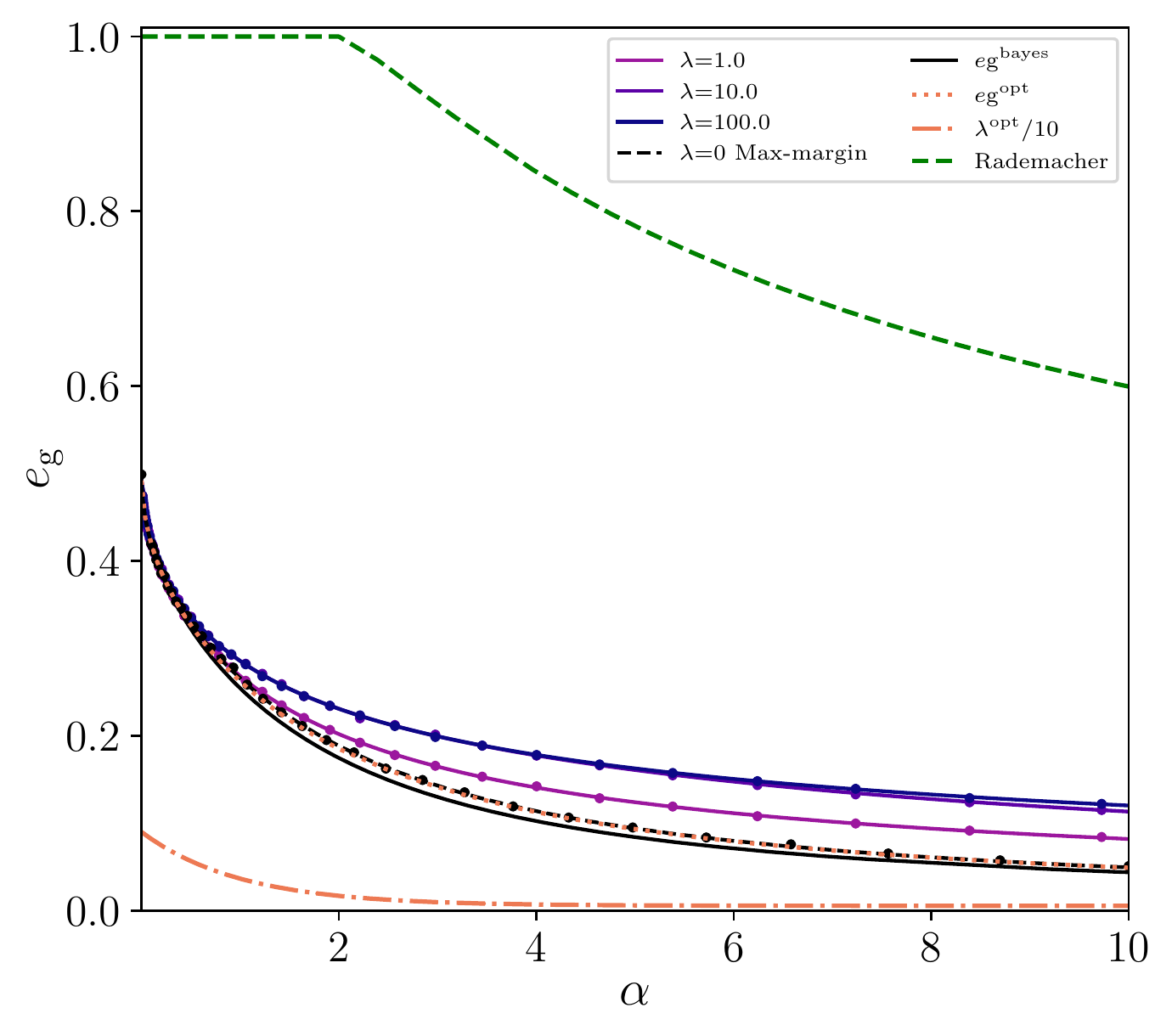}
	\hfill
	\includegraphics[scale=0.44]{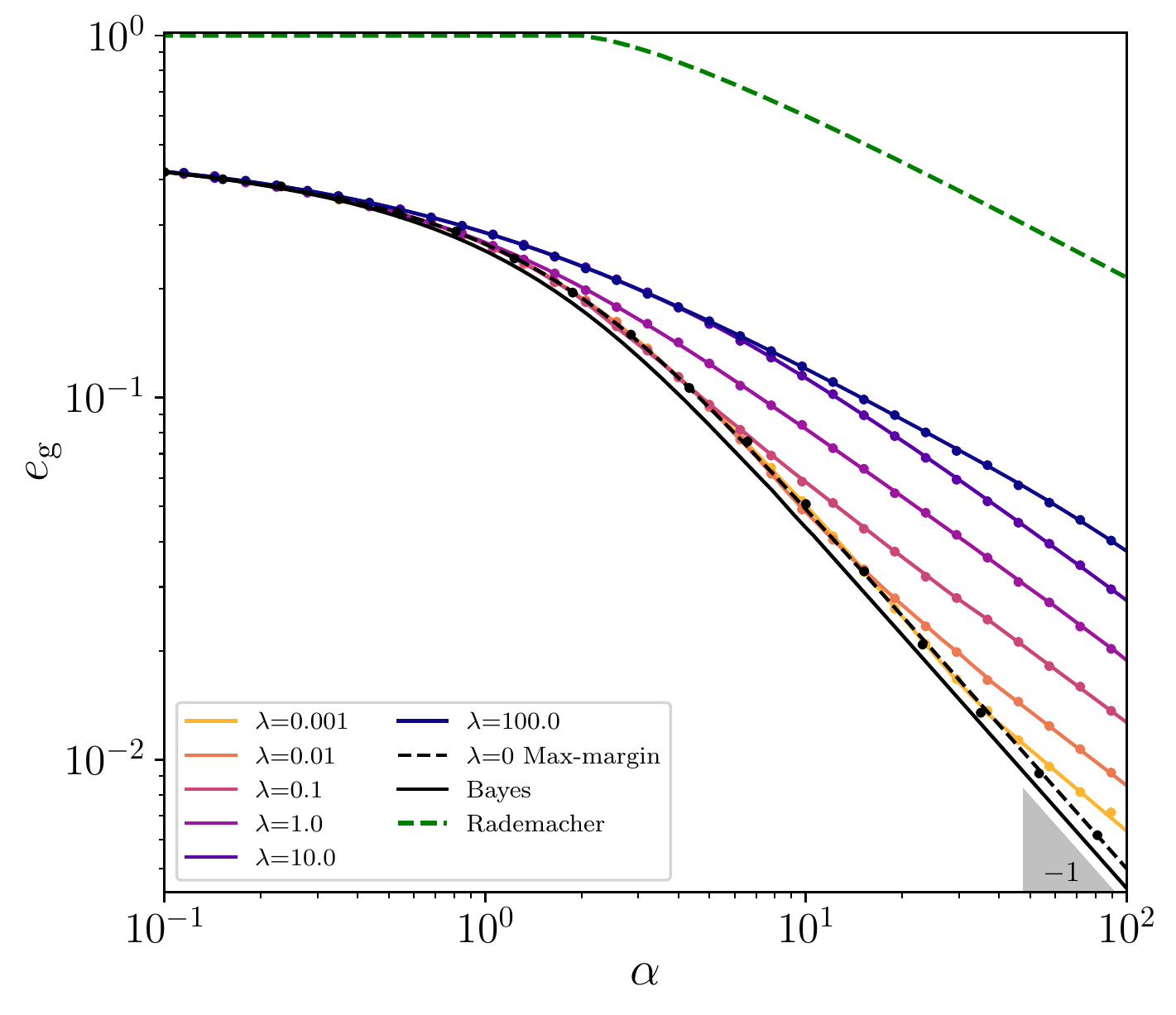}
		\caption{Hinge regression: hinge loss with $\rL_2$ regularization. For clarity the rescaled value of $\lambda^{\rm opt}/10$ (dotted-dashed orange) is shown as well as its generalization error $e_{\rm g}^{\rm opt}$ (dotted orange) that is slightly below and almost indistinguishable of the max-margin performances (dashed black).}
		\label{fig:gen_error_hinge}	
	\end{subfigure}
\end{figure}
\begin{figure}[p]
\ContinuedFloat
    \begin{subfigure}[b]{\textwidth}
		\centering
		\includegraphics[width=0.44\linewidth]{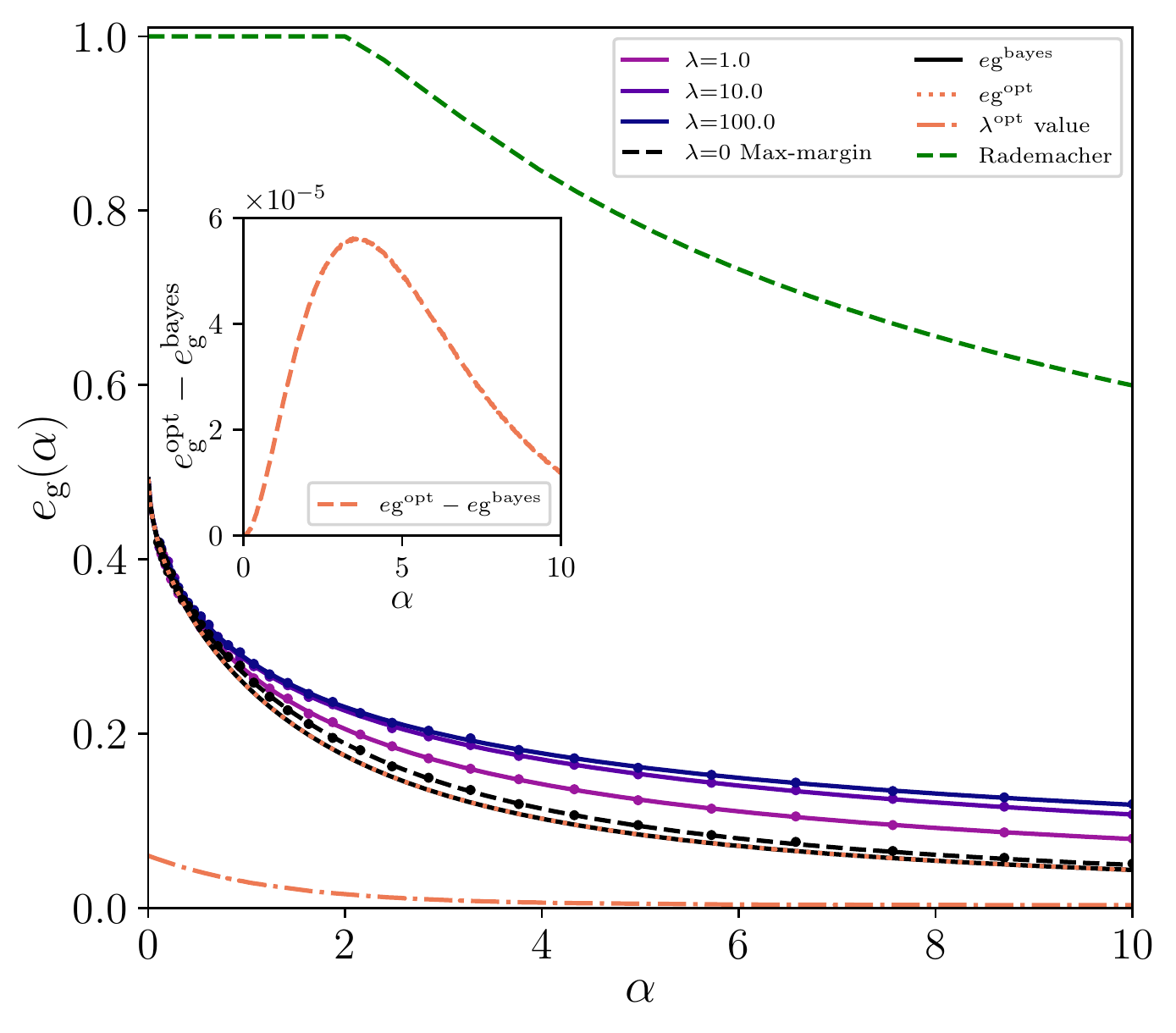}
		\hfill
		\includegraphics[scale=0.44]{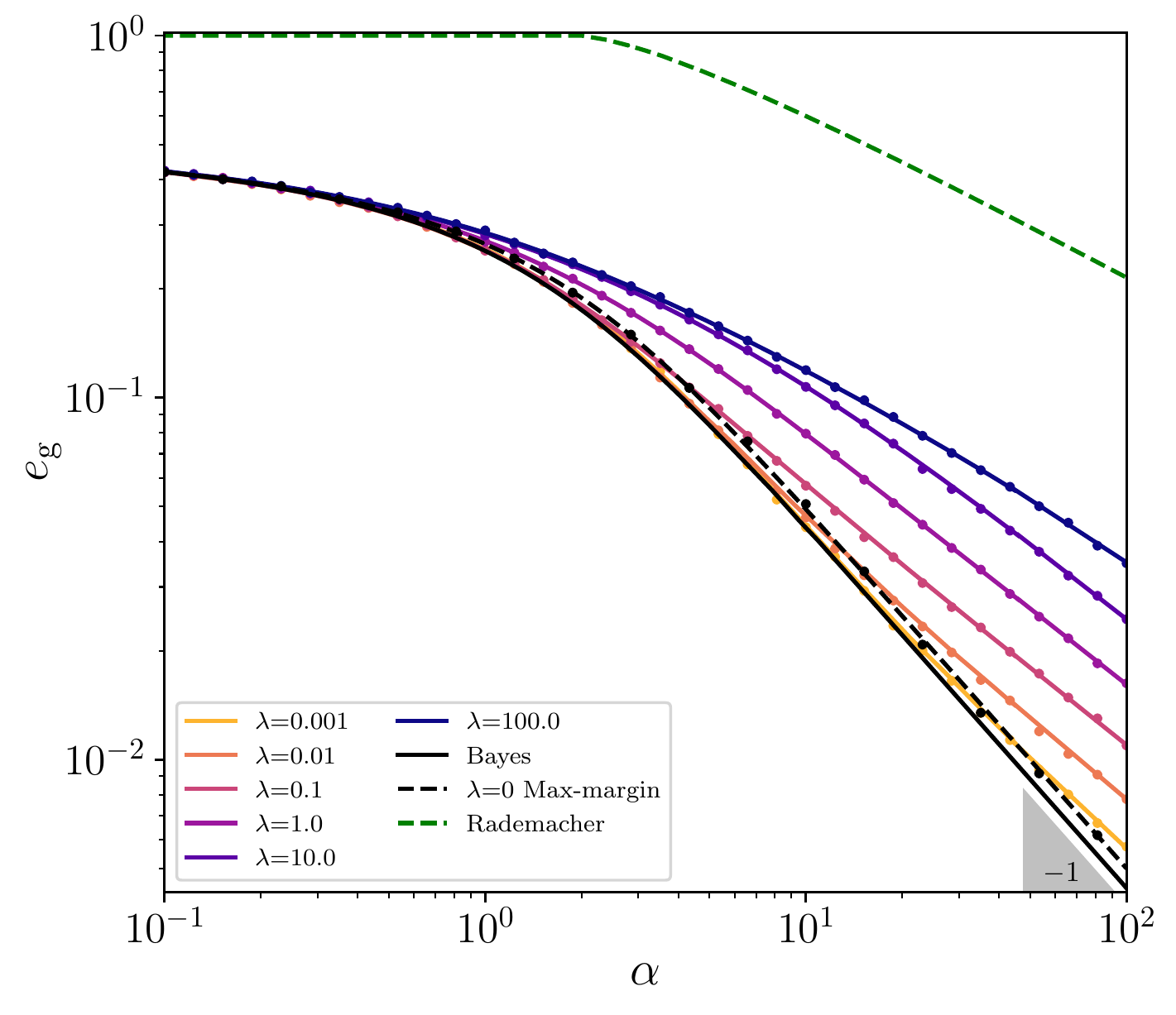}
	 \caption{Logistic regression: logistic loss with $\rL_2$ regularization - The value of $\lambda^{\rm opt}$ (dotted-dashed orange) is shown as well as its generalization error $e_{\rm g}^{\rm opt}$ (dotted orange). Visually indistinguishable from the Bayes-optimal line, their difference $e_{\rm g}^{\rm opt}-e_{\rm g}^{\rm bayes}$ is shown as an inset (dashed orange).}
	 \label{fig:gen_error_logistic}	
	 \end{subfigure}
     \caption{Asymptotic generalization error for $\rL_2$ regularization ($\ndim \to \infty$) as a function of $\alpha$ for different regularizations strengths $\lambda$, compared to numerical simulation (points) of ridge regression for $\ndim=10^{3}$ and averaged over $n_s=20$ samples. Numerics has been performed with the default methods \textit{Ridge}, \textit{LinearSVC},  \textit{LogisticRegression} of \texttt{scikit-learn} package \cite{scikit-learn}. Bayes optimal performances are shown with a black line and goes as $\Theta\(\alpha^{-1}\)$, while the Rademacher complexity \cite{Abbara2019} (dashed green) decreases as $\Theta\(\alpha^{-1/2}\)$. Both hinge and logistic converge to max-margin estimator (limit $\lambda=0$) which is shown in dashed black and deceases as $\Theta(\alpha^{-1})$, while Ridge decreases as $\Theta(\alpha^{-1/2})$.
     }
     \label{fig:gen_error_global}
\end{figure}

\paragraph{Generalization rates at large $\alpha$---} Finally, we turn to the very instructive behavior at large values of $\alpha$ when a large amount of data is available. First, we notice that the Bayes-optimal generalization error, whose large $\alpha$ analysis is performed in  SM.~\ref{appendix:applications:bayes}, decreases as $e_{\rm g}^{\rm bayes} \underset{\alpha \to \infty}{\sim} 0.4417 \alpha^{-1}$. Compared to this optimal value, ridge regression gives poor performances in this regime. For any value of the regularization $\lambda$ --- and in particular for both the pseudo-inverse case at $\lambda=0$ and the optimal estimator $\lambda^{\rm opt}$ --- its generalization performances decrease much slower than the Bayes rate, and goes only as $e_{\rm g}^{\rm ridge} \!\underset{\alpha \to \infty}{\sim} \!  0.2405 \alpha^{-1/2}$ (see SM.~\ref{appendix:applications:ridge} for the derivation).
Hinge and logistic regressions present a radically different, and more favorable, behavior. Fig.~\ref{fig:gen_error_hinge}-\ref{fig:gen_error_logistic} show that keeping $\lambda$ finite when $\alpha$ goes to $\infty$, does not yield the Bayes-optimal rates. However the max-margin solution (that corresponds to the $\lambda \to 0$ limit of these estimators) gives extremely good performances $e_{\rm g}^{\rm logistic,hinge} \underset{\lambda \to 0}{\sim} e_{\rm g}^{\rm max-margin} \!\underset{\alpha \to \infty}{\sim} \! 0.500 \alpha^{-1}$ see derivation in SM.~\ref{appendix:applications:hinge}). 
This is the same rate as the Bayes one, only that the constant is slightly higher.
However, we do not know whether there is a general criteria that would distinguish when the decay is $\Theta (\alpha^{-1})$ or $\Theta (\alpha^{-1/2})$. Providing such a generic criteria is definitely a line of research we would like to investigate in the future.
{\color{black}
Moreover, let us point out the work \cite{sridharan2009fast} that discusses fast convergence rates $\Theta(\nsamples^{-1})$ for the hinge loss, whose analysis for only very large $\alpha$ and Lipshitz functions does not directly apply to our setting.
}

\paragraph{Comparison with VC and Rademacher statistical bounds---} Given the fact that both the max-margin estimator and the optimized logistic achieve optimal generalization rates going as $\Theta\(\alpha^{-1}\)$, it is of interest to compare those rates to the prediction of statistical learning theory bounds. Statistical learning analysis (see e.g.~\cite{vapnik2006estimation,Bartlett98,shalev2014understanding}) relies to a large extent on the \emph{Vapnik-Chervonenkis} dimension (VC) analysis and on the so-called \emph{Rademacher complexity}. The uniform convergence result states that if the Rademacher complexity or the Vapnik-Chervonenkis dimension $d_{\rm VC}$ is finite, then for a large enough number of samples the generalization gap will vanish uniformly over all possible values of parameters. Informally, uniform convergence tells us that with high probability, for any value of the weights $\vec{w}$, the generalization gap satisfies $ {\mR}_{\rm population}(\vec{w}) - {\cal R}_{\rm empirical}^{\nsamples} (\vec{w}) = \Theta\(\sqrt{d_{\rm VC} / \nsamples}\)$ where $d_{\rm VC} = \ndim-1$ for our GLM hypothesis class. Therefore, given that the empirical risk can go to \emph{zero} (since our data are separable), this provides a generalization error upper-bound $ e_{\rm g} \!\leq\!  \Theta(\alpha^{-1/2})$.
This is much worse that what we observe in practice, where we reach the Bayes rate $e_{\rm g} = \Theta(\alpha^{-1})$. Tighter bounds can be obtained using the Rademacher complexity, and this was studied recently (using the aforementioned \emph{replica method}) in \cite{Abbara2019} for the very same problem. To bring to light interesting conclusions, we reproduced their results and plotted the Rademacher complexity generalization bound in Fig.\ref{fig:gen_error_global} (dashed-green) that decreases as $\Theta\(\alpha^{-1/2}\)$ for the binary classification task eq.~\eqref{main:teacher_sign}.

One may wonder if this could be somehow improved. Another statistical-physics heuristic computation, however, suggests that, unfortunately, uniform bound are plagued to a slow rate $\Theta\(\alpha^{-1/2}\)$. Indeed, the authors of \cite{engel1993statistical} showed with a replica method-style computation that \emph{there exists} some set of weights, in the binary classification task eq.~\eqref{main:teacher_sign}, that leads to $\Theta\(\alpha^{-1/2}\)$ rates: the uniform bound is thus tight. The gap observed between the uniform bound and the almost Bayes-optimal results observed in practice in this case is therefore not a paradox, but an illustration that the price to pay for uniform convergence is the inability to describe the optimal rates one can sometimes get in practice. Therefore, we believe, that the fact this phenomena can be observed in a such simple problem sheds an interesting light on the current debate in understanding generalization in deep learning \cite{zhang2016understanding}.

Remarking our synthetic dataset is linearly separable, we may try to take this fact into consideration to improve the generalization rate. In particular, it can be done using the max-margin  based generalization error for separable data:
\begin{theorem}[Hard-margin generalization bound \cite{vapnik2006estimation,Bartlett98,shalev2014understanding}]
	Given $S=\{\vec{x}_1,\cdots, \vec{x}_\nsamples\}$ such that $\forall \mu \in [1:\nsamples], \|\vec{x}_\mu\| \leq r$. Let $\hat{\vec{w}}$ the hard-margin SVM estimator on $S$ drawn with distribution $D$. With probability $1-\delta$, the generalization error is bounded by
	\begin{align}
		e_{\rm g}(\alpha) \underset{\alpha \to \infty}{\leq} \(4 r \|\hat{\vec{w}}\|  + \sqrt{\log\(4/\delta\) \log_2 \|\hat{\vec{w}}\|} \)/\sqrt{\nsamples}\,.
	\end{align}
\end{theorem}
In our case one has $r^2 \simeq \frac{1}{\ndim} \EE_{\vec{x}} \| \vec{x} \|_2^2 = \frac{1}{\ndim}\sum_{i=1}^\ndim \EE x_i^2 = 1$. On the other hand, in the large size limit, the  norm of the estimator $\|\hat{\vec{w}}\|_2 / \sqrt{\ndim}\simeq \sqrt{q}$, that yields $e_{\rm g}(\alpha)\leq 4 \sqrt{\frac{q}{\alpha}}$. We now need to plug the values of the norm $q$ obtained by our max-margin solution to finally obtain the results. Unfortunately, this bound turns out to be even worse than the previous one. Indeed the norm of the hard margin estimator $q$ is found to grow with $\alpha$ in the solution of the fixed point equation, and therefore the margin decay rather fast, rendering the bound vacuous. For small values of $\alpha$, one finds that $q\sim \alpha$ that provides a vacuous constant generalisation bound $e_{\rm g}\leq \Theta\(1\)$, while for large $\alpha$, $q\sim \alpha^2$ that yields an even worse bound $e_{\rm g}\leq \Theta\(\sqrt{\alpha}\)$. Clearly, max-margin based bounds do not perform well in this high-dimensional example.

%% file: files/optimality.tex
\label{sec:optimality}
Given the fact that logistic and hinge losses reach values extremely close to Bayes optimal generalization performances, one may wonder if by somehow slightly altering these losses one could actually reach the Bayesian values with a plug-in estimator obtained by ERM. 
This is what we achieve in this section, by constructing a (non-convex) optimization problem with a specially tuned loss and regularization from the knowledge of the teacher distributions $P_{\rm out^\star}$, $P_{\rm w^\star}$, whose solution yields Bayes-optimal generalization. 
Indeed, in the Bayes-optimal setting, we may directly use the Bayes-optimal AMP algorithm to achieve optimal performances as proven in \cite{Barbier5451}. Nevertheless, it seems to us interesting to point out that Bayes performances, which require in principle to compute an intractable high-dimensional posterior sampling, can be obtained instead by the easier, more common and practical ERM estimation.
Recent insights have shown that indeed one can sometime re-interpret Bayesian estimation as an optimization program in inverse
problems\cite{gribonval2011should,NIPS2013_4868,gribonval2018characterization,gribonval2019bayesian}. In
particular, \cite{Advani2016} showed explicitly, on the basis of the non-rigorous replica method of statistical mechanics, that some Bayes-optimal reconstruction problems could be turned into convex M-estimation. 

Matching ERM and Bayes-optimal generalization errors eqs.~\eqref{main:generalization_errors}-\eqref{main:generalization_error_bayes} with overlaps respectively solutions of eq.~\eqref{main:fixed_point_equations_replicas}-\eqref{main:fixed_point_equations_bayes} and assuming that $\mZ_{\w^\star}\(\gamma,\Lambda\)\equiv \EE_{w \sim P_{\w^\star}} \exp(- 1/2\Lambda w^2 + \gamma w)$ and $\mZ_{\out^\star}\(y,\omega,V\)$, defined in \eqref{main:proximal_replicas}, are log-concave in $\gamma$ and $\omega$, we define the optimal loss and regularizer $l^{\rm opt}$, $r^{\rm opt}$:
\begin{align}
	\begin{aligned}
		l^{\rm opt}\(y,z\) &=- \min_\omega \( \frac{(z-\omega)^2}{2 (\rho_{\w^\star}-q_\bayes)} + \log \mZ_{\out^\star} \(y,\omega,\rho_{\w^\star}-q_\bayes\) \)\,, \spacecase
		r^{\rm opt}\(w\) &= - \min_{\gamma} \( \frac{1}{2} \hat{q}_\bayes w^2 - \gamma w + \log \mZ_{\w^\star} \(\gamma, \hat{q}_\bayes\) \)\,, \text{ with } (q_\bayes,\hat{q}_\bayes) \text{ solution of eq.~\eqref{main:fixed_point_equations_bayes}}\,.  
		\label{main:opt_loss_reg}
	\end{aligned}
\end{align}
See SM.~\ref{appendix:optimal_loss_reg} for the derivation. Following these considerations, we provide the following theorem:
\begin{theorem}
	The result of empirical risk minimization eq.~(\ref{main:training_loss}) with $l^{\rm opt}$ and $r^{\rm opt}$ in eq.~\eqref{main:opt_loss_reg}, leads to Bayes optimal generalization error in the high-dimensional regime.
\end{theorem}
\begin{proof}
\vspace{-0.2cm}
  We present only the sketch of the proof here. 
First we note that the so called Bayes-optimal Generalized Approximate Message Passing (GAMP) algorithm \cite{Rangan2010}, recalled in SM.~\ref{appendix:optimal_loss_reg:amp}, with Bayes-optimal updates $f_{\rm out}^{\rm bayes}$ and $f_{\rm w}^{\rm bayes}$ in SM.~\ref{appendix:sec:definitions:update_bayes}
  is provably convergent and reaches Bayes-optimal performances (see \cite{Barbier2017b}). 
  Second, we remark that the GAMP algorithm is valid for ERM estimation
   with the corresponding updates $f_{\rm out}^{\rm erm}(l,r), f_{\rm w}^{\rm erm}(l,r)$, defined in  SM.~\ref{appendix:definitions:map_updates}, for a given loss $l$ and regularizer $r$. 
  To achieve Bayes-optimal performances, we design optimal loss $l^{\rm opt}$ and regularizer $r^{\rm opt}$ eq.~\ref{main:opt_loss_reg} such that at each time step the ERM denoisers match the Bayes-optimal ones: $f_{\rm out}^{\rm erm}(l^{\rm opt},r^{\rm opt}) = f_{\rm out}^{\rm bayes}$ and $f_{\rm w}^{\rm erm}(l^{\rm opt},r^{\rm opt}) = f_{\rm w}^{\rm bayes}$.
  In this context, AMP algorithm for ERM with loss and regularization given by (\ref{main:opt_loss_reg}) is exactly identical to the Bayes-optimal AMP. 
  This shows that AMP applied to the ERM problem corresponding to \eqref{main:opt_loss_reg} both converge to its fixed point and reach Bayes-optimal performances. The theorem finally follows by noting (see \cite{montanari2012graphical,gerbelot2020asymptotic}) that the AMP fixed point corresponds to the extremization conditions of the loss.
\end{proof}
\begin{figure}[!htb]
	\vspace{-0.3cm}
 	\centering
        \includegraphics[scale=0.38]{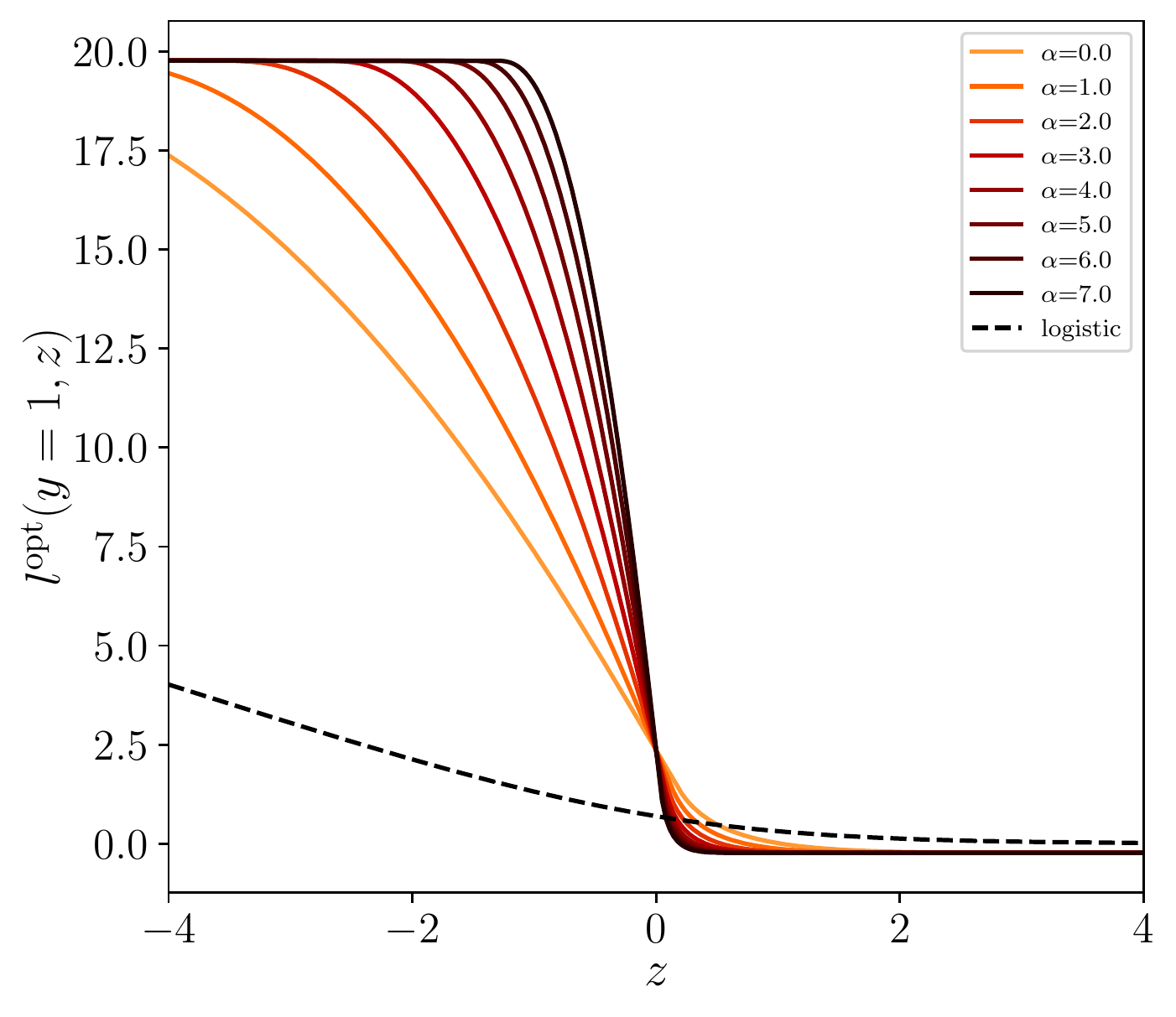}
        \hspace{2cm}
        \includegraphics[scale=0.38]{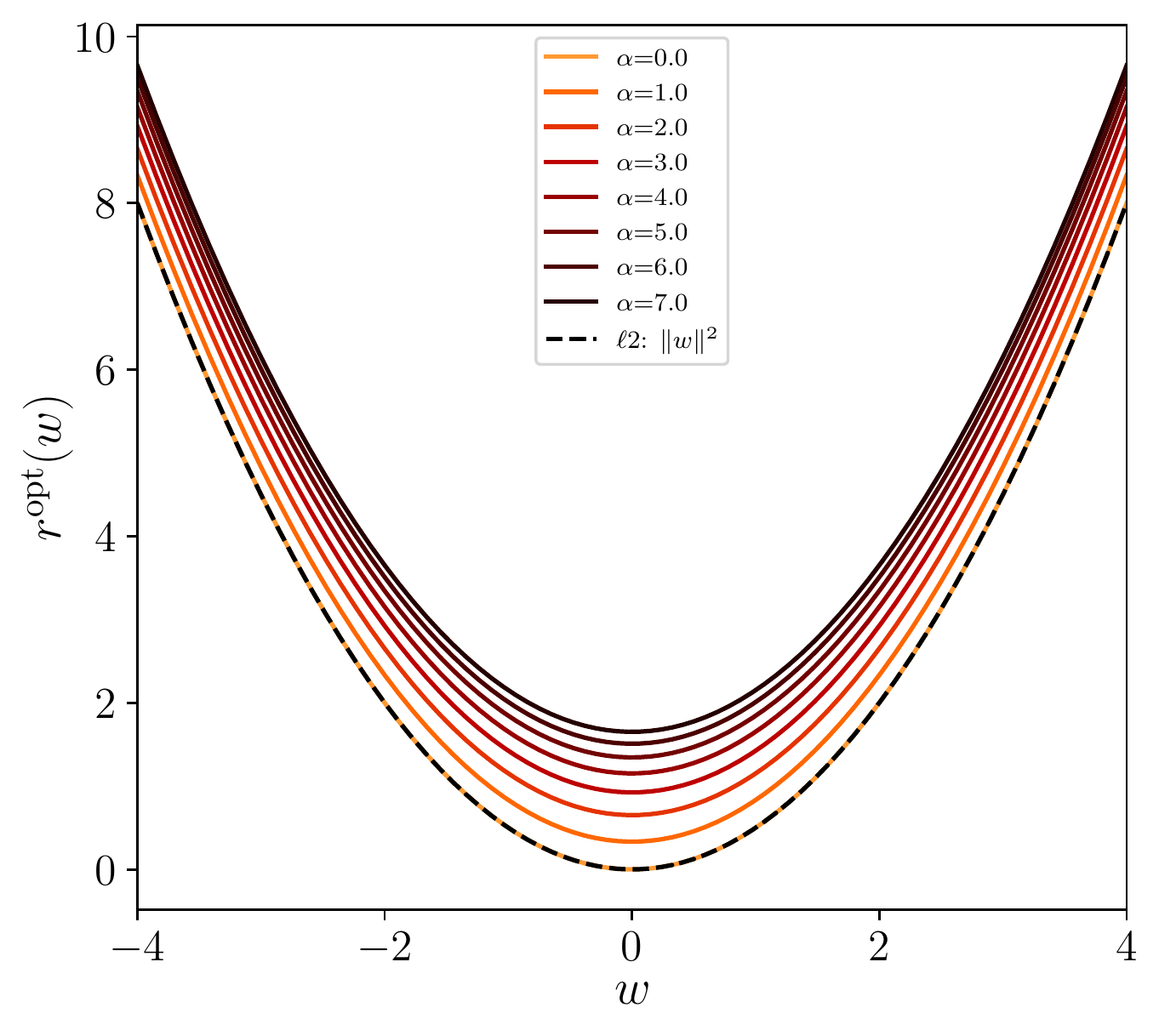}
        \vspace{-0.2cm}
        \caption{Optimal loss $l^{\rm opt}\(y=1,z\)$ and regularizer $r^{\rm opt}\(w\)$ for model eq.~\eqref{main:teacher_sign} as a function of $\alpha$.
     			}
        \label{fig:opt_loss_reg_sign_gaussian}
 \end{figure}
The optimal loss and regularizer $\lambda^{\rm opt}$ and $r^{\rm opt}$ for the model \eqref{main:teacher_sign} are illustrated in Fig.~\eqref{fig:opt_loss_reg_sign_gaussian}. And numerical evidences of ERM with (\ref{main:opt_loss_reg}) compared to $\rL_2$ logistic regression and Bayes performances are presented in SM.~\ref{appendix:optimal_loss_reg}.

%% file: files/acknowledgment.tex
This work is supported by the ERC under the
European Unions Horizon 2020 Research and Innovation
Program 714608-SMiLe, by the French Agence Nationale
de la Recherche under grant ANR-17-CE23-0023-01 PAIL
and ANR-19-P3IA-0001 PRAIRIE, and by the US National Science Foundation under grants CCF-1718698 and
CCF-1910410. 
We would also like to thank the Kavli Institute for Theoretical Physics (KITP)  for welcoming us during part of this research, with the support of the National Science Foundation under Grant No. NSF PHY-1748958.
We also acknowledge support from the chaire CFM-ENS “Science des données”. 
Part of this work was done when Yue M. Lu was visiting Ecole Normale as a CFM-ENS “Laplace” invited researcher.

%% file: appendix.tex
\renewcommand{\partname}{}
\renewcommand{\thesection}{\Roman{section}}
\part{Supplementary material}

In this supplementary material (SM), we provide the proofs and computation details leading to the results presented in the main manuscript. 
In Sec.~\ref{appendix:definitions}, we first recall the definition of the statistical model used in Sec.~\ref{sec:introduction} and we give proper definitions of the denoising distributions involved in the analysis of the Bayes-optimal and Empirical Risk Minimization (ERM) estimation. In particular, we provide the analytical expressions of the denoising functions used in Sec.~\ref{sec:applications} to analyze ridge, hinge and logistic regressions. 
In Sec.~\ref{appendix:generalization_error}, we detail the computation of the binary classification generalization error leading to the expressions in Proposition.~\ref{main:thm:generalization_errors} and Thm.~\ref{main:thm:fixed_point_equations_bayes} respectively for ERM and Bayes-optimal estimation.
In Sec.~\ref{appendix:proof}, we present the proofs of the central theorems stated in Sec.~\ref{sec:fixed_point}. In particular, we derive the Gordon-based proof of the Thm.~\ref{main:thm:gordon_fixed_points:classification} in the more general \emph{regression} (real-valued) version and provide as well the proof of Corollary.~\ref{main:corollary:equivalence_gordon_replicas_formulation_l2} which establishes the equivalence between the set of fixed-point equations of the Gordon's proof in the binary classification case and the one resulting from the heuristic replica computation.
The corresponding statistical physics framework used to analyze Bayes and ERM statistical estimations and the replica computation leading to expressions in Corollary.~\ref{main:corollary:equivalence_gordon_replicas_formulation_l2} are detailed In Sec.~\ref{appendix:replicas}.
The section \ref{appendix:applications} is devoted to provide additional technical details on the results with $\rL_2$ regularization addressed in Sec.~\ref{sec:applications}. In particular, we present the large $\alpha$ expansions of the generalization error for the \emph{Bayes-optimal}, \emph{ridge}, \emph{pseudo-inverse} and \emph{max-margin} estimators, and we investigate the performances of logistic regression on non-linearly separable data.
Finally in Sec.~\ref{appendix:optimal_loss_reg}, we show the derivation of the fine-tuned loss and regularizer provably leading to Bayes-optimal performances, as explained and advocated in Sec.~\ref{sec:optimality}, and we show some numerical evidences that ERM achieves indeed Bayes-optimal error in Fig.~\ref{fig:opt_loss_reg_sign_gaussian_numerics}.

\newpage
\setcounter{tocdepth}{4}
\parttoc 
\newpage


\section{Definitions and notations}
\input{files/supplementary/definitions.tex}

\newpage
\section{Binary classification generalization errors}
\input{files/supplementary/generalization_error.tex}

\newpage
\section{Proofs of the ERM fixed points}
\input{files/supplementary/proof.tex}

\newpage
\section{Replica computation for Bayes-optimal and ERM estimations}
\input{files/supplementary/replicas.tex}

\newpage
\section{Applications}
\input{files/supplementary/applications.tex}

\newpage
\section{Reaching Bayes optimality}
\input{files/supplementary/optimality.tex}

%% file: files/supplementary/definitions.tex
\label{appendix:definitions}

\subsection{Statistical model}

We recall the supervised machine learning task considered in the main manuscript eq.~\eqref{main:teacher_model}, whose dataset is generated by a single layer neural network, often named a \emph{teacher}, that belongs to the Generalized Linear Model (GLM) class. Therefore we assume the $\nsamples$ samples are drawn according to
\begin{align}
	\vec{y} = \varphi_{\out}^\star\(\frac{1}{\sqrt{\ndim}} \mat{X} \vec{w}^\star \) \Leftrightarrow \vec{y} \sim P_{\out}^\star \(.\) \,,
\label{appendix:teacher_model}	
\end{align}
where $\vec{w}^\star \in \bbR^\ndim$ denotes the ground truth vector drawn from a probability distribution $P_{\w^\star}$ with second moment $\rho_{\w^\star}\equiv \frac{1}{\ndim}\EE\[ \|\vec{w}^\star\|_2^2 \]$ and $\varphi_{\out}^\star$ represents a deterministic or stochastic activation function equivalently associated to a distribution $P_{\out}^\star$. The input data matrix $\mat{X}=\(\vec{x}_\mu\)_{\mu=1}^\nsamples \in \bbR^{\nsamples \times \ndim}$ contains \iid Gaussian vectors, i.e $\forall \mu \in [1:\nsamples],~\vec{x}_\mu \sim \mN\(\vec{0},\mat{I}_\ndim\)$. 

\subsection{Bayes-optimal and ERM estimation}
Inferring the above statistical model from observations $\{\vec{y}, \mat{X}\}$ can be tackled in several ways. In particular, Bayesian inference provides a generic framework for statistical estimation based on the high-dimensional, often intractable, posterior distribution
\begin{align}
	\bbP\(\vec{w} | \vec{y}, \mat{X}\)  &= \frac{\bbP\(\vec{y} | \vec{w}, \mat{X} \) \bbP\(\vec{w}\) }{\bbP\(\vec{y},\mat{X}\)}   \,.
	\label{appendix:posterior_distribution}
\end{align}
Estimating the average of the above posterior distribution in the case we have access to the ground truth prior distributions $\bbP\(\vec{y} | \vec{w}, \mat{X} \) = P_{\out^\star}\(\vec{y} | \vec{z}\)$ with $\vec{z}\equiv\frac{1}{\sqrt{\ndim}}\mat{X}\vec{w}$ and $\bbP\(\vec{w}\) = P_{\w^\star}\(\vec{w}\)$, refers to Bayes-optimal estimation and leads to the corresponding Minimal Mean-Squared Error (MMSE) estimator $\hat{\vec{w}}_{\rm mmse} = \EE_{\bbP\(\vec{w} | \vec{y}, \mat{X}\)}\[\vec{w}\]$. It has been rigorously analyzed in details in \cite{Barbier5451} for the whole GLM class eq.~\eqref{appendix:teacher_model}.
Another celebrated approach and widely used in practice is the Empirical Risk Minimization (ERM) that minimizes instead a regularized loss: $\hat{\vec{w}}_{\rm erm} =  \argmin_{\vec{w}} \[ \mL\(\vec{w}; \vec{y}, \mat{X}\) \]$ with
\begin{align}
	 \mL\(\vec{w}; \vec{y}, \mat{X}\) = \sum_{\mu=1}^\nsamples l\(\vec{w}; y_\mu, \vec{x}_\mu \) +   r\(\vec{w}\) \,.
	\label{appendix:loss}
\end{align}
Interestingly analyzing the ERM estimation may be included in the above Bayesian framework. Indeed exponentiating eq.~\eqref{appendix:loss}, we see that minimizing the loss $\mL$ is equivalent to maximize the posterior distribution $\bbP\(\vec{w} | \vec{y}, \mat{X}\) = e^{- \mL\(\vec{w}; \vec{y}, \mat{X}\)}$ if we choose carefully the prior distributions as functions of the regularizer $r$ and the loss $l$:
\begin{align}
\begin{aligned}
	 - \log \bbP\(\vec{y} | \vec{w}, \mat{X} \) &= l\(\vec{w}; \vec{y}, \mat{X} \) \,, &&-\log \bbP\(\vec{w}\) &= r\(\vec{w}\)\,.
\end{aligned}
\label{appendix:ERM_map_mapping}
\end{align}
Computing the maximum of the posterior $\bbP\(\vec{y} | \vec{w}, \mat{X} \)$ refers instead to the so-called Maximum A Posteriori (MAP) estimator, and therefore analyzing the empirical minimization of \eqref{appendix:loss} is equivalent to obtain the performance of the MAP estimator with prior distributions given by \eqref{appendix:ERM_map_mapping}. Thus both the study of ERM (MAP) and Bayes-optimal (MMSE) estimations are simply reduced to the analysis of the posterior eq.~\eqref{appendix:posterior_distribution}. 

\subsection{Denoising distributions and updates}
\label{appendix:sec:definitions:update_generic}
Analyzing the posterior distribution eq.~\eqref{appendix:posterior_distribution} in the high-dimensional regime \cite{Barbier5451} will boil down to introducing the scalar denoising distributions $Q_{\w}, Q_{\out}$ and their respective normalizations $\mZ_{\w}$, $\mZ_{\out}$
\begin{align}
	\begin{aligned}
		Q_{\w}(w; \gamma,\Lambda) &\equiv \displaystyle \frac{P_{\w}(w) }{\mZ_{\w} (\gamma,\Lambda)} e^{ - \frac{1}{2} \Lambda w^2  + \gamma w  }\,, && Q_{\out} (z;y, \omega,V) \equiv  \displaystyle\frac{P_{\out}\(y|z\) }{\mZ_{\out}(y,\omega,V)}  \frac{e^{ -\frac{1}{2}V^{-1}  \(z - \omega\)^2  }}{\sqrt{2\pi V}}\,,  \spacecase
		\mZ_{\w} (\gamma,\Lambda) &\equiv \EE_{w \sim P_{\w}} \[ e^{ - \frac{1}{2} \Lambda w^2  + \gamma w  } \]\,, && \mZ_{\out}(y,\omega,V) \equiv \EE_{z \sim \mN(0,1)}\[ P_{\out}\(y| \sqrt{V} z + \omega\) \]  \,.
	\end{aligned}
	\label{appendix:definitions:update_generic}
\end{align}
We define as well the denoising functions, that play a central role in Bayesian inference. Note in particular that they correspond to the \emph{updates} of the Approximate Message Passing algorithm in \cite{Rangan2010} that we recalled in Sec.~\ref{appendix:optimal_loss_reg:amp}. They are defined as the derivatives of $\log \mZ_{\w}$ and $\log \mZ_{\out}$, namely
\begin{align}
\begin{aligned}
	f_{\w}(\gamma ,\Lambda) &\equiv  \partial_\gamma \log\(\mZ_{\w}\) = \EE_{Q_{\w}} \[ w \] \andcase 
	\partial_\gamma f_{\w} (\gamma ,\Lambda) \equiv  \EE_{Q_{\w}} \[w^2 \] - f_{\w}^2 \spacecase	
	f_{\out} (y,\omega,V) &\equiv \partial_\omega \log \( \mZ_{\out} \) =V^{-1}\EE_{Q_{\out}} \[ z - \omega\] \andcase
	\partial_{\omega} f_{\out} (y,\omega,V) \equiv \displaystyle \frac{\partial f_{\out}(y,\omega,V)}{\partial \omega}\,.
	\label{appendix:update_functions_generic}
\end{aligned}
\end{align}

\subsubsection{Bayes-optimal - MMSE denoising functions}
\label{appendix:sec:definitions:update_bayes}
In Bayes-optimal estimation, the ground truth prior and channel distributions $P_{\w^\star}(w)$ and $P_{\out^\star}\(y|z\)$ of the \emph{teacher} eq.~\eqref{main:teacher_model} are known. Hence, replacing $P_{\w}$ and $P_{\out}$ in \eqref{appendix:definitions:update_generic}, we obtain the Bayes-optimal scalar denoising distributions in terms of which the Bayes-optimal free entropy eq.~\eqref{appendix:free_entropy_bayes} is written
\begin{align}
	\begin{aligned}
		Q_{\w^\star}(w; \gamma,\Lambda) &\equiv \displaystyle \frac{P_{\w^\star}(w)}{\mZ_{\w^\star} (\gamma,\Lambda)}  e^{ - \frac{1}{2} \Lambda w^2  + \gamma w  } \,, && Q_{\out^\star} (z;y, \omega,V) \equiv  \displaystyle\frac{P_{\out^\star}\(y|z\)}{\mZ_{\out^\star}(y,\omega,V)}   \frac{e^{ -\frac{1}{2}V^{-1}  \(z - \omega\)^2  }}{\sqrt{2\pi V}} \,.,
	\end{aligned}
	\label{appendix:definitions:update_bayes}
\end{align}
and the denoising updates are therefore given by eq.~\eqref{appendix:update_functions_generic} with the corresponding distributions
\begin{align}
	& f_{\w^\star}(\gamma ,\Lambda) \equiv  \partial_\gamma \log\mZ_{\w^\star} (\gamma,\Lambda) \,, && f_{\out^\star} (y,\omega,V) \equiv \partial_\omega \log  \mZ_{\out^\star}(y,\omega,V) \,.
	\label{appendix:definitions:update_functions_bayes}
\end{align}

\subsubsection{ERM - MAP denoising functions}
\label{appendix:definitions:map_updates}
Before defining similar denoising functions to analyze the MAP for ERM estimation, we first recall the definition of the Moreau-Yosida regularization.
\paragraph{Moreau-Yosida regularization and proximal}
Let $\Sigma>0$, $f(,z)$ a convex function in $z$. Defining the regularized functional
\begin{align}
	\mL_{\Sigma}[f(,.)](z;x) &= f(,z) + \frac{1}{2\Sigma}\(z - x\)^2\,,
	\label{appendix:definitions:proximal_moreau:functionnal}
\end{align}
the Moreau-Yosida regularization $\mM_\Sigma$ and the proximal map $\mP_\Sigma$ are defined by
\begin{align}
	\mP_{\Sigma}[f(,.)](x) &= \argmin_z \mL_{\Sigma}[f(,.)](z;x) = \argmin_z \[f(,z) + \frac{1}{2\Sigma}\(z - x\)^2\] \,, \label{appendix:definitions:proximal}\\
	\mM_{\Sigma}[f(,.)](x) &= \min_z \mL_{\Sigma}[f(,.)](z;x) = \min_z \[f(,z) + \frac{1}{2\Sigma}\(z - x\)^2\] \,, 
	\label{appendix:definitions:moreau}
\end{align}
where $(,z)$ denotes all the arguments of the function $f$, where $z$ plays a central role. The MAP denoising functions for any convex loss $l(,.)$ and convex separable regularizer $r(.)$ can be written in terms of the Moreau-Yosida regularization or the proximal map as follows
\begin{align}
\begin{aligned}
	f_{\rm w}^{{\rm map}, r}(\gamma, \Lambda) &\equiv  \mP_{\Lambda^{-1}}\[ r(.) \](\Lambda^{-1}\gamma) = \Lambda^{-1}\gamma - \Lambda^{-1} \partial_{\Lambda^{-1}\gamma}\mM_{\Lambda^{-1}}\[ r(.) \] (\Lambda^{-1}\gamma)\,, \spacecase
	f_{\out}^{{\rm map}, l} (y, \omega, V) &\equiv  - \partial_{\omega} \mM_{V}[l(y,.)](\omega) = V^{-1} \( \mP_{V}[l(y,.)](\omega) - \omega\)
	\label{appendix:definitions:update_functions_map}\,.
\end{aligned}
\end{align}
The above updates can be considered as definitions, but it is instructive to derive them from the generic definition of the denoising distributions eq.~\eqref{appendix:update_functions_generic} if we maximize the posterior distribution. This is done by taking, in a physics language, a \emph{zero temperature} limit and we present it in details in the next paragraph.

\paragraph{Derivation of the MAP updates}
\label{appendix:definitions:map_updates:derivation}
To have access to the maximum of the generic distributions eq.~\eqref{appendix:definitions:update_generic}, we introduce a \emph{fictive} noise/temperature $\Delta$ or inverse temperature $\beta$, $\Delta = \frac{1}{\beta}$. In particular for Bayes-optimal estimation this temperature is finite and fixed to $\Delta=\beta=1$. 
Indeed with the mapping eq.~\eqref{appendix:ERM_map_mapping}, minimizing the loss function $\mL$ \eqref{appendix:loss} is equivalent to maximize the posterior distribution. Therefore it can be done by taking the \emph{zero noise/temperature} limit $\Delta \to 0$ of the channel and prior denoising distributions $Q_{\out}$ and $Q_{\w}$. It is the purpose of the following paragraphs where we present the derivation leading to the result \eqref{appendix:definitions:update_functions_map}. 

\paragraph{Channel}
Using the mapping eq.~\eqref{appendix:ERM_map_mapping}, we assume that the channel distribution can be expressed as $\bbP\(y | z\) \propto e^{-l\(y, z\)}$. 
Therefore we introduce the corresponding channel distribution $P_\out$ at finite temperature $\Delta$ associated to the convex loss $l(y,z)$ 
\begin{align*}
	P_\out^{\rm map} \(y | z\) &= \frac{e^{-\frac{1}{\Delta}l(y, z)}}{\sqrt{2\pi \Delta}}\,.
\end{align*}
 Note that the case of the square loss $l(y,z)=\frac{1}{2}\(y-z\)^2$ is very specific. Its channel distribution simply reads $P_\out \(y | z\) = \frac{e^{-\frac{1}{2\Delta}(y-z)^2}}{\sqrt{2\pi \Delta}} $ and is therefore equivalent to predict labels $y$ according to a noisy Gaussian linear model $y= z + \sqrt{\Delta} \xi$, where $\xi \sim \mN(0, 1)$ and $\Delta$ denotes therefore the \emph{real} noise of the model.

In order to obtain a non trivial limit and a closed set of equations when $\Delta \to 0$, we must define rescaled variables as follows:
\begin{align*}
	V_\dag &\equiv \lim_{\Delta \to 0} \frac{V}{\Delta}\,,  && f_{\out,\dag}^{\rm map} (y, \omega,V_\dag) \equiv  \lim_{\Delta \to 0} \Delta \times  f_{\out}^{\rm map}(y, \omega,V)  \,,
\end{align*}
where we denote the rescaled quantities after taking the limit $\Delta \to 0$ by $\dagger$.
Similarly to eq.~\eqref{appendix:definitions:proximal_moreau:functionnal}, we introduce therefore the rescaled functional
\begin{align}
	\mL_{V_\dag}[l(y,.)](z;\omega) &= l(y,z) + \frac{1}{2V_\dag}\(z - \omega\)^2 \,,
\end{align}
such that, injecting $P_\out^{\rm map}$, the channel denoising distribution $Q_{\out}^{\rm map}$ and the corresponding partition function $\mZ_\out^{\rm map}$ eq.~\eqref{appendix:definitions:update_generic} simplify in the zero temperature limit as follows:
\begin{align}
	Q_{\out}^{\rm map} \(z;y, \omega,V\)&\equiv \lim_{\Delta \to 0} \frac{e^{-\frac{1}{\Delta} l(y,z) + \frac{1}{2V}\(z - \omega\)^2 } }{\sqrt{2\pi \Delta V_\dag }\sqrt{2\pi \Delta}}  = \lim_{\Delta \to 0} \frac{e^{-\frac{1}{\Delta} \mL_{V_\dag}[l(y,.)](z;\omega) } }{\sqrt{2\pi \Delta V_\dag }\sqrt{2\pi \Delta}}\,, \\
		& \propto \delta\(z - \mP_{V_\dag}[l(y,.)](\omega) \) \nonumber \\
	\mZ_\out^{\rm map} \(y, \omega,V\) &= \lim_{\Delta \to 0} \int_{\bbR} \d z  Q_{\out}^{\rm map} (z;y, \omega,V)  = \lim_{\Delta \to 0} \frac{e^{-\frac{1}{\Delta} \mM_{V_\dag}[l(y,.)](\omega) } }{\sqrt{2\pi \Delta V_\dag }\sqrt{2\pi \Delta}}\,,
\end{align}
that involve the proximal map and the Moreau-Yosida regularization defined in eq.~\eqref{appendix:definitions:moreau}. 
Finally taking the zero temperature limit, the MAP denoising function $f_{\out,\dag}^{\rm map}$ leads to the result \eqref{appendix:definitions:update_functions_map}:
\begin{align}
\begin{aligned}
	f_{\out,\dag}^{\rm map} (y, \omega,V_\dag) &\equiv \lim_{\Delta \to 0} \Delta \times f_{\out}^{\rm map} (y, \omega, V) \\
	& \equiv \lim_{\Delta \to 0} \Delta \times \partial_{\omega} \log \mZ_\out^{\rm map} \equiv \lim_{\Delta \to 0} \Delta V^{-1} \EE_{Q_{\out}^{\rm map}}\[ z- \omega \] \\
	&=  - \partial_{\omega} \mM_{V_\dag}[l(y,.)](\omega) = V_\dag^{-1} \( \mP_{V_\dag}[l(y,.)](\omega) - \omega\)\,.
\end{aligned}
\end{align}

\paragraph{Prior}
Similarly as above, using the mapping eq.~\eqref{appendix:ERM_map_mapping}, for a convex and separable regularizer $r$, the corresponding prior distribution at temperature $\Delta$ can be written
\begin{align*}
	P_\w^{\rm map} \(w\) &= e^{-\frac{1}{\Delta} r(w)}\,.
\end{align*}
Note that at $\Delta=1$ the classical $\rL_1$ regularization with strength $\lambda$, $r^{\rL_1}(w)=-\lambda |w|$, and the $\rL_2$ regularization $r^{\rL_2}(w)=-\lambda  w^2 /2$ are equivalent to choosing a Laplace prior $P_\w(w) \propto e^{- \lambda |w|}$ or a Gaussian prior $P_\w(w) \propto e^{-\frac{\lambda w^2}{2}}$. To obtain a meaningful limit as $\Delta \to 0$, we again introduce the following rescaled variables
\begin{align*}
\Lambda_\dag \equiv \lim_{\Delta \to 0} \Delta \times  \Lambda \,, && \gamma_\dag \equiv \lim_{\Delta \to 0} \Delta \times \gamma \,,
\end{align*}
and the functional
\begin{align}
\begin{aligned}
	\mL_{\Lambda_\dag^{-1}}\[ r(.) \](w;\Lambda_\dag^{-1}\gamma_\dag) &= r(w) + \frac{1}{2}\Lambda_\dag\( w - \Lambda_\dag^{-1}\gamma_\dag \)^2=  \[ r(w) + \frac{1}{2}\Lambda_\dag w^2 - \gamma_\dag w \] + \frac{1}{2} \gamma_\dag^2 \Lambda_\dag^{-1}   \,,
\end{aligned}
\end{align}
such that in the zero temperature limit, the prior denoising distribution $Q_{\w}^{\rm map}$ and the partition function $\mZ_\w^{\rm map}$ reduce to
\begin{align}
	Q_{\w}^{\rm map}  \(w;\gamma, \Lambda\)  &\equiv \lim_{\Delta \to 0} P_{\w}(w) e^{ - \frac{1}{2} \Lambda w^2  + \gamma w  } = \lim_{\Delta \to 0} e^{-\frac{1}{\Delta} \mL_{\Lambda_\dag^{-1}}\[ r \](w;\Lambda_\dag^{-1}\gamma_\dag) }e^{-\frac{1}{2\Delta} \gamma_\dag^2 \Lambda_\dag^{-1}} \nonumber \\
	&\propto \delta\(w - \mP_{\Lambda_\dag^{-1}}\[ r \](\Lambda_\dag^{-1}\gamma_\dag) \)\\
	\mZ_\w^{\rm map}  \(y, \omega,V\) &= \lim_{\Delta \to 0} \int_{\bbR} \d w  Q_{\w}^{\rm map} (w;\gamma, \Lambda) = \lim_{\Delta \to 0} e^{-\frac{1}{\Delta} \mM_{\Lambda_\dag^{-1}}\[ r \](\Lambda_\dag^{-1}\gamma_\dag) } e^{-\frac{1}{2\Delta} \gamma_\dag^2 \Lambda_\dag^{-1}}\,,
\end{align}
that involve again the proximal map $\mP_{\Lambda_\dag^{-1}}$ and the Moreau-Yosida regularization $\mM_{\Lambda_\dag^{-1}}$ defined in eq.~\eqref{appendix:definitions:moreau}. Finally the MAP denoising update $f_{\rm w, \dag}^{\rm map}$ is simply given by:
\begin{align}
	f_{\rm w, \dag}^{\rm map}(\gamma_\dag, \Lambda_\dag) &\equiv \lim_{\Delta \to 0} f_{\rm w}^{\rm map}(\gamma,\Lambda) = \lim_{\Delta \to 0} \partial_\gamma \log \mZ_{\w}^{\rm map} \equiv \lim_{\Delta \to 0} \EE_{Q_{\rm w}^{\rm map}}\[w\] \nonumber \\
	&= \lim_{\Delta \to 0} \partial_{\gamma} \( -\frac{1}{\Delta} \mM_{\Lambda_\dag^{-1}}\[ r(.) \](\Lambda_\dag^{-1}\gamma_\dag) -\frac{1}{2 \Delta} \gamma_\dag^2 \Lambda_\dag^{-1} \) \nonumber\\
	&= \partial_{\gamma_\dag} \( - \mM_{\Lambda_\dag^{-1}}\[ r(.) \](\Lambda_\dag^{-1}\gamma_\dag) -\frac{1}{2} \gamma_\dag^2 \Lambda_\dag^{-1} \)\\
	&= \Lambda_\dag^{-1}\gamma_\dag - \Lambda_\dag^{-1} \partial_{\Lambda_\dag^{-1}\gamma_\dag}\mM_{\Lambda_\dag^{-1}}\[ r(.) \] (\Lambda_\dag^{-1}\gamma_\dag) = \mP_{\Lambda_\dag^{-1}}\[ r(.) \](\Lambda_\dag^{-1}\gamma_\dag) \nonumber\\
	&= \argmin_{w} \[ r(w) + \frac{1}{2}\Lambda_\dag(w - \Lambda_\dag^{-1}\gamma_\dag )^2  \] = \argmin_{w} \[ r(w) + \frac{1}{2}\Lambda_\dag w^2 - \gamma_\dag w  \] \,, \nonumber
\end{align}
and we recover the result \eqref{appendix:definitions:update_functions_map}.

\subsection{Applications}
\label{appendix:bayes_map_updates}

In this section we list the explicit expressions of the Bayes-optimal eq.~\eqref{appendix:definitions:update_functions_bayes} and ERM eq.~\eqref{appendix:definitions:update_functions_map} denoising functions largely used to produce the examples in Sec.~\ref{sec:applications}.

\subsubsection{Bayes-optimal updates}
The Bayes-optimal denoising functions \eqref{appendix:definitions:update_functions_bayes} are detailed in the case of a \emph{linear}, \emph{sign} and \emph{rectangle door} channel with a Gaussian noise $\xi \sim \mN(0,1)$ and variance $\Delta \geq 0$, and for \emph{Gaussian} and \emph{sparse-binary} weights.
\paragraph{Channel\\}
\subparagraph{Linear: $y = \varphi_{\out^\star}(z) = z + \sqrt{\Delta}\xi$}
\begin{align}
	\begin{aligned}
		\mZ_{\out^\star} (y,\omega,V)  &=  \mN_\omega \( y , \Delta^\star +V  \)\,,\spacecase
		f_{\out^\star}  (y,\omega,V) &= \(\Delta^\star +V\)^{-1} \( y - \omega\)\,, 
		&&\partial_\omega f_{\out^\star} (y,\omega,V)  = - \(\Delta^\star +V\)^{-1}\,.
	\end{aligned}
	\label{appendix:definitions:application:linear}
\end{align}

\subparagraph{Sign: $y = \varphi_{\out^\star}(z) = \sign(z) + \sqrt{\Delta^\star}\xi$}
\begin{align}
	\begin{aligned}
		\mZ_{\out^\star} \( y,  \omega, V \) &=  \displaystyle \mN_y(1,\Delta^\star) \frac{1}{2} \(1 + \erf\(\frac{\omega}{\sqrt{2V}} \) \) + \mN_y(-1,\Delta^\star) \frac{1}{2} \(1 - \erf\(\frac{\omega}{\sqrt{2V}} \) \)\,, \spacecase
		f_{\out^\star} \( y,  \omega, V   \) &= \frac{ \mN_y(1,\Delta^\star)  - \mN_y(-1,\Delta^\star) }{\mZ_{\out^\star} \( y,  \omega, V \)} \mN_\omega(0,V)\,.
	\end{aligned}	
	\label{appendix:definitions:application:sign}
\end{align}

\subparagraph{Rectangle door: $y = \varphi_{\out^\star}(z)= \id \( \kappa_m \leq z \leq \kappa_M \) - \id \( z \leq \kappa_m \right.$ or $\left. z \geq \kappa_M \) + \sqrt{\Delta^\star}\xi$}
For $\kappa_m < \kappa_M$, we obtain
\begin{align}
	\begin{aligned}
		\mZ_{\out^\star} (y,\omega,V) &= \displaystyle \mN_y(1,\Delta^\star) \frac{1}{2} \(  \erf\(\frac{\kappa_M - \omega}{\sqrt{2V}} \) - \erf\(\frac{\kappa_m-\omega}{\sqrt{2V}} \) \) \\
		& + \mN_y(-1,\Delta^\star) \frac{1}{2} \(1 - \frac{1}{2}\(  \erf\(\frac{\kappa_M - \omega}{\sqrt{2V}} \) - \erf\(\frac{\kappa_m-\omega}{\sqrt{2V}} \) \) \)\,, \spacecase
		f_{\out^\star} (y,\omega,V) &= \frac{1}{\mZ_\out} \( \mN_y(1,\Delta^\star) \( - \mN_\omega(\kappa_M,V) + \mN_\omega(\kappa_m,V)  \) \right.\\
		& \hspace{2cm} \left. + \mN_y(-1,\Delta^\star) \( \mN_\omega(\kappa_M,V) - \mN_\omega(\kappa_m,V) \)\)\,.
	\end{aligned}
	\label{appendix:definitions:application:door}
\end{align}

\paragraph{Prior\\}
\subparagraph{Gaussian weights: $w\sim P_{\w}(w) = \mN_w(\mu,\sigma)$}
\begin{align}
	\begin{aligned}
		\mZ_{\w^\star}(\gamma,\Lambda) &= \frac{e^{\frac{\gamma^2 \sigma + 2 \gamma  \mu -\Lambda  \mu ^2}{2\( \Lambda  \sigma +1 \)}}}{\sqrt{\Lambda \sigma + 1 }}\,, && f_{\w^\star}(\gamma,\Lambda) = \frac{\gamma  \sigma +\mu }{1 + \Lambda  \sigma}\,, &&
		\partial_\gamma f_{\w^\star}(\gamma,\Lambda) = \frac{\sigma }{1 + \Lambda  \sigma}\,.
	\end{aligned}
	\label{appendix:definitions:application:gaussian}
\end{align}

\subparagraph{Sparse-binary weights: $w \sim P_{\w}(w) = \rho \delta(w) + (\rho-1) \frac{1}{2} \( \delta(w-1) +\delta(w+1) \)$}
 \begin{align}
	\begin{aligned}
	\mZ_{\w^\star}(\gamma,\Lambda) &=\rho + e^{-\frac{\Lambda }{2}} (1 -\rho) \cosh (\gamma)\,, \spacecase
	f_{\w^\star}(\gamma,\Lambda) &= \frac{e^{-\frac{\Lambda }{2}} (1-\rho) \sinh (\gamma )}{\rho + e^{-\frac{\Lambda }{2}} (1 -\rho) \cosh (\gamma)}\,, &&
	\partial_\gamma f_{\w^\star}(\gamma,\Lambda) &=  \frac{e^{-\frac{\Lambda }{2}} (1-\rho) \cosh (\gamma )}{\rho + e^{-\frac{\Lambda }{2}} (1 -\rho) \cosh (\gamma)}\,.
	\end{aligned}
	\label{appendix:definitions:application:binary}
\end{align}

\subsubsection{ERM updates}
The ERM denoising functions \eqref{appendix:definitions:update_functions_map} have, very often, no explicit expression except for the \emph{square} and \emph{hinge} losses, and for $\rL_1$, $\rL_2$ regularizations that are analytical. However, in the particular case of a two times differentiable convex loss the denoising functions can still be written as the solution of an implicit equation detailed below. 

\paragraph{Convex losses\\}
\subparagraph{Square loss}
The proximal map for the square loss $l^{\rm square}(y,z)= \frac{1}{2} (y-z)^2$ is easily obtained and reads 
\begin{align*}
	\mP_{V}\[\frac{1}{2}(y,.)^2\](\omega) =  \argmin_z \[\frac{1}{2}\(y-z\)^2 + \frac{1}{2V}\(z - \omega\)^2\] = \(1+V\)^{-1}\(\omega + y V\)\,.
\end{align*}
Therefore \eqref{appendix:definitions:update_functions_map} yields
\begin{align}
\begin{aligned}
	f_{\out}^{\rm square} (y, \omega, V) &= V^{-1} \( \mP_{V}\[\frac{1}{2}(y,.)^2\](\omega) - \omega\) = \(1 + V\)^{-1}\(y - \omega\)\,,\\
	\partial_\omega f_{\out}^{\rm square} (y, \omega, V) &= - \(1 + V\)^{-1}.
	\label{appendix:definitions:application:square}
\end{aligned}
\end{align}

\subparagraph{Hinge loss}
The proximal map of the hinge loss $l^{\rm hinge}(y,z)= \max\(0, 1 - yz\)$ 
\begin{align*}
	\mP_{V}\[ l^{\rm hinge}(y,.) \](\omega) &=  \argmin_z \[\underbrace{ \max\(0, 1 - yz\) + \frac{1}{2V}\(z - \omega\)^2}_{\equiv\mL_0} \] \equiv z^{\star}(y, \omega, V)\,.
\end{align*}
can be expressed analytically by distinguishing all the possible cases:
\begin{itemize}
	\item $1-yz<0$: $\mL_0 = \frac{1}{2V}\(z - \omega\)^2 \Rightarrow$ $z^\star= \omega$ if $yz^\star<1 \Leftrightarrow z^\star= \omega$ if $ \omega y<1$.
	\item $1-yz>0$: $\mL_0 = \frac{1}{2V}\(z - \omega\)^2 + 1 - yz \Rightarrow (z^\star-\omega)=yV \Leftrightarrow z^\star =  \omega  + Vy $ if $1-yz^\star>0 \Leftrightarrow z^\star =  \omega  + Vy$ if $   \omega y < 1  - y^2V = 1- V$, as $y^2=1$.
	\item Hence we have one last region to study $1 - V< \omega y < 1$. It follows $y(1 - V)< \omega < y $:
	$$ \frac{1}{2V}\(z - y\)^2 \leq \frac{1}{2V}\(z - \omega\)^2 \Rightarrow z^\star = y\,. $$
\end{itemize}
Finally we obtain a simple analytical expression for the proximal and its derivative
\begin{align*}
\begin{aligned}
	\mP_{V}\[l^{\rm hinge}(y,.)\](\omega) &= 
		\begin{cases}
		\omega + V y \textrm{ if } \omega y < 1 - V \\ 
		y \textrm{ if }  1 - V< \omega y < 1 \\ 
		\omega  \textrm{ if } \omega y > 1 \\ 	
		\end{cases} \hspace{-0.5cm} \,, \partial_\omega \mP_{V}\[l^{\rm hinge}(y,.)\](\omega) = 
		\begin{cases}
			1  \textrm{ if } \omega y < 1 - V \\ 
			0 \textrm{ if }  1 - V< \omega y < 1 \\ 
			1 \textrm{ if } \omega y > 1 \\ 	
		\end{cases}	\,.
	\end{aligned}	
\end{align*}
Hence with \eqref{appendix:definitions:update_functions_map}, the hinge denoising function and its derivative read
\begin{align}
\begin{aligned}
	f_\out^{\rm hinge}\(y, \omega, V\) = 
		\begin{cases}
		y \textrm{ if } \omega y < 1 - V \\ 
		\frac{(y-\omega)}{V} \textrm{ if }  1 - V< \omega y < 1 \\ 
		0 \textrm{ otherwise } \\	
		\end{cases}	
		\hspace{-0.4cm} \,,
	\partial_\omega f_\out^{\rm hinge}\(y, \omega, V\) = 
		\begin{cases}
			-\frac{1}{V} \textrm{ if }  1 - V< \omega y < 1 \\ 
			0 \textrm{ otherwise } \\ 	
		\end{cases}	
	\end{aligned}\,.
	\label{appendix:definitions:application:hinge}
\end{align}

\subparagraph{Generic differentiable convex loss}
In general, finding the proximal map in \eqref{appendix:definitions:update_functions_map} is intractable. In particular, it is the case for the logistic loss considered in Sec.~\ref{appendix:applications:logistic}. However assuming the convex loss is a generic two times differentiable function $l \in \mD^2$, taking the derivative of the proximal map
\begin{align*}
	\mP_{V}\[l(y,.)\](\omega) &=  \argmin_z \[l\(y, z\) + \frac{1}{2V}\(z - \omega\)^2\] \equiv z^{\star}(y, \omega, V)\,,
\end{align*}
verifies therefore the implicit equations:
\begin{align}
	 z^{\star}(y, \omega, V) &= \omega - V \partial_z l\(y, z^{\star}(y, \omega, V) \)\,,&& \partial_\omega z^{\star} (y, \omega, V) = \( 1 + V \partial^2_{z} l(y,  z^{\star}(y, \omega, V) )  \)^{-1} \,.
	 \label{appendix:definitions:application:proximal_differentiable}
\end{align}
Once those equations solved, the denoising function and its derivative are simply expressed as
\begin{align}
		f_{\out}^{\rm diff} \( y, \omega, V  \) &= V^{-1} (z^\star \( y, \omega, V  \) - \omega) \,,&& \partial_\omega f_{\out}^{\rm diff} \( y, \omega, V  \)   = V^{-1} \( \partial_\omega z^\star \( y, \omega, V  \) - 1 \) \,,
\label{appendix:definitions:application:differentiable}
\end{align}	
with $z^\star \( y, \omega, V  \)=\mP_{V}\[l(y,.)\](\omega)$ solution of \eqref{appendix:definitions:application:proximal_differentiable}.

\paragraph{Regularizations\\}
\subparagraph{$\rL_2$ regularization}
Using the definition of the prior update in eq.~\eqref{appendix:definitions:update_functions_map} for the $\rL_2$ regularization $r(w) = \frac{\lambda w^2}{2}$, we obtain
\begin{align}
\label{appendix:definitions:application:L2}
\begin{aligned}
	f_{\rm w}^{\rL_2}(\gamma, \Lambda) &= \argmin_{w} \[ \frac{\lambda w^2}{2} + \frac{1}{2}\Lambda w^2 - \gamma w  \] = \frac{\gamma}{\lambda + \Lambda}\,, \\
	\partial_\gamma f_{\rm w}^{\rL_2}(\gamma, \Lambda) &= \frac{1}{\lambda + \Lambda} ~~\text{ and }~~ \mZ_{\w}^{\rL_2}(\gamma, \Lambda) = \exp\(\frac{\gamma ^2 \Lambda }{2 (\lambda +\Lambda )^2}  \)\,.
\end{aligned}
\end{align}

\subparagraph{$\rL_1$ regularization}
Performing the same computation for the $\rL_1$ regularization $r(w) = \lambda |w|$, we obtain
\begin{align}
\begin{aligned}
	f_{\rm w}^{\rL_1}(\gamma, \Lambda) &= \argmin_{w} \[ \lambda \|w\| + \frac{1}{2}\Lambda w^2 - \gamma w  \] = 
	\begin{cases}
\frac{\gamma -\lambda }{\Lambda } & \gamma >\lambda  \\
 \frac{\gamma +\lambda }{\Lambda } & \gamma +\lambda <0 \\
 0 \textrm{ otherwise}
	\end{cases}\,,\\
	\partial_{\gamma} f_{\rm w}^{\rL_1}(\gamma, \Lambda) &= \begin{cases}
\frac{1}{\Lambda } & \|\gamma\| >\lambda \\
 0 \textrm{ otherwise}
	\end{cases}\,.
	\end{aligned}
\label{appendix:definitions:application:L1}
\end{align}


%% file: files/supplementary/generalization_error.tex
\label{appendix:generalization_error}

In this section, we present the computation of the asymptotic generalization error
\begin{align}
e_{\rm g}(\alpha) \equiv \lim_{\ndim \to \infty} \EE_{y, \vec{x}}  \id\[ y \ne \hat{y}\(\hat{\vec{w}}(\alpha); \vec{x} \) \]\,,
\end{align}
 leading to expressions in Proposition.~\ref{main:thm:generalization_errors} and Thm.~\ref{main:thm:fixed_point_equations_bayes}. The computation at finite dimension is similar if we do not consider the limit $d\to \infty$.

\subsection{General case}
\label{appendix:sec:generalization_error:general}

The generalization error $e_{\rm g}$ is the prediction error of the estimator $\hat{\vec{w}}$ on new samples $\{\vec{y}, \mat{X}\}$, where $\mat{X}$ is an \iid Gaussian matrix and $\vec{y}$ are $\pm 1$ labels generated according to \eqref{appendix:teacher_model}:
\begin{align}
	\vec{y} = \varphi_{\out^\star}\(\vec{z}\) \hhspace \textrm{ with } \hhspace \vec{z} = \frac{1}{\sqrt{\ndim}} \mat{X} \vec{w}^\star	\,.
\end{align}
 As the model fitted by ERM may not lead to binary outputs, we may add a non-linearity $\varphi:\bbR \mapsto \{\pm 1\}$ (for example a sign) on top of it to insure to obtain binary outputs $\hat{\vec{y}} = \pm 1$ according to
 \begin{align}
 		\hat{\vec{y}} = \varphi\(\hat{\vec{z}}\) \hhspace \textrm{ with } \hhspace \hat{\vec{z}} = \frac{1}{\sqrt{\ndim}} \mat{X} \hat{\vec{w}} \,.
 \end{align}
The classification generalization error is given by the probability that the predicted labels $\hat{y}$ and the true labels $y$ do not match. To compute it, first note that the vectors $(\vec{z},\hat{\vec{z}})$
averaged over all possible ground truth vectors $\vec{w}^\star$ (or equivalently labels $y$) and input matrix $\mat{X}$
follow in the large size limit a joint Gaussian distribution with zero mean and covariance matrix 
\begin{align}
	\sigma = \lim_{\ndim \to \infty} \EE_{\vec{w}^\star, \mat{X}} \frac{1}{\ndim} 
\begin{bmatrix}
	\vec{w}^{\star \intercal} \vec{w}^\star & \vec{w}^{\star \intercal} \hat{\vec{w}} \\
	\vec{w}^{\star \intercal} \hat{\vec{w}} & \hat{\vec{w}}^{\intercal} \hat{\vec{w}}
\end{bmatrix}
\equiv
\begin{bmatrix}
	\sigma_{\w^\star} & \sigma_{\w^\star \hat{\w}} \\
	\sigma_{\w^\star \hat{\w}} & \sigma_{\hat{\w}}
\end{bmatrix}\,.
\end{align} 
The asymptotic generalization error depends only on the covariance matrix $\sigma$ and as the samples are \iid it reads
\begin{align}
\begin{aligned}
	e_{\rm g}(\alpha) &=\lim_{\ndim \to \infty} \EE_{y, \vec{x}}  \id\[ y \ne \hat{y}\(\hat{\vec{w}}(\alpha); \vec{x} \) \] = 1 - \bbP[y = \hat{y}\(\hat{\vec{w}}(\alpha); \vec{x} \)] =  1 - 2  \int_{\(\bbR^+\)^{2}} d\vec{x} \mN_{\vec{x}} \( \vec{0}, \sigma\)  \\
	&=  1 - \( \frac{1}{2}  + \frac{1}{\pi}  \atan\(\sqrt{\frac{\sigma_{\w^\star\hat{\w}}^2}{\sigma_{\w^\star} \sigma_{\hat{\w}} - \sigma_{\w^\star \hat{\w}}^2}}\) \)  = \frac{1}{\pi} \acos\(\frac{\sigma_{\w^\star\hat{\w}}}{\sqrt{\sigma_{\w^\star} \sigma_{\hat{\w}}}}\) \,,
\end{aligned}
\end{align}
where we used the fact that $\atan(x) = \frac{\pi}{2} - \frac{1}{2}\acos\(\frac{x^2-1}{1+x^2}\)$ and $\frac{1}{2} \acos(2x^2-1) = \acos(x)$. Finally
\begin{align}
	e_{\rm g}(\alpha) &\equiv \lim_{\ndim \to \infty} \EE_{y, \vec{x}}  \id\[ y \ne \hat{y}\(\hat{\vec{w}}(\alpha); \vec{x} \) \] =\frac{1}{\pi} \acos\(\frac{\sigma_{\w^\star\hat{\w}}}{\sqrt{\rho_{\w^\star} \sigma_{\hat{\w}}}}\)\,, 
	\label{appendix:generalization_error:general}
\end{align}
with
\begin{align*}
 	\sigma_{\w^\star\hat{\w}} &\equiv \lim_{\ndim \to \infty} \EE_{\vec{w}^\star,\mat{X}} \frac{1}{\ndim} \hat{\vec{w}}^\intercal \vec{w}^\star \,, 
 	&& \rho_{\w^\star} \equiv \lim_{\ndim \to \infty}  \EE_{\vec{w}^\star} \frac{1}{\ndim}  \|\vec{w}^\star\|_2^2 \,, 
 	&& \sigma_{\hat{\w}} \equiv \lim_{\ndim \to \infty}  \EE_{\vec{w}^\star,\mat{X}} \frac{1}{\ndim} \|\hat{\vec{w}}\|_2^2\,.
\end{align*}

\subsection{Bayes-optimal generalization error}
The Bayes-optimal generalization error for classification is equal to eq.~\eqref{appendix:generalization_error:general} where the Bayes estimator $\hat{\vec{w}}$ is the average over the posterior distribution eq.~\eqref{appendix:posterior_distribution} denoted $\langle . \rangle$, knowing the teacher prior $P_{\w^\star}$ and channel $P_{\out^\star}$ distributions: $\hat{\vec{w}}=\langle \vec{w} \rangle_{\vec{w}}$. Hence the parameters $\sigma_{\hat{\w}}$ and $\sigma_{\w^\star\hat{\w}}$ read in the Bayes-optimal case
\begin{align*}
	\sigma_{\hat{\w}} &\equiv \lim_{\ndim \to \infty}  \EE_{\vec{w}^\star,\mat{X}} \frac{1}{\ndim} \|\hat{\vec{w}}\|_2^2 = \lim_{\ndim \to \infty}  \EE_{\vec{w}^\star,\mat{X}} \frac{1}{\ndim} \|\langle \vec{w} \rangle_{\vec{w}}\|_2^2  \equiv q_{\rm b}\,,\\
	\sigma_{\w^\star\hat{\w}} &\equiv \lim_{\ndim \to \infty} \EE_{\vec{w}^\star,\mat{X}} \frac{1}{\ndim} \hat{\vec{w}}^\intercal \vec{w}^\star = \lim_{\ndim \to \infty} \EE_{\vec{w}^\star,\mat{X}} \frac{1}{\ndim} \langle \vec{w} \rangle_{\vec{w}}^\intercal \vec{w}^\star \equiv m_{\rm b} \,.
\end{align*}
Using Nishimori identity \cite{Nishimori_1980}, we easily obtain $m_{\rm b}=q_{\rm b}$ which is solution of eq.~\eqref{main:fixed_point_equations_bayes}. Therefore the generalization error simplifies
\begin{align}
	e_{\rm g}^{\rm bayes}(\alpha) = \frac{1}{\pi} \acos\( \sqrt{\eta}_\bayes\)\,, &\text{~~with~~~} \eta_\bayes = \frac{q_\bayes}{\rho_{\w^\star}} \,.
	\label{appendix:generalization_error:bayes}
\end{align}

\subsection{ERM generalization error}
The generalization error of the ERM estimator is given again by eq.~\eqref{appendix:generalization_error:general} with parameters
\begin{align*}
	\sigma_{\hat{\w}} &\equiv \lim_{\ndim \to \infty}  \EE_{\vec{w}^\star,\mat{X}} \frac{1}{\ndim} \|\hat{\vec{w}}\|_2^2 = \lim_{\ndim \to \infty}  \EE_{\vec{w}^\star,\mat{X}} \frac{1}{\ndim} \| \hat{\vec{w}}^{\rm erm}\|_2^2  \equiv q\,,\\
	\sigma_{\w^\star\hat{\w}} &\equiv \lim_{\ndim \to \infty} \EE_{\vec{w}^\star,\mat{X}} \frac{1}{\ndim} \hat{\vec{w}}^\intercal \vec{w}^\star =  \lim_{\ndim \to \infty} \EE_{\vec{w}^\star,\mat{X}} \frac{1}{\ndim} \(\hat{\vec{w}}^{\rm erm}\)^\intercal \vec{w}^\star \equiv m \,.
\end{align*}
where the parameters $m, q$ are the asymptotic ERM overlaps solutions of eq.~\eqref{main:fixed_point_equations_replicas} and that finally lead to the ERM generalization error for classification:
\begin{align}
	e_{\rm g}^{\rm erm}(\alpha) &= \frac{1}{\pi} \textrm{acos}\( \sqrt{\eta} \)\,, 
	&& \text{with~~} \eta \equiv \frac{m^2}{\rho_{\w^\star}q}\,. 
	\label{appendix:generalization_error:erm}
\end{align}


%% file: files/supplementary/proof.tex
\label{appendix:proof}

\subsection{Gordon's result and proofs}
\label{appendix:proof:gordon}

We consider in this section that the data have been generated by a teacher \eqref{appendix:teacher_model} with Gaussian weights
\begin{align}
	\vec{w}^\star \sim P_{\w^\star}(\vec{w}^\star) = \mN_{\vec{w}^\star}\(\vec{0},\rho_{\w^\star} \rI_{\ndim}\) ~~\text{ with }~~ \rho_{\w^\star} \equiv \EE\[ (w^\star)^2 \] \,.
	\label{appendix:model:gaussian_weights}
\end{align}

\subsubsection{For real outputs - Regression with $\rL_2$ regularization}

In what follows, we prove a theorem that characterizes the asymptotic performance of empirical risk minimization 
\begin{equation}\label{eq:erm_l2}
\hat{\vec{w}}_{\rm erm} =  \argmin_{\vec{w}} \sum_{i=1}^\nsamples l\(y_i , \tfrac{1}{\sqrt{d}}\vec{x}_i^\intercal \vec{w} \) +   \frac{\lambda \norm{\vec{w}}^2}{2},
\end{equation}
where $\{y_i\}_{1 \le i \le n}$ are general real-valued outputs (that are not necessarily binary), $l(y, z)$ is a loss function that is convex with respect to $z$, and $\lambda > 0$ is the strength of the $\rL_2$ regularization. Note that this setting is more general than the one considered in Thm.~\ref{main:thm:gordon_fixed_points:classification} in the main text, which focuses on binary outputs and loss functions in the form of $l(y, z) = \ell(yz)$ for some convex function $\ell(\cdot)$.

\begin{theorem}[Regression with $\rL_2$ regularization]
	\label{apppendix:thm:gordon_fixed_points:classification}
As $\nsamples, \ndim \to \infty$ with $\nsamples/ \ndim = \alpha = \Theta(1)$, the overlap parameters $m, q$ concentrate to
\begin{align}
m & \underlim{\ndim}{\infty} \sqrt{\rho_{\w^\star}} \mu^\ast\,, && q \underlim{\ndim}{\infty} (\mu^\ast)^2 + (\delta^\ast)^2\,, 
\end{align}
where the parameters $\mu^\ast, \delta^\ast$ are the solutions of 
\begin{equation}\label{eq:pot_func_general}
(\mu^\ast, \delta^\ast) = \underset{\mu, \delta \ge 0}{\arg\min} \ \sup_{\tau > 0} \left\{\frac{\lambda(\mu^2 + \delta^2)}{2} - \frac{\delta^2}{2\tau} + \alpha \EE_{g, s} \mM_{\tau}[l(\varphi_{\out^\star}(\sqrt{\rho_{\w^\star}} s),.)](\mu s + \delta g)\right\}.
\end{equation}
Here, $\mM_{\tau}[l(,.)](x)$ is the Moreau-Yosida regularization defined in \eqref{appendix:definitions:moreau}, and $g, s$ are two \iid standard normal random variables.
\end{theorem}
\begin{proof}
	\label{appendix:proof:gordon_regression}
Since the teacher weight vector $\vec{w}^\star$ is independent of the input data matrix $\mat{X}$, we can assume without loss of generality that
\begin{align*}
\vec{w}^\star = \sqrt{d} \rho_d \vec{e}_1,
\end{align*}
where $\vec{e}_1$ is the first natural basis vector of $\bbR^d$, and $\rho_d = \norm{\vec{w}^\star}/\sqrt{d}$. As $d \to \infty$, $\rho_d \to \sqrt{\rho_{\w^\star}}$. Accordingly, it will be convenient to split the data matrix into two parts:
\begin{equation}\label{eq:X_split}
\mat{X} = \begin{bmatrix}
\vec{s} & \mat{B}
\end{bmatrix},
\end{equation}
where $\vec{s} \in \bbR^{n \times 1}$ and $\mat{B} \in \bbR^{n \times (d-1)}$ are two sub-matrices of \iid standard normal entries. The weight vector $\vec{w}$ in \eqref{eq:erm_l2} can also be written as $\vec{w} = [\sqrt{d} \mu, \vec{v}^\intercal]^\intercal$, where $\mu \in \bbR$ denotes the projection of $\vec{w}$ onto the direction spanned by the teacher weight vector $\vec{w}^\star$, and $\vec{v} \in \bbR^{d-1}$ is the projection of $\vec{w}$ onto the complement subspace. These representations serve to simplify the notations in our subsequent derivations. For example, we can now write the output as
\begin{equation}\label{eq:s2y}
y_i = \varphi_{\out^\star}(\rho_d s_i),
\end{equation}
where $s_i$ is the $i$th entry of the Gaussian vector $\vec{s}$ in \eqref{eq:X_split}.

Let $\Phi_d$ denote the cost of the ERM in \eqref{eq:erm_l2}, normalized by $d$. Using our new representations introduced above, we have
\begin{equation}\label{eq:Phi}
\Phi_d = \min_{\mu, \vec{v}} \frac{1}{d}\sum_{i=1}^\nsamples l\(y_i, \mu s_i + \tfrac{1}{\sqrt{d}}\vec{b}_i^\intercal \vec{v} \) +   \frac{\lambda (d\mu^2 + \norm{\vec{v}}^2)}{2d},
\end{equation}
where $\vec{b}_i^\intercal$ denotes the $i$th row of $\mat{B}$. Since the loss function $l(y_i, z)$ is convex with respect to $z$, we can rewrite it as
\begin{equation}\label{eq:conjugate}
l(y_i, z) = \sup_q \{q z - l^\ast(y_i, q)\},
\end{equation}
where $l^\ast(y_i, q) = \sup_z \{qz - l(y_i, z)\}$ is its convex conjugate. Substituting \eqref{eq:conjugate} into \eqref{eq:Phi}, we have
\begin{equation}\label{eq:Phi_o}
\Phi_d = \min_{\mu, \vec{v}} \, \sup_{\vec{q}}  \left\{\frac{\mu \vec{q}^\intercal \vec{s}}{d} + \frac{1}{d^{3/2}} \vec{q}^\intercal \mat{B} \vec{v} - \frac{1}{d} \sum_{i=1}^n l^\ast(y_i, q_i) + \frac{\lambda \(d\mu^2 + \norm{\vec{v}}^2\)}{2d}\right\}. 
\end{equation}

Now consider a new optimization problem
\begin{equation}\label{eq:Phi_s}
\widetilde{\Phi}_d = \min_{\mu, \vec{v}} \, \sup_{\vec{q}}  \left\{\frac{\mu \vec{q}^\intercal \vec{s}}{d} + \frac{\norm{\vec{q}}}{\sqrt{d}} \frac{\vec{h}^\intercal \vec{v}}{d} + \frac{\norm{\vec{v}}}{\sqrt{d}} \frac{\vec{g}^\intercal \vec{q}}{d} - \frac{1}{d} \sum_{i=1}^n l^\ast(y_i, q_i) + \frac{\lambda \(d\mu^2 + \norm{\vec{v}}^2\)}{2d}\right\},
\end{equation}
where $h \sim \mN\(\vec{0},\rI_{\ndim-1}\)$ and $g \sim  \mN\(\vec{0},\rI_{n}\)$ are two independent standard normal vectors. It follows from Gordon's minimax comparison inequality (see, \emph{e.g.}, \cite{pmlr-v40-Thrampoulidis15}) that
\begin{equation}\label{eq:cgmt}
\mathbb{P}(\abs{\Phi_d - c} \ge \epsilon) \le 2 \mathbb{P}\(\abs{\widetilde{\Phi}_d-c} \ge \epsilon\)
\end{equation}
for any constants $c$ and $\epsilon > 0$. This implies that $\widetilde{\Phi}_d$ serves as a surrogate of $\Phi_d$. Specifically, if $\widetilde{\Phi}_d$ concentrates around some deterministic limit $c$ as $d \to \infty$, so does $\Phi_d$. In what follows, we proceed to solve the surrogate problem in \eqref{eq:Phi_s}. First, let $\delta = \norm{\vec{v}}/\sqrt{d}$. It is easy to see that \eqref{eq:Phi_s} can be simplified as
\begin{align*}
\widetilde{\Phi}_d &= \min_{\mu, \delta \ge 0} \, \sup_{\vec{q}}  \left\{\frac{\vec{q}^\intercal (\mu\vec{s} + \delta \vec{g})}{d} - \delta \frac{\norm{\vec{q}}}{\sqrt{d}} \frac{\norm{\vec{h}}}{\sqrt{d}}  - \frac{1}{d} \sum_{i=1}^n l^\ast(y_i, q_i) + \frac{\lambda (\mu^2 + \delta^2)}{2}\right\}\\
&\overset{(a)}= \min_{\mu, \delta \ge 0} \, \sup_{\tau > 0} \, \sup_{\vec{q}} \left\{-\frac{\tau \norm{\vec{q}}^2}{2d} - \frac{\delta^2 \norm{\vec{h}}^2}{2\tau d} +\frac{\vec{q}^\intercal (\mu\vec{s} + \delta \vec{g})}{d}- \frac{1}{d} \sum_{i=1}^n l^\ast(y_i, q_i) + \frac{\lambda (\mu^2 + \delta^2)}{2}\right\}\\
&=\min_{\mu, \delta \ge 0} \, \sup_{\tau > 0}\left\{ \frac{\lambda (\mu^2 + \delta^2)}{2}-\frac{\delta^2 \norm{\vec{h}}^2}{2\tau d} - \frac{\alpha}{n}\inf_{\vec{q}} \Big[\frac{\tau \norm{\vec{q}}^2}{2} - \vec{q}^\intercal (\mu\vec{s} + \delta \vec{g})+ \sum_{i=1}^n l^\ast(y_i, q_i)\Big]\right\}\\
&\overset{(b)}=\min_{\mu, \delta \ge 0} \, \sup_{\tau > 0}\left\{ \frac{\lambda (\mu^2 + \delta^2)}{2}-\frac{\delta^2 \norm{\vec{h}}^2}{2\tau d} - \frac{\alpha}{n}\sum_{i=1}^n \mM_{\tau}[l(y_i,.)](\mu s_i + \delta g_i)\right\}.
\end{align*}
In $(a)$, we have introduced an auxiliary variable $\tau$ to rewrite $- \delta \frac{\norm{\vec{q}}}{\sqrt{d}} \frac{\norm{\vec{h}}}{\sqrt{d}}$ as
\begin{align*}
- \delta \frac{\norm{\vec{q}}}{\sqrt{d}} \frac{\norm{\vec{h}}}{\sqrt{d}} = \sup_{\tau > 0} \left\{-\frac{\tau \norm{\vec{q}}^2}{2d} - \frac{\delta^2 \norm{\vec{h}}^2}{2\tau d}\right\}\,,
\end{align*}
and to get $(b)$, we use the identity 
\begin{align*}
\inf_q \left\{\frac{\tau}{2} q^2 - q z + \ell^\ast(q)\right\} = - \inf_x \left\{\frac{(z-x)^2}{2\tau} + \ell(x)\right\}
\end{align*}
that holds for any $z$ and for any convex function $\ell(x)$ and its conjugate $\ell^\ast(q)$. As $d \to \infty$, standard concentration arguments give us $\frac{\norm{\vec{h}}^2}{d} \to 1$ and $\frac{1}{n}\sum_{i=1}^n \mM_{\tau}[l(y_i,.)](\mu s_i + \delta g_i) \to \mathbb{E}_{g, s} \mM_{\tau}[l(y,.)](\mu s + \delta g)$ locally uniformly over $\tau, \mu$ and $\delta$. Using \eqref{eq:cgmt} and recalling \eqref{eq:s2y}, we can then conclude that the normalized cost of the ERM $\Phi_d$ converges to the optimal value of the deterministic optimization problem in \eqref{eq:pot_func_general}. Finally, since $\lambda > 0$, one can show that the cost function of \eqref{eq:pot_func_general} has a unique global minima at $\mu^\ast$ and $\delta^\ast$. It follows that the empirical values of $(\mu, \delta)$ associated with the surrogate optimization problem \eqref{eq:Phi_s} converge to their corresponding deterministic limits $(\mu^\ast, \delta^\ast)$. Finally, the convergence of $(\mu, \delta)$ associated with the original optimization problem \eqref{eq:Phi_o} towards the same limits can be established by evoking standard arguments (see, \emph{e.g.}, \cite[Theorem 6.1, statement (iii)]{Thrampoulidis16}).
\end{proof}

\subsubsection{For binary outputs - Classification with $\rL_2$ regularization}

In what follows, we specialize the previous theorem to the case of binary classification, with a convex loss function in the form of $l(y, z) = \ell(yz)$ for some function $\ell(\cdot)$.

\begin{theorem}[Thm.~\ref{main:thm:gordon_fixed_points:classification} in the main text. Gordon's min-max fixed point - Classification with $\rL_2$ regularization]
\label{appendix:thm:gordon_fixed_points:classification}
As $\nsamples, \ndim \to \infty$ with $\nsamples/ \ndim = \alpha = \Theta(1)$, the overlap parameters $m, q$ concentrate to
\begin{align}
m & \underlim{\ndim}{\infty} \sqrt{\rho_{\w^\star}} \mu^\ast\,, && q \underlim{\ndim}{\infty} (\mu^\ast)^2 + (\delta^\ast)^2\,, 
\end{align}
where parameters $\mu^\ast, \delta^\ast$ are solutions of 
\begin{equation}\label{appendix:eq:pot_func}
(\mu^\ast, \delta^\ast) = \underset{\mu, \delta \ge 0}{\arg\min} \ \sup_{\tau > 0} \left\{\frac{\lambda(\mu^2 + \delta^2)}{2} - \frac{\delta^2}{2\tau} + \alpha \EE_{g, s} \mM_\tau[\delta g + \mu s \varphi_{\out^\star}(\sqrt{\rho_{\w^\star}} s)]\right\},
\end{equation}
and $g, s$ are two \iid standard normal random variables. The solutions $(\mu^\ast, \delta^\ast, \tau^\ast)$ of \eqref{appendix:eq:pot_func} can be reformulated as a set of fixed point equations 
\begin{align}
\begin{aligned}
	\mu^\ast &= \frac{\alpha}{\lambda \tau^\ast + \alpha } \EE [s \cdot \varphi_{\out^\star}(\sqrt{\rho_{\w^\star}} s) \cdot \mP_{\tau^\ast}(\delta^\ast g+ \mu^\ast s \varphi_{\out^\star}(\sqrt{\rho_{\w^\star}} s))]\,, \spacecase
	\delta^\ast &= \frac{\alpha}{\lambda \tau^\ast + \alpha -1} \EE [g \cdot \mP_{\tau^\ast}(\delta^\ast g+ \mu^\ast s \varphi_{\out^\star}(\sqrt{\rho_{\w^\star}} s))]\,,\spacecase
	(\delta^\ast)^2 &= \alpha \EE [\(\delta^\ast g + \mu^\ast s \varphi_{\out^\star}(\sqrt{\rho_{\w^\star}} s) - \mP_{\tau^\ast}(\delta^\ast g + \mu^\ast s \varphi_{\out^\star}(\sqrt{\rho_{\w^\star}} s)) \)^2] \,,
\label{appendix:fixed_point_equations_gordon}
\end{aligned}
\end{align}
where $\mM_\tau$ and $\mP_\tau$ denote the Moreau-Yosida regularization and the proximal map of a convex loss function $(y, z) \mapsto \ell(yz)$:
\begin{align*}
\mM_\tau(z) = \min_x \left\{\ell(x) + \frac{(x-z)^2}{2\tau}\right\}, \qquad \mP_\tau(z) = \underset{x}{\arg\,\min} \left\{\ell(x) + \frac{(x-z)^2}{2\tau}\right\}.
\end{align*}
\end{theorem}
\begin{proof}
	\label{appendix:proof:gordon_classification}
We start by deriving \eqref{appendix:eq:pot_func} as a special case of \eqref{eq:pot_func_general}. To that end, we note that
\begin{align*}
\mM_{\tau}[l(y,.)](z) &= \min_x \left\{l(y; x) + \frac{(x-z)^2}{2\tau}\right\}\\
&= \min_x \left\{\ell(yx) + \frac{(x-z)^2}{2\tau}\right\}\\
&=\min_x \left\{\ell(x) + \frac{(x-yz)^2}{2\tau}\right\} = \mM_\tau(yz),
\end{align*}
where to reach the last equality we have used the fact that $y \in \{\pm 1\}$. Substituting this special form into \eqref{eq:pot_func_general} and recalling \eqref{eq:s2y}, we reach \eqref{appendix:eq:pot_func}.

Finally, to obtain the fixed point equations \eqref{appendix:fixed_point_equations_gordon}, we simply take the partial derivatives of the cost function in \eqref{appendix:eq:pot_func} with respect to $\mu, \delta, \tau$, and use the following well-known calculus rules for the Moreau-Yosida regularization \cite{hiriartUrruty1993}:
\begin{align*}
\frac{\partial \mM_\tau(z)}{\partial z} &= \frac{z - \mP_\tau(z)}{\tau}\,,\\
\frac{\partial \mM_\tau(z)}{\partial \tau} &= -\frac{(z - \mP_\tau(z))^2}{2\tau^2}.
\end{align*}
\end{proof}

\subsection{Replica's formulation}
\label{appendix:proof:replicas_formulation}
The replica computation presented in Sec.~\ref{appendix:replicas} boils down to the characterization of the overlaps $m,q$ in the high-dimensional limit $\nsamples, \ndim \to \infty$ with $\alpha = \frac{\nsamples}{\ndim} = \Theta(1)$, given by the solution of a set of, in the most general case, six fixed point equations over $m, q, Q, \hat{m}, \hat{q}, \hat{Q}$.  Introducing the natural variables $\Sigma \equiv Q - q$, $\hat{\Sigma} \equiv \hat{Q} +  \hat{q}$, $\eta \equiv \frac{m^2}{\rho_{\w^\star}q}$ and $\hat{\eta} \equiv \frac{\hat{m}^2}{\hat{q}}$, the set of fixed point equations for arbitrary $P_{\w^\star}, P_{\out^\star}$, convex loss $l(y,z)$ and regularizer $r(w)$, is finally given by
\begin{align}
\begin{aligned}
	m&= \EE_{\xi} \[ \mZ_{\w^\star} \(\sqrt{\hat{\eta}}  \xi , \hat{\eta} \) f_{\w^\star}\(\sqrt{\hat{\eta}}  \xi , \hat{\eta} \) f_{\w} \(  \hat{q}^{1/2}\xi  , \hat{\Sigma} \)   \]\,, \\
	q &= \EE_{\xi} \[ \mZ_{\w^\star}\( \sqrt{\hat{\eta}}  \xi , \hat{\eta}  \) f_{\w} \(  \hat{q}^{1/2}\xi  , \hat{\Sigma} \)^2   \]\,, \\
	\Sigma &=  \EE_{\xi} \[ \mZ_{\w^\star}\( \sqrt{\hat{\eta}}  \xi , \hat{\eta}  \)  \partial_\gamma f_{\w} \(  \hat{q}^{1/2}\xi  , \hat{\Sigma} \) \]\,, \\
	\hat{m} &= \alpha \EE_{y, \xi } \[ \mZ_{\out^\star}(.) \cdot f_{\out^\star} \(y,  \sqrt{\rho_{\w^\star} \eta} \xi, \rho_{\w^\star}\(1 - \eta\)\)   f_{\out} \( y,  q^{1/2}\xi, \Sigma \) \]\,, \\
	\hat{q} &= \alpha \EE_{y, \xi } \[ \mZ_{\out^\star} \( y,  \sqrt{\rho_{\w^\star} \eta} \xi, \rho_{\w^\star}\(1 - \eta\)  \)   f_{\out} \( y,  q^{1/2}\xi, \Sigma \)^2 \]\,, \\
	\hat{\Sigma} &=  - \alpha \EE_{y, \xi } \[ \mZ_{\out^\star} \( y,  \sqrt{\rho_{\w^\star} \eta} \xi, \rho_{\w^\star}\(1 - \eta\)  \)  \partial_\omega f_{\out} \( y,  q^{1/2}\xi, \Sigma \) \]\,.
\end{aligned}
\label{appendix:fixed_point_replicas}
\end{align}
The above equations depend on the Bayes-optimal partition functions $\mZ_{\w^\star}, \mZ_{\out^\star}$ defined in eq.~\eqref{appendix:definitions:update_bayes}, the updates $f_{\w^\star}$, $f_{\out^\star}$ in eq.~\eqref{appendix:definitions:update_functions_bayes} and the ERM updates $f_{\w}$, $f_{\out}$ eq.~\eqref{appendix:definitions:update_functions_map}.

\subsection{Equivalence Gordon-Replica's formulation - $\rL_2$ regularization and Gaussian weights}
\label{appendix:proof:equivalence_gordon_replicas_formulation_l2}
\subsubsection{Replica's formulation for $\rL_2$ regularization}
The proximal for the $\rL_2$ penalty with strength $\lambda$ can be computed explicitly in eq.~\eqref{appendix:definitions:application:L2} and the corresponding denoising function is simply given by $f_{\w}^{\rL_2, \lambda} \(\gamma , \Lambda \) =\frac{\gamma}{\lambda + \Lambda }$.
Therefore, for a Gaussian teacher \eqref{appendix:model:gaussian_weights} already considered in Thm.~\eqref{appendix:fixed_point_equations_gordon} with second moment $\rho_{\w^\star}$, using the denoising function \eqref{appendix:definitions:application:gaussian},
	the fixed point equations over $m,q,\Sigma$ can be computed analytically and lead to
\begin{align}
\begin{aligned}
	m &= \frac{\rho_{\w^\star}\hat{m}}{\lambda + \hat{\Sigma}}\,, \hhspace \hhspace 
	&& q = \frac{\rho_{\w^\star}\hat{m}^2+\hat{q}}{(\lambda + \hat{\Sigma})^2}\,, \hhspace \hhspace 
	&& \Sigma =   \frac{1}{\lambda + \hat{\Sigma}}\,.
\end{aligned}
\end{align}	
Hence, removing the \emph{hat} variables in eqs.~\eqref{appendix:fixed_point_replicas}, the set of fixed point equations can be rewritten in a more compact way leading to the Corollary.~\ref{main:corollary:equivalence_gordon_replicas_formulation_l2} that we recall here:

\begin{corollary}[Corollary.~\ref{main:corollary:equivalence_gordon_replicas_formulation_l2} in the main text. Equivalence Gordon-Replicas]
\label{appendix:corollary:equivalence_gordon_replicas_formulation_l2}

The set of fixed point equations \eqref{appendix:fixed_point_equations_gordon} in Thm.~\ref{appendix:thm:gordon_fixed_points:classification} that govern the asymptotic behaviour of the overlaps $m$ and $q$ is equivalent to the following set of equations, obtained from the heuristic replica computation:
\begin{align}
	m &= \alpha \Sigma \rho_{\w^\star} \cdot \EE_{y, \xi } \[ \mZ_{\out^\star}\(.\) \cdot f_{\out^\star} \(y,  \sqrt{\rho_{\w^\star}\eta} \xi, \rho_{\w^\star}\(1 - \eta\)\) \cdot f_{\out} \( y,  q^{1/2}\xi, \Sigma \)    \]  \nonumber \\
	q &= m^2/\rho_{\w^\star} + \alpha \Sigma^2 \cdot \EE_{y, \xi } \[ \mZ_{\out^\star} \( y,  \sqrt{\rho_{\w^\star}\eta} \xi, \rho_{\w^\star}\(1 - \eta\)  \)  \cdot f_{\out} \( y,  q^{1/2}\xi, \Sigma \)^2    \]   \label{appendix:fixed_point_equations_replicas} \\
	\Sigma &=   \(\lambda - \alpha \cdot \EE_{y, \xi } \[ \mZ_{\out^\star} \( y,  \sqrt{\rho_{\w^\star}\eta} \xi, \rho_{\w^\star}\(1 - \eta\)  \) \cdot  \partial_\omega f_{\out} \( y,  q^{1/2}\xi, \Sigma \)    \] \)^{-1} \nonumber
\end{align}
with $\eta \equiv \frac{m^2}{\rho_{\w^\star}q}$, $\xi \sim \mN(0,1)$ and $\EE_y$ the continuous or discrete sum over all possible values $y$ according to $P_{\out^\star}$.
\end{corollary}
\begin{proof}[Proof of Corollary.~\ref{appendix:corollary:equivalence_gordon_replicas_formulation_l2}(Corollary.~\ref{main:corollary:equivalence_gordon_replicas_formulation_l2})]
For the sake of clarity, we use the abusive notation $\mP_V(y, \omega) = \mP_{V}[l(y,.)](\omega)$, and we remove the $\ast$.
\item 
\paragraph{Dictionary} 		
We first map the Gordon's parameters ($\mu, \delta, \tau)$ in eq.~\eqref{appendix:fixed_point_equations_gordon} to ($m,q,\Sigma$) in eq.~\eqref{appendix:fixed_point_equations_replicas}:
\begin{align*}
	&\sqrt{\rho_{\w^\star}} \mu \leftrightarrow m\,, && \mu^2 + \delta^2 \leftrightarrow q \,,  && \tau \leftrightarrow \Sigma\,.
\end{align*}
so that 
\begin{align*}
	\eta &= \frac{m^2}{\rho_{\w^\star}q} = \frac{\mu^2}{\mu^2 + \delta^2}\,, && 1-\eta = \frac{\delta^2}{\mu^2 + \delta^2}\,.
\end{align*}
From eq.~\eqref{appendix:definitions:update_bayes}, we can rewrite the channel partition function $\mZ_{\out^\star}$ and its derivative 
\begin{align}
\begin{aligned}
	\mZ_{\out^\star}\(y, \omega, V\) &= \EE_{z} \[ P_{\out^\star}\(y| \sqrt{V} z + \omega \)  \],\\
	\partial_\omega \mZ_{\out^\star}\(y, \omega, V\) &= \frac{1}{\sqrt{V}} \EE_{z} \[ z P_{\out^\star}\(y| \sqrt{V} z + \omega \)  \]\,,
\end{aligned}
\label{appendix:proof:channel_reformulation}
\end{align}
where $z$ denotes a standard normal random variable. 
\item
\paragraph{Equation over $m$}
Let us start with the equation over $m$ in eq.~\eqref{appendix:fixed_point_equations_replicas}:
\begin{align*}
	m &= \Sigma \alpha \rho_{\w^\star} \EE_{y, \xi } \[ \mZ_{\out^\star} \(y,  \sqrt{\rho_{\w^\star}\eta} \xi, \rho_{\w^\star}\(1 - \eta\)\) f_{\out^\star} \(y,  \sqrt{\rho_{\w^\star}\eta} \xi, \rho_{\w^\star}\(1 - \eta\)\)  \right. \\
	&  \qquad \qquad \qquad \qquad \qquad \qquad \qquad \qquad \qquad \qquad \qquad \qquad \qquad \qquad  \left. \times  f_{\out} \( y,  q^{1/2}\xi, \Sigma \)    \]\\
	&= \Sigma \alpha \frac{\sqrt{\rho_{\w^\star}}}{\sqrt{1-\eta}}  \EE_{y, \xi, z} \[  z  P_{\out^\star}\(y| \sqrt{\rho_{\w^\star}} \(\sqrt{1-\eta} z + \sqrt{\eta} \xi\) \)   \Sigma^{-1} ( \mP_{\Sigma} \(y, \sqrt{q}\xi\) - \sqrt{q}\xi) \] \tag{Using eq.~\eqref{appendix:proof:channel_reformulation}} \\
	\Leftrightarrow  \mu  &= \frac{\sqrt{\mu^2 + \delta^2}}{\delta}  \alpha \EE_{y, \xi, z} \[  z  P_{\out^\star}\[y \vert \sqrt{\rho_{\w^\star}}\frac{\delta z + \mu \xi}{\sqrt{\mu^2 + \delta^2}}  \]  \( \mP_{\tau} \(y, \sqrt{\mu^2 + \delta^2}\xi\) - \sqrt{\mu^2 + \delta^2}\xi \) \] \tag{Dictionary} \\
	&= \frac{\sqrt{\mu^2 + \delta^2}}{\delta} \alpha \EE_{\xi, z} \[  z \( \mP_{\tau} \( \varphi_{\out^\star}\(\sqrt{\rho_{\w^\star}}\frac{ \delta z + \mu \xi}{\sqrt{\mu^2 + \delta^2}} \), \sqrt{\mu^2 + \delta^2}\xi\) - \sqrt{\mu^2 + \delta^2}\xi\)     \] \tag{Integration over $y$}\\
	&= \alpha \EE_{s, g} \[ \(s- \frac{\mu}{\delta} g \) \( \mP_{\tau} \( \varphi_{\out^\star}\(\sqrt{\rho_{\w^\star}} s \), \delta g + \mu s\)  - (\delta g + \mu s)\)      \] \tag{Change of variables $(\xi, z) \to (g, s)$} \\
	&= \alpha \EE_{s, g} \[ \(s - \frac{\mu}{\delta} g \) \( \mP_{\tau} \( \varphi_{\out^\star}\(\sqrt{\rho_{\w^\star}} s \), \delta g + \mu s\) \)  \]\tag{Gaussian integrations} \\
	\Leftrightarrow  \mu &=  \frac{\alpha \EE_{s, g} \[ s \cdot   \mP_{\tau} \( \varphi_{\out^\star}\(\sqrt{\rho_{\w^\star}}  s \), \delta g + \mu s\)  \] }{1 + \frac{\alpha }{\delta}  \EE_{s, g} \[ g \cdot \mP_{\tau} \( \varphi_{\out^\star}\(\sqrt{\rho_{\w^\star}}  s \), \delta g + \mu s\)       \]  } \\
	&= \frac{\alpha}{ \lambda \tau + \alpha } \EE_{s, g} \[ s \cdot \varphi_{\out^\star}\(\sqrt{\rho_{\w^\star}} s \)  \( \mP_{\tau} \( \delta g + \mu s\) \varphi_{\out^\star}\(\sqrt{\rho_{\w^\star}} s \) \)    \]\tag{Second fixed point equation}\,,
\end{align*}
where we used the fact that $P_{\out^\star}\(y|z\) = \delta(y-\varphi_{\out^\star}(z))$,  the change of variables
\begin{align}
\begin{cases}
	s = \frac{\mu \xi + \delta z}{\sqrt{\mu^2 + \delta^2}} \spacecase
	g = \frac{\delta \xi - \mu z}{\sqrt{\mu^2 + \delta^2}}
\end{cases}	
\Leftrightarrow
\begin{cases}
	\xi = \frac{\delta g+ \mu s}{\sqrt{\mu^2+\delta^2}} \spacecase
	z=  \frac{\delta s -\mu g}{\sqrt{\mu^2+\delta^2}}
\end{cases}\,,
\end{align}
and finally in the last equality the definition of the second fixed point equation in eqs.~\eqref{appendix:fixed_point_equations_gordon}:
\begin{align}
	\delta = \alpha \frac{\EE_{s, g} \[ g  \cdot \mP_{\tau} \( \varphi_{\out^\star}\(\sqrt{\rho_{\w^\star}} s \), \delta g + \mu s\)  \]}{\lambda \tau + \alpha - 1}\,.
	\label{appendix:equivalence:Sigma}
\end{align}
\item
\paragraph{Equation over $q$}
Let us now compute the equation over $q$ in eq.~\eqref{appendix:fixed_point_equations_replicas}:
\begin{align*}
	q- m^2/\rho_{\w^\star} &= \Sigma^2 \alpha \EE_{y, \xi } \[ \mZ_{\out^\star} \( y,  \sqrt{\rho_{\w^\star}\eta} \xi,\rho_{\w^\star}\(1 - \eta\)  \)   f_{\out} \( y,  q^{1/2}\xi, \Sigma \)^2    \] \\
	&= \Sigma^2 \alpha \EE_{y, \xi, z } \[ P_{\out^\star}\(y| \sqrt{\rho_{\w^\star}}\(\sqrt{1-\eta} z + \sqrt{\eta}\xi\) \)  \frac{1}{\Sigma^2} \(p_{\Sigma}\(y, \sqrt{q}\xi\)  - \sqrt{q}\xi  \)^2 \] \tag{Using eq.~\eqref{appendix:proof:channel_reformulation}}\\
\Leftrightarrow \delta^2 &=  \alpha \EE_{y, \xi, z } \[ P_{\out^\star}\(y| \sqrt{\rho_{\w^\star}} \frac{\delta z + \mu \xi}{\sqrt{\mu^2 + \delta^2}} \)  \(p_{\tau}\(y, \sqrt{\mu^2 + \delta^2}\xi\)  - \sqrt{\mu^2 + \delta^2}\xi  \)^2 \] \tag{Dictionary} \\	
	&=  \alpha \EE_{\xi, z} \[ \(p_{\tau}\(\varphi_{\out^\star}\(\sqrt{\rho_{\w^\star}}\frac{ \delta z + \mu \xi}{\sqrt{\mu^2 + \delta^2}}\), \sqrt{\mu^2 + \delta^2}\xi \)  - \sqrt{\mu^2 + \delta^2}\xi  \)^2 \]\tag{Integration over $y$} \\
	&= \alpha \EE_{g, s} \[ \(p_{\tau}\(\varphi_{\out^\star}\(\sqrt{\rho_{\w^\star}}s\), \delta g + \mu s \)  - \(\delta g + \mu s\)  \)^2 \] \tag{Change of variables $(\xi, z) \to (g, s)$} 
\end{align*}

\item
\paragraph{Equation over $\Sigma$}
Let us conclude with the equation over $\Sigma$ in eq.~\eqref{appendix:fixed_point_equations_replicas} that we encountered in eq.~\eqref{appendix:equivalence:Sigma}. Let us first compute
\begin{align*}
&\alpha \EE_{y, \xi } \[ \mZ_{\out^\star} \( y,  \sqrt{\rho_{\w^\star}\eta} \xi, \rho_{\w^\star}\(1 - \eta\)  \)  \partial_\omega f_{\out} \( y,  q^{1/2}\xi, \Sigma \)    \]\\
&= \alpha \EE_{y, \xi, z } \[ P_{\out^\star}\(y| \sqrt{\rho_{\w^\star}}\(\sqrt{1-\eta} z + \sqrt{\eta}\xi \)\)  \frac{1}{\Sigma} \( \partial_\omega p_{\Sigma}\(y, \sqrt{q}\xi\)  - 1  \) \] \tag{Using eq.~\eqref{appendix:proof:channel_reformulation}}\\
	&=  \frac{\alpha}{\tau}  \EE_{y, \xi, z } \[ P_{\out^\star}\(y|  \sqrt{\rho_{\w^\star}}\frac{\delta z + \mu \xi}{\sqrt{\mu^2 + \delta^2}}\)  \(\partial_\omega \mP_{\tau}\(y, \sqrt{\mu^2 + \delta^2}\xi\)  - 1  \) \] \tag{Dictionary} \\	
	&=  \frac{\alpha}{\tau} \EE_{\xi, z} \[ \partial_\omega \mP_{\tau}\(\varphi_{\out^\star}\( \sqrt{\rho_{\w^\star}}\frac{\delta z + \mu \xi}{\sqrt{\mu^2 + \delta^2}}\), \sqrt{\mu^2 + \delta^2}\xi \)    \] - \frac{\alpha}{\tau} \tag{Integration over $y$} \\
	&= \frac{1}{\tau} \alpha \( \EE_{g, s} \[ \partial_\omega \mP_{\tau}\(\varphi_{\out^\star}\( \sqrt{\rho_{\w^\star}}s\), \delta g + \mu s \)  \] -1 \)  \tag{Change of variables $(\xi, z) \to (g, s)$} 
\end{align*}
therefore, the last equation over $\Sigma$ in eq.~\eqref{appendix:fixed_point_equations_replicas}  reads
\begin{align*} 
	\Sigma &=  \( \lambda - \alpha \EE_{y, \xi } \[ \mZ_{\out^\star} \( y,  \sqrt{\rho_{\w^\star} \eta} \xi, \rho_{\w^\star}\(1 - \eta\)  \)  \partial_\omega f_{\out} \( y,  q^{1/2}\xi, \Sigma \)    \] \)^{-1}\\
	& \Leftrightarrow\\
	\tau &= \(\lambda - \frac{1}{\tau} \alpha \( \EE_{g, s} \[ \partial_\omega \mP_{\tau}\(\varphi_{\out^\star}\(\sqrt{\rho_{\w^\star}}s\), \delta g + \mu s \)  \] -1 \)  \)^{-1} \\
	&\Leftrightarrow\\
	\alpha & \EE_{g, s} \[ \partial_\omega \mP_{\tau}\(\varphi_{\out^\star}\(\sqrt{\rho_{\w^\star}}s\), \delta g + \mu s \)  \]  = \tau \lambda + \alpha -1\,.
\end{align*}
Noting that 
\begin{align*}
	&\EE_{g, s} \[ \partial_\omega \mP_{\tau}\(\varphi_{\out^\star}\(\sqrt{\rho_{\w^\star}}s\), \delta g + \mu s \) \] = \frac{1}{\delta} \EE_{g, s} \[ d \partial_\omega \mP_{\tau}\(\varphi_{\out^\star}\(\sqrt{\rho_{\w^\star}}s\), \delta g + \mu s  \) \]\\
	&= \frac{1}{\delta} \EE_{g, s} \[ \partial_g \mP_{\tau}\(\varphi_{\out^\star}\(\sqrt{\rho_{\w^\star}}s\), \delta g + \mu s \) \] = \frac{1}{\delta} \EE_{g, s} \[ g \mP_{\tau}\(\delta g + \mu s\varphi_{\out^\star}\(\sqrt{\rho_{\w^\star}}s\) \) \] \tag{Stein's lemma}
\end{align*}
where we used the Stein's lemma in the last equality, we finally obtain
\begin{align*}
	&\alpha  \EE_{g, s} \[ \partial_\omega \mP_{\tau}\(\varphi_{\out^\star}\(\sqrt{\rho_{\w^\star}}s\), \delta g + \mu s \)  \]  = \tau \lambda + \alpha -1\\
	\Leftrightarrow & \delta =\frac{\alpha}{\tau \lambda + \alpha -1 } \EE_{g, s} \[ g \cdot \mP_{\tau}\(\varphi_{\out^\star}\(\sqrt{\rho_{\w^\star}}s\), \delta g + \mu s \) \]\,.
\end{align*}
\item
\paragraph{Gauge transformation}
We still remain to prove that
\begin{align}
\begin{aligned}
	\EE_{s, g} \[ g  \cdot \mP_{\tau} \( \varphi_{\out^\star}\( \sqrt{\rho_{\w^\star}} s \), \delta g + \mu s\)  \]&= \EE_{s, g} \[ g  \cdot \mP_{\tau} \( \delta g + \mu s \varphi_{\out^\star}\(\sqrt{\rho_{\w^\star}} s \)\)  \]\\
	\EE_{s, g} \[ s \cdot  \mP_{\tau} \( \varphi_{\out^\star}\(\sqrt{\rho_{\w^\star}} s \), \delta g + \mu s\)  \] &= \EE_{s, g} \[ s \cdot  \mP_{\tau} \( \delta g + \mu s \varphi_{\out^\star}\(\sqrt{\rho_{\w^\star}} s \) \)  \]\\
	\EE_{g, s} \[ \(p_{\tau}\(\varphi_{\out^\star}\( \sqrt{\rho_{\w^\star}}s\), \delta g + \mu s \)  - \(\delta g + \mu s\)  \)^2 \] &= \EE_{g, s} \[ \right. \\
	& \qquad \left.  \(\(p_{\tau} - \id \)\(\delta g + \mu s \varphi_{\out^\star}\(\sqrt{\rho_{\w^\star}}s\) \) \)^2 \]   
\end{aligned}
\end{align}
As $\varphi_{\out^\star}\(\sqrt{\rho_{\w^\star}} s \) = \pm 1$, we can transform $s \to s \varphi_{\out^\star}\(\sqrt{\rho_{\w^\star}} s \)=\td{s}$. It does not change the distribution of the random variable $\td{s}$ that is still a normal random variable. Finally denoting 
$\mP_{\tau}\(1, \delta g + \mu s \varphi_{\out^\star}\(\sqrt{\rho_{\w^\star}} s \) \) = \mP_{\tau}\(\delta g + \mu s \varphi_{\out^\star}\(\sqrt{\rho_{\w^\star}} s \) \)$, we obtain the equivalence with eq.~\eqref{appendix:fixed_point_equations_gordon}, which concludes the proof.
\end{proof}

%% file: files/supplementary/replicas.tex
\label{appendix:replicas}

In this section, we present the statistical physics framework and the replica computation leading to the general set of fixed point equations \eqref{main:fixed_point_equations_replicas} and to the Bayes-optimal fixed point equations \eqref{main:fixed_point_equations_bayes}. 
\subsection{Statistical inference and free entropy}
\label{appendix:bayesian}
 
As stressed in Sec.~\ref{appendix:definitions}, both ERM and Bayes-optimal estimations can be analyzed in a unified framework that consists in studying the joint distribution $\bbP\(\vec{y}, \mat{X}\)$ in the following posterior distribution
\begin{align}
	\bbP\(\vec{w} | \vec{y}, \mat{X}\)  &= \frac{\bbP\(\vec{y} | \vec{w}, \mat{X} \) \bbP\(\vec{w}\) }{\bbP\(\vec{y},\mat{X}\)}   \,,
	\label{appendix:bayes_formula}
\end{align}
known as the so-called \emph{partition function} in the physics literature. It is the generating function of many useful statistical quantities and is defined by 
\begin{align}
\begin{aligned}
\mZ\(\vec{y}, \mat{X} \) &\equiv P\(\vec{y},\mat{X}\) = \int_{\bbR^\ndim} \d\vec{w} P_{\out} \(\vec{y} | \vec{w}, \mat{X} \) P_{\w}\(\vec{w}\) \spacecase
	&= \int_{\bbR^\nsamples} \d\vec{z} P_{\out}\(\vec{y} | \vec{z} \) \int_{\bbR^\ndim} \d\vec{w} P_{\w}\(\vec{w}\) \delta\(\vec{z} - \frac{1}{\sqrt{\ndim}} \mat{X} \vec{w}\)\,,
\end{aligned}
\end{align}
where we introduced the variable $\vec{z} = \frac{1}{\sqrt{\ndim}} \mat{X} \vec{w}$.
However in the considered high-dimensional regime ($\ndim \to \infty, \nsamples \to\infty, \alpha = \Theta(1)$), we are interested instead in the \emph{averaged} (over instances of input data $\mat{X}$ and teacher weights $\vec{w}^\star$ or equivalently over the output labels $\vec{y}$) \emph{free entropy} $\Phi$ defined as
\begin{align}
	\Phi(\alpha) \equiv \EE_{\vec{y},\mat{X}} \[ \lim_{\ndim \to \infty} \frac{1}{\ndim} \log  \mZ\(\vec{y}, \mat{X}\) \]\,.
	\label{appendix:free_entropy}
\end{align}
The replica method is an heuristic method of statistical mechanics that allows to compute the above average over the random dataset $\{\vec{y}, \mat{X}\}$. We show in the next section the classical computation for the Generalized Linear Model hypothesis class and \iid data $\mat{X}$.

\subsection{Replica computation}
\label{appendix:replicas_computation}

\subsubsection{Derivation}
We present here the replica computation of the averaged free entropy $\Phi(\alpha)$ in eq.~\eqref{appendix:free_entropy} for general prior distributions $P_{\w},P_{\w^\star}$ and channel distributions $P_{\out},P_{\out^\star}$, so that the computation remain valid for both Bayes-optimal and ERM estimation (with any convex loss $l$ and regularizer $r$).

\paragraph{Replica trick}
\label{appendix:sec:replicas:replica_trick}
The average in eq.~\eqref{appendix:free_entropy} is intractable in general, and the computation relies on the so called \emph{replica trick} that consists in applying the identity  
\begin{align}
	\EE_{\vec{y},\mat{X}} \[ \lim_{\ndim \to \infty} \frac{1}{\ndim} \log  \mZ\(\vec{y}, \mat{X}\) \] =  \lim_{r \to 0} \[ \lim_{\ndim \to \infty}  \frac{1}{\ndim}  \frac{\partial \log \EE_{\vec{y},\mat{X}} \[  \mZ\(\vec{y}, \mat{X}\)^r\] }{\partial r} \]\,.
	\label{appendix:replicas:replica_trick}
\end{align}
This is interesting in the sense that it reduces the intractable average to the computation of the moments of the averaged partition function, which are easiest quantities to compute. Note that for $r \in \bbN$, $\mZ\(\vec{y}, \mat{X}\)^r$ represents the partition function of $r \in \mathbb{N}$ identical non-interacting copies of the initial system, called \emph{replicas}. Taking the average will then correlate the replicas, before taking the number of replicas $r\to 0$.
Therefore, we assume there exists an analytical continuation so that $r\in \bbR$ and the limit is well defined. Finally, note we exchanged the order of the limits $r \to 0$ and $\ndim \to \infty$. These technicalities are crucial points but are not rigorously justified and we will ignore them in the rest of the computation.

Thus the replicated partition function in eq.~\eqref{appendix:replicas:replica_trick} can be written as 
\begin{align}
\begin{aligned}
\EE_{\vec{y},\mat{X}} \[  \mZ\(\vec{y}, \mat{X}\)^r \] &=  \EE_{\vec{w}^\star,\mat{X}} \[  \prod_{a=1}^r \int_{\bbR^\nsamples} \d\vec{z}^a P_{\out^a}\(\vec{y} | \vec{z}^a \) \int_{\bbR^\ndim} \d\vec{w}^a P_{\w^a}\(\vec{w}^a\) \delta\(\vec{z}^a - \frac{1}{\sqrt{\ndim}}\mat{X} \vec{w}^a\)\] \\ 
&= \EE_{\mat{X}} \int_{\bbR^\nsamples} \d \vec{y} \int_{\bbR^\nsamples} \d\vec{z}^\star P_{\out^\star}\(\vec{y} | \vec{z}^\star \)  \int_{\bbR^\ndim} \d\vec{w}^\star P_{\w^\star}\(\vec{w}^\star\) \delta\(\vec{z}^\star - \frac{1}{\sqrt{\ndim}}\mat{X} \vec{w}^\star\) \\
& \hspace{1cm}  \times \[  \prod_{a=1}^r \int_{\bbR^\nsamples} \d\vec{z}^a P_{\out^a}\(\vec{y} | \vec{z}^a \) \int_{\bbR^\ndim} \d\vec{w}^a P_{\w^a}\(\vec{w}^a\) \delta\(\vec{z}^a - \frac{1}{\sqrt{\ndim}}\mat{X} \vec{w}^a\)\]  \\
&= \EE_{\mat{X}} \int_{\bbR^\nsamples} \d \vec{y}  \prod_{a=0}^r \int_{\bbR^\nsamples} \d\vec{z}^a P_{\out^a}\(\vec{y} | \vec{z}^a \) \int_{\bbR^\ndim} \d\vec{w}^a P_{\w^a}\(\vec{w}^a\) \delta\(\vec{z}^a - \frac{1}{\sqrt{\ndim}}\mat{X} \vec{w}^a\)
\label{appendix:average_Zn}
\end{aligned}
\end{align}
with the decoupled channel $P_{\out}\(\vec{y} | \vec{z} \) = \displaystyle \prod_{\mu=1}^\nsamples P_{\out}\(y_{\mu} | z_\mu \)$.
Note that the average over $\vec{y}$ is equivalent to the one over the ground truth vector $\vec{w}^\star$, which can be considered as a new replica $\vec{w}^0$ with index $a=0$ leading to a total of $r+1$ replicas.

We suppose that inputs are drawn from an \iid distribution, for example a Gaussian $\mathcal{N}\left(0,1\right)$. More precisely, for $i,j\in [1:\ndim]$, $\mu,\nu \in [1:\nsamples]$, $\mathbb{E}_\mat{X} \[ x_{i}^{(\mu)} x_{j}^{(\nu)} \] =  \delta_{\mu\nu} \delta_{ij}$. Hence $z_{\mu}^a =\frac{1}{\sqrt{\ndim}} \sum_{i=1}^\ndim x^{(\mu)}_{i} w_i^a$ is the sum of \iid random variables. The central limit theorem insures that $z_{\mu}^a \sim \mathcal{N}\left(\mathbb{E}_{\mat{X}}[z_\mu^a]  ,\mathbb{E}_{\mat{X}}[z_\mu^a z_\mu^b] \right)$, with the two first moments given by:
\begin{equation}
	\begin{cases}
		\mathbb{E}_{\mat{X}}[z_\mu^a] = \frac{1}{\sqrt{\ndim}} \sum_{i=1}^\ndim \mathbb{E}_{\mat{X}}\[x^{(\mu)}_{i}\] w_i^a =0 \spacecase
		\mathbb{E}_{\mat{X}}[z_\mu^a z_\mu^b] = \frac{1}{\ndim} \sum_{ij} \mathbb{E}_{\mat{X}}\[x^{(\mu)}_{i} x^{(\mu)}_{j}\] w_i^a w_j^b  = \frac{1}{\ndim} \sum_{ij}  \delta_{ij} w_i^a w_j^b = \frac{1}{\ndim} \vec{w}^a \cdot \vec{w}^b   \, .
	\end{cases}
\end{equation}
In the following we introduce the symmetric \emph{overlap} matrix $Q(\{\vec{w}^a\})\equiv\(\frac{1}{\ndim} \vec{w}^a \cdot \vec{w}^b\)_{a,b=0..r}$. Let us define $\tbf{\tilde{z}}_{\mu} \equiv (z^a_{\mu})_{a=0..r}$ and $\tbf{\tilde{w}}_i \equiv (w_i^a)_{a=0..r}$. The vector $\tbf{\tilde{z}}_{\mu}$ follows a multivariate Gaussian distribution $\tbf{\tilde{z}}_{\mu} \sim P_{\tilde{\rm z}}(\tilde{\vec{z}};Q) = \mathcal{N}_{\tilde{\vec{z}}}( \tbf{0}_{r+1}, Q)$ and as $
        P_{\tilde{\rm w}}(\tbf{\tilde{w}}) = \prod_{a=0}^r
        P_{\w}(\td{w}^a)$ it follows
 \begin{align*}       
  \EE_{\vec{y},\mat{X}} \[  \mZ\(\vec{y}, \mat{X}\)^r \] &= \EE_{\mat{X}}  \int_{\bbR^\nsamples} \d \vec{y} \prod_{a=0}^r \int_{\bbR^\nsamples} \d\vec{z}^a P_{\out^a}\(\vec{y} | \vec{z}^a \) \int_{\bbR^\ndim} \d\vec{w}^a P_{\w^a}\(\vec{w}^a\) \delta\(\vec{z}^a -\frac{1}{\sqrt{\ndim}} \mat{X} \vec{w}^a\)\\
  &= \[ \int_{\bbR} \d y \int_{\bbR^{r+1}} \d\td{\vec{z}} P_{\out}\(y | \td{\vec{z}} \) P_{\tilde{\rm z}}(\tilde{\vec{z}};Q(\td{\vec{w}}))    \]^{\nsamples} \[\int_{\bbR^{r+1}} \d\td{\vec{w}} P_{\td{\w}}\(\td{\vec{w}}\)\]^{\ndim}\,,
\end{align*}     
because the channel and the prior distributions factorize. Introducing the change of variable and the Fourier representation of the $\delta$-Dirac function, which involves a new ad-hoc parameter $\hat{Q}$:
 \begin{align*}
 	1 &= \int_{\bbR^{r+1 \times r+1}} \d Q \prod_{a \leq b} \delta \left(\ndim Q_{ab}-\sum_{i=1}^\ndim w_i^a w_i^b \right)\\
	&\propto \int_{\bbR^{r+1 \times r+1}} \d Q \int_{\bbR^{r+1 \times r+1}} \d\hat{Q} \exp \left(- \ndim \tr{Q \hat{Q}} \right)  \exp\left(\frac{1}{2}\sum_{i=1}^\ndim \tbf{\tilde{w}}_i^{\intercal} \hat{Q} \tbf{\tilde{w}}_i\right) \,,
 \end{align*}
the replicated partition function becomes an integral over the symmetric matrices $Q \in \bbR^{r+1 \times r+1}$ and $\hat{Q} \in \bbR^{r+1 \times r+1}$, that can be evaluated using a Laplace method in the $\ndim \to \infty$ limit,
\begin{align}
	\EE_{\vec{y},\mat{X}} \[  \mZ\(\vec{y}, \mat{X}\)^r \] &= \int_{\bbR^{r+1 \times r+1}} \d Q \int_{\bbR^{r+1 \times r+1}} \d\hat{Q} e^{\ndim \Phi^{(r)} (Q,\hat{Q} ) }\\
	&\underset{\ndim \to \infty}{\simeq} \exp\(\ndim \cdot  \extr_{Q, \hat{Q}} \left\{ \Phi^{(r)}(Q,\hat{Q}) \right\}\), 
	\label{appendix:expectation_Zn}
\end{align}
where we defined
\begin{equation}
     \begin{cases}
     \Phi^{(r)} (Q,\hat{Q}) = -\tr{Q\hat{Q}} +\log \Psi_{\w}^{(r)} (\hat{Q})+\alpha\log \Psi_{\out}^{(r)}(Q)
      \spacecase
      \Psi_{\w}^{(r)} (\hat{Q}) = \displaystyle \int_{\mathbb{R}^{r+1}} \d \tbf{\tilde{w}} P_{\tilde{\rm w}}(\tbf{\tilde{w}})  e^{ \frac{1}{2}\tbf{\tilde{w}}^{\intercal} \hat{Q} \tbf{\tilde{w}} }   \spacecase
      \Psi_{\out}^{(r)}(Q) =  \displaystyle \int \d y \int_{\mathbb{R}^{r+1}}  \d\tbf{\tilde{z}} P_{\tilde{z}}(\tbf{\tilde{z}};Q) P_{\out}(y| \tbf{\tilde{z}})\,,
     \end{cases}
     \label{appendix:replicas:Phi_r}
\end{equation}
and $P_{\td{z}} (\td{\vec{z}};Q) = \displaystyle \frac{e^{-\frac{1}{2}\td{\vec{z}}^\intercal Q^{-1} \td{\vec{z}} }}{\det(2\pi Q)^{1/2}}$.

Finally switching the two limits $r\to 0$ and $\ndim \to \infty$, the quenched free entropy $\Phi$ simplifies as a saddle point equation
\begin{equation}
\Phi (\alpha) = \extr_{ Q, \hat{Q} } \left\{\lim_{r\rightarrow 0} \frac{\partial \Phi^{(r)}(Q,\hat{Q})}{\partial  r} \right\} ,
\label{appendix:replicas:free_entropy_noansatz}
\end{equation}
over symmetric matrices $Q\in \bbR^{r+1 \times r+1}$ and $\hat{Q} \in \bbR^{r+1 \times r+1}$. In the following we will assume a simple ansatz for these matrices in order to first obtain an analytic expression in $r$ before taking the derivative with respect to $r$.

\paragraph{RS free entropy}
\label{appendix:replicas:rs_free_entropy}

Let's compute the functional $\Phi^{(r)}(Q,\hat{Q})$ appearing in the free entropy eq.~\eqref{appendix:replicas:free_entropy_noansatz} in the simplest ansatz: the Replica Symmetric ansatz. This later assumes that all replica remain equivalent with a common overlap $q = \frac{1}{\ndim} \vec{w}^a \cdot \vec{w}^b$ for $a \ne b$, a norm $Q= \frac{1}{\ndim} \|\vec{w}^a\|_2^2$, and an overlap with the ground truth $m =\frac{1}{\ndim} \vec{w}^a \cdot \vec{w}^\star$, leading to the following expressions of the replica symmetric matrices $Q_{\rm rs} \in \mathbb{R}^{r+1\times r+1}$ and $\hat{Q}_{\rm rs} \in \mathbb{R}^{r+1\times r+1}$:
\begin{equation}
\begin{aligned}[c]
Q_{\rm rs} =\begin{pmatrix} 
Q^0 & m & ... & m \\
 m & Q & ... & ...  \\
 ... &... & ... & q  \\
 m &... & q & Q    \\
\end{pmatrix}
\end{aligned}
\hspace{0.5cm}
\textrm{ and } 
\hspace{0.5cm}
\begin{aligned}[c]
\hat{Q}_{\rm rs}=\begin{pmatrix} 
 \hat{Q}^0 & \hat{m} & ... & \hat{m}\\
\hat{m} &-\frac{1}{2}\hat{Q} & ... & ...  \\
 ... &... & ... & \hat{q}  \\
\hat{m} &... & \hat{q} & -\frac{1}{2}\hat{Q}\\  
\end{pmatrix}\,,
\end{aligned}
\end{equation}
with $Q^0=\rho_{\w^\star} = \frac{1}{\ndim} \|\vec{w}^\star\|_2^2$.
Let's compute separately the terms involved in the functional $\Phi^{(r)}(Q,\hat{Q})$ eq.~\eqref{appendix:replicas:Phi_r} with this ansatz: the first is a trace term, the second a term  $\Psi_{\w}^{(r)}$ depending on the prior distributions $P_\w$, $P_{\w^\star}$ and finally the third a term $\Psi_{\out}^{(r)}$ that depends on the channel distributions $P_{\out^\star}$,$P_\out$.

\paragraph{Trace term} 
The trace term can be easily computed and takes the following form:
\begin{equation}
	\left.\Tr(Q\hat{Q}) \right|_{\rm rs} =  Q^0\hat{Q}^0 + r m \hat{m} - \frac{1}{2}r Q \hat{Q} + \frac{r(r-1)}{2} q \hat{q}\,.
\end{equation}

\paragraph{Prior integral} Evaluated at the RS fixed point, and using a Gaussian identity also known as a Hubbard-Stratonovich transformation $\EE_{\xi}\exp(\sqrt{a}\xi)=e^{\frac{a}{2}}$, the prior integral can be further simplified 
\begin{align}
\begin{aligned}
	\left.\Psi_{\w}^{(r)} (\hat{Q})\right|_{\rm rs} &= \displaystyle \int_{\mathbb{R}^{r+1}} d\tbf{\tilde{w}} P_{\tilde{\rm w}}(\tbf{\tilde{w}})  e^{ \frac{1}{2}\tbf{\tilde{w}}^{\intercal} \hat{Q}_{\rm rs} \tbf{\tilde{w}} }  \\
	&= \EE_{w^\star} e^{\frac{1}{2}\hat{Q}^0 (w^\star)^2} \int_{\mathbb{R}^{r}} d\tbf{\tilde{w}} P_{\tilde{\rm w}}(\tbf{\tilde{w}})  e^{ w^\star \hat{m} \sum_{a=1} ^r \td{w}^a - \frac{1}{2}(\hat{Q}+\hat{q}) \sum_{a=1}^r (\td{w}^a)^2 + \frac{1}{2}\hat{q} (\sum_{a=1}^r \td{w}^a)^2 } \\
	&= \displaystyle \EE_{\xi, w^\star} e^{\frac{1}{2}\hat{Q}^0 (w^\star)^2} \[ \EE_w  \exp \(\[ \hat{m} w^\star w - \frac{1}{2}(\hat{Q} + \hat{q}) w^2 + \hat{q}^{1/2} \xi w \] \) \]^r\,.
\end{aligned}
\end{align}

\paragraph{Channel integral}
Let's focus on the inverse matrix
\begin{equation}
   Q_{\rm rs}^{-1}=\begin{bmatrix}
   Q^{-1}_{00} & Q^{-1}_{01} & Q^{-1}_{01} & Q^{-1}_{01}  \\
   Q^{-1}_{01} & Q^{-1}_{11} & Q^{-1}_{12} & Q^{-1}_{12} \\
   Q^{-1}_{01} & Q^{-1}_{12} & Q^{-1}_{11} & Q^{-1}_{12} \\
   Q^{-1}_{01} & Q^{-1}_{12} & Q^{-1}_{12} & Q^{-1}_{11} \\
  \end{bmatrix}
\end{equation}

with 
\begin{align*}
\begin{cases}
	Q_{00}^{-1} &= \(Q^0 - r m (Q + (r-1) q)^{-1} m \)^{-1}   \spacecase
	Q_{01}^{-1} &= -\(Q^0 - r m(Q + (r-1) q)^{-1} m \)^{-1} m ( q + (r-1) q)^{-1}  \spacecase
	Q_{11}^{-1} &= (Q-q)^{-1} - (Q +(r-1) q)^{-1} q (Q-q)^{-1} \\
	& + ( Q + (r-1) q )^{-1} m  \(Q^0 - r m ( Q + (r-1) q)^{-1} 	m \)^{-1} m( Q + (r-1)q)^{-1}\spacecase
	Q_{12}^{-1} &= - ( Q + (r-1) q)^{-1} q (Q-q)^{-1} \\
	  &+  ( Q + (r-1)q)^{-1}m \(Q - r m(Q + (r-1) q)^{-1} m \)^{-1} m( Q + (r-1)q)^{-1}
	\end{cases}
\end{align*}
and its determinant:
\begin{align*}
	\det Q_{\rm rs} = \(Q-q\)^{r-1} \(Q + (r-1)q\) \(Q^0 - r m( Q + (r-1)q )^{-1} m  \)
\end{align*}
Using the same kind of Gaussian transformation, we obtain
\begin{align*}
\left. \Psi_{\out}^{(r)}(Q) \right|_{\rm rs} &= \displaystyle \int \d y \int_{\mathbb{R}^{r+1}}  d\tbf{\tilde{z}} e^{-\frac{1}{2}\td{\vec{z}}^\intercal Q_{\rm rs}^{-1} \td{\vec{z}} - \frac{1}{2} \log\(\det(2\pi Q_{\rm rs}) \)} P_{\out}(y| \tbf{\tilde{z}})\\
&= \displaystyle \EE_{y,\xi} \e^{- \frac{1}{2}\log(\det \( 2\pi Q_{\rm rs}\) )} \\
& \times \int \d z^\star P_{\out^\star}\(y | z^\star \) e^{ -\frac{1}{2} Q^{-1}_{00} (z^\star)^2 } \[ \int dz  P_{\out}\(y | z\) e^{ -  Q^{-1}_{01} z^\star z -\frac{1}{2}   \( Q^{-1}_{11} - Q^{-1}_{12} \)z^2  - Q^{-1/2}_{12} \xi z } \]^r
\end{align*}

\subsection{ERM and Bayes-optimal free entropy}

Taking carefully the derivative and the $r\to 0$ limit imposes $\hat{Q}^0=0$ and we finally obtain the replica symmetric free entropy $\Phi_{\rm rs}$:
\begin{align}
	\Phi_{\rm rs}(\alpha) &\equiv \EE_{\vec{y},\mat{X}} \[ \lim_{\ndim \to \infty} \frac{1}{\ndim} \log\( \mZ\(\vec{y}, \mat{X}\) \) \] \label{appendix:replicas:free_entropy_non_bayes}\\
	&=  \extr_{Q,\hat{Q},q,\hat{q},m,\hat{m}} \left \{  - m \hat{m} + \frac{1}{2} Q \hat{Q} + \frac{1}{2} q \hat{q}  + \Psi_\w\(\hat{Q},\hat{m},\hat{q}  \) + \alpha \Psi_\out\(Q,m,q ;  \rho_{\w^\star}\) \right\} \nonumber \,,
\end{align}
where $\rho_{\w^\star}= \lim_{\ndim \to \infty} \EE_{\vec{w}^\star} \frac{1}{\ndim} \|\vec{w}^\star\|_2^2$ and the channel and prior integrals are defined by
\begin{align}
\begin{aligned}
	\Psi_\w\(\hat{Q},\hat{m},\hat{q}  \) &\equiv \EE_{\xi} \[ \mZ_{\w^\star}\( \hat{m}  \hat{q}^{-1/2}  \xi , \hat{m} \hat{q}^{-1} \hat{m}  \) \log \mZ_{\w} \(  \hat{q}^{1/2}\xi  , \hat{Q} + \hat{q} \)   \]\,, \spacecase
	\Psi_\out\(Q,m,q; \rho_{\w^\star}\) &\equiv \EE_{y, \xi } \[ \mZ_{\out^\star} \( y,  m q^{-1/2}\xi, \rho_{\w^\star} - m q^{-1}m  \) \log \mZ_{\out} \( y,  q^{1/2}\xi, Q - q \)  \]\,,
	\label{appendix:replicas:free_entropy_terms_non_bayes}
\end{aligned}
\end{align}
where again $\mZ_{\out^\star}$ and $\mZ_{\w^\star}$ are defined in eq.~\eqref{appendix:definitions:update_bayes} and depend on the \emph{teacher}, while the denoising functions $\mZ_{\out}$ and $\mZ_{\w}$ depend on the inference model. In particular, we explicit in the next sections the above free entropy in the case of ERM and Bayes-optimal estimation.

\subsubsection{ERM estimation}
As described in eq.~\eqref{appendix:ERM_map_mapping}, the free entropy for ERM estimation is therefore given by eq.~\eqref{appendix:replicas:free_entropy_non_bayes} if we take $-\log \bbP\(\vec{y} | \vec{z} \) = l(\vec{y}, \vec{z}) $ and $- \log \bbP\(\vec{w}\) = r(\vec{w})$. As described in Sec.~\ref{appendix:definitions:map_updates} they lead to the following partition functions:
\begin{align}
\begin{aligned}
	\mZ_\w^\lambda \(\gamma, \Lambda\) &=  \lim_{\Delta \to 0}  e^{-\frac{1}{\Delta} \mM_{\Lambda^{-1}}\[ r(\lambda,.) \](\Lambda^{-1}\gamma) } e^{-\frac{1}{2\Delta} \gamma^2 \Lambda^{-1}}\,,\spacecase 
	\mZ_\out \(y, \omega,V\) &= \lim_{\Delta \to 0} \frac{e^{-\frac{1}{\Delta} \mM_{\frac{V}{\Delta}}[l(y,.)](\omega) } }{\sqrt{2\pi V }\sqrt{2\pi \Delta}}\,,
\end{aligned}	
\end{align}
with the Moreau-Yosida regularization \eqref{appendix:definitions:moreau}.

\subsubsection{Bayes-optimal estimation}
In the Bayes-optimal case, we have access to the ground truth distributions $\bbP\(\vec{y} | \vec{z} \) = P_{\out^\star}\(\vec{y} | \vec{z}\) $ and $\bbP\(\vec{w}\) = P_{\w^\star}(\vec{w})$, and therefore $\mZ_{\out}=\mZ_{\out^\star}$, $\mZ_{\w}=\mZ_{\w^\star}$. 
Nishimori conditions in the Bayes-optimal case \cite{Nishimori_1980} imply that $Q=\rho_{\w^\star}$, $m=q=q_\bayes$, $\hat{Q}=0$, $\hat{m}=\hat{q}=\hat{q}_\bayes$.
Therefore the free entropy eq.~\eqref{appendix:replicas:free_entropy_non_bayes} simplifies as an optimization problem over two scalar \emph{overlaps} $q_\bayes,\hat{q}_\bayes$:
\begin{align}
	\Phi^\bayes(\alpha) &=  \extr_{q_\bayes,\hat{q}_\bayes} \left \{ - \frac{1}{2} q_\bayes \hat{q}_\bayes  + \Psi_{\w}^\bayes\(\hat{q}_\bayes  \) + \alpha \Psi_{\out}^\bayes\(q_\bayes; \rho_{\w^\star}\) \right\} \,,
	\label{appendix:free_entropy_bayes}
\end{align}
with free entropy terms $\Psi_{\w}^\bayes$ and $\Psi_{\out}^\bayes$ given by
\begin{align*}
		\Psi_{\w}^\bayes\(\hat{q}\) &= \EE_{\xi} \[ \mZ_{\w^\star} \(  \hat{q}^{1/2}\xi,   \hat{q} \) \log \mZ_{\w^\star} \(  \hat{q}^{1/2}\xi,   \hat{q} \)   \] \,, \spacecase 
		\Psi_{\out}^\bayes \(q; \rho_{\w^\star}\) &= \EE_{y, \xi } \[ \mZ_{\out^\star} \( y,  q^{1/2}\xi, \rho_{\w^\star} - q \) \log \mZ_{\out^\star} \( y,  q^{1/2}\xi, \rho_{\w^\star} - q \)  \]\,.
\end{align*} 
and again $\mZ_{\out^\star}$ and $\mZ_{\w^\star}$ are defined in eq.~\eqref{appendix:definitions:update_bayes}. The above replica symmetric free entropy in the Bayes-optimal case has been rigorously proven in \cite{Barbier5451}.

\subsection{Sets of fixed point equations}
As highlighted in Sec.~\ref{appendix:generalization_error}, the asymptotic overlaps $m, q$ measure the performances of the ERM or Bayes-optimal statistical estimators, whose behaviours are respectively characterized by extremizing the free entropy \eqref{appendix:replicas:free_entropy_non_bayes} and \eqref{appendix:free_entropy_bayes}.
This section is devoted to derive the corresponding sets of fixed point equations.

\subsubsection{ERM estimation}
\label{appendix:fixed_point_erm}
Extremizing the free entropy eq.~\eqref{appendix:replicas:free_entropy_non_bayes}, we easily obtain the set of six fixed point equations
\begin{align}
\begin{aligned}
	\hat{Q} &= -2 \alpha \partial_Q \Psi_{\out}\,, \hspace{2cm}
			&& Q = - 2  \partial_{\hat{Q}} \Psi_{\w}\spacecase
	\hat{q} &= -2 \alpha \partial_q \Psi_{\out}\,, 
			&& q =   -2\partial_{\hat{q}} \Psi_{\w} \,,\spacecase
	\hat{m} &= \alpha \partial_{m} \Psi_{\out}\,,
			&& m = \partial_{\hat{m}} \Psi_{\w} \,.
\end{aligned}
\label{appendix:se_equations_generic_not_simplified}
\end{align}
These equations can be formulated as functions of the partition functions $\mZ_{\out^\star}$, $\mZ_{\w^\star}$ and the denoising functions $f_{\out^\star}, f_{\w^\star}, f_\out, f_{\w}$ defined in eq.~\eqref{appendix:definitions:update_functions_bayes} and eq.~\eqref{appendix:definitions:update_functions_map}. 
The derivation is shown in Appendix.~\ref{appendix:annex:state_evolution:simplications} and defining the natural variables $\Sigma = Q - q$, $\hat{\Sigma}= \hat{Q}+\hat{q}$, $\eta \equiv \frac{m^2}{\rho_{\w^\star}q}$ and $\hat{\eta} \equiv \frac{\hat{m}^2}{\hat{q}}$, it can be written as
\begin{align}
\begin{aligned}
	m&= \EE_{\xi} \[ \mZ_{\w^\star} \(\sqrt{\hat{\eta}}  \xi , \hat{\eta} \) f_{\w^\star}\(\sqrt{\hat{\eta}}  \xi , \hat{\eta} \) f_{\w} \(  \hat{q}^{1/2}\xi  , \hat{\Sigma} \)   \]\,, \\
	q &= \EE_{\xi} \[ \mZ_{\w^\star}\( \sqrt{\hat{\eta}}  \xi , \hat{\eta}  \) f_{\w} \(  \hat{q}^{1/2}\xi  , \hat{\Sigma} \)^2   \]\,, \\
	\Sigma &=  \EE_{\xi} \[ \mZ_{\w^\star}\( \sqrt{\hat{\eta}}  \xi , \hat{\eta}  \)  \partial_\gamma f_{\w} \(  \hat{q}^{1/2}\xi  , \hat{\Sigma} \) \]\,, \\
	\hat{m} &= \alpha \EE_{y, \xi } \[ \mZ_{\out^\star}(.) \cdot f_{\out^\star} \(y,  \sqrt{\rho_{\w^\star} \eta} \xi, \rho_{\w^\star}\(1 - \eta\)\)   f_{\out} \( y,  q^{1/2}\xi, \Sigma \) \]\,, \\
	\hat{q} &= \alpha \EE_{y, \xi } \[ \mZ_{\out^\star} \( y,  \sqrt{\rho_{\w^\star} \eta} \xi, \rho_{\w^\star}\(1 - \eta\)  \)   f_{\out} \( y,  q^{1/2}\xi, \Sigma \)^2 \]\,, \\
	\hat{\Sigma} &=  - \alpha \EE_{y, \xi } \[ \mZ_{\out^\star} \( y,  \sqrt{\rho_{\w^\star} \eta} \xi, \rho_{\w^\star}\(1 - \eta\)  \)  \partial_\omega f_{\out} \( y,  q^{1/2}\xi, \Sigma \) \]\,,
\end{aligned}
\label{appendix:se_equations_generic}
\end{align}
and we finally obtain the set of equations eqs.~\eqref{appendix:fixed_point_replicas}.

\subsubsection{Bayes-optimal estimation}
Extremizing the Bayes-optimal free entropy eq.~\eqref{appendix:free_entropy_bayes}, we easily obtain the set of 2 fixed point equations over the scalar parameters $q_\bayes, \hat{q}_\bayes$. In fact, it can also be deduced from eq.~\eqref{appendix:se_equations_generic} using the Nishimori conditions $f_{\w}=f_{\w^\star}$, $f_{\out}=f_{\out^\star}$, $m=q=q_\bayes, \Sigma = \rho_{\w^\star} - q, \hat{m}=\hat{q}=\hat{q}_\bayes$ and $\hat{Q}=0$ that lead to the result \eqref{main:fixed_point_equations_bayes} in  Thm.~\ref{main:thm:fixed_point_equations_bayes}, from \cite{Barbier5451}
\begin{align}
\begin{aligned}
	\hat{q}_\bayes &= \alpha \EE_{y, \xi } \[ \mZ_{\out^\star} \( y,  q_\bayes^{1/2}\xi, \rho_{\w^\star} - q_\bayes \)   f_{\out^\star} \( y,  q_\bayes^{1/2}\xi, \rho_{\w^\star} - q_\bayes \)^2    \] \,, \spacecase
	q_\bayes &= \EE_{\xi} \[ \mZ_{\w^\star}\( \hat{q}_\bayes^{1/2}  \xi , \hat{q}_\bayes  \) f_{\w^\star} \(  \hat{q}_\bayes^{1/2}\xi, \hat{q}_\bayes\)^2   \] \,.
\end{aligned}
\label{appendix:se_equations_generic:bayes}
\end{align}

\subsection{Useful derivations}
In this section, we give useful computation steps that we used to transform the sets of fixed point equations \eqref{appendix:se_equations_generic_not_simplified}.

\subsubsection{Prior free entropy term}
\label{appendix:annex:free_entropy}
In specific simple cases, the prior free entropy term 
\begin{align*}
	\Psi_\w\(\hat{Q},\hat{m},\hat{q}  \) &\equiv \EE_{\xi} \[ \mZ_{\w^\star}\( \hat{m}  \hat{q}^{-1/2}  \xi , \hat{m} \hat{q}^{-1} \hat{m}  \) \log \mZ_{\w} \(  \hat{q}^{1/2}\xi  , \hat{Q} + \hat{q} \)   \]
\end{align*}
in \eqref{appendix:replicas:free_entropy_terms_non_bayes} can be computed explicitly. This is the case of Gaussian and binary priors $P_{\w^\star}$ with $\rL_2$ regularization. In particular, they lead surprisingly to the same expression meaning that choosing a binary or Gaussian teacher distribution does not affect the ERM performances with $\rL_2$ regularization. 

\paragraph{Gaussian prior}
\label{appendix:sec:annex:free_entropy:gaussian}
Let us compute the corresponding free entropy term with partition functions 
$\mZ_{\w^\star}$ for a Gaussian prior $P_{\w^\star}(w^\star) = \mN_{w^\star}(0,\rho_{\w^\star})$ and $\mZ_{\w}^{\rL_2, \lambda}$ for a $\rL_2$ regularization respectively given by eq.~\eqref{appendix:definitions:application:gaussian} and eq.~\eqref{appendix:definitions:application:L2}:
\begin{align*}
	\begin{aligned}
		\mZ_{\w^\star}\( \gamma , \Lambda  \) &= \frac{e^{\frac{\gamma^2\rho_{\w^\star}}{2\( \Lambda\rho_{\w^\star} + 1 \)}}}{\sqrt{\Lambda\rho_{\w^\star} + 1 }} \,, 
		&& \mZ_{\w}^{\rL_2, \lambda}\( \gamma , \Lambda  \) = \frac{e^{\frac{\gamma^2}{2\( \Lambda + \lambda \)}}}{\sqrt{\Lambda + \lambda }} \,.
	\end{aligned}
\end{align*}
The prior free entropy term reads
\begin{align}
\begin{aligned}
	\Psi_{\w}\(\hat{Q},\hat{m},\hat{q}  \) &= \EE_{\xi} \[ \mZ_{\w^\star}\( \hat{m}  \hat{q}^{-1/2}  \xi , \hat{m} \hat{q}^{-1} \hat{m}  \) \log \mZ_{\w}^{\rL_2, \lambda} \(  \hat{q}^{1/2}\xi  , \hat{q} + \hat{Q} \)   \]\\
	&= \EE_{\xi} \[ \mZ_{\w^\star}\( \hat{m}  \hat{q}^{-1/2}  \xi , \hat{m} \hat{q}^{-1} \hat{m}  \) \(\frac{ \hat{q} \xi^2 }{2\(\lambda + \hat{Q} + \hat{q}  \)} - \frac{1}{2} \log\(\lambda + \hat{Q} + \hat{q} \) \) \]\\
	&= \int \d\xi \mN_\xi\(0,1+\rho_{\w^\star}\hat{m}^2\hat{q}^{-1}\) \(\frac{ \hat{q} \xi^2 }{2\(\lambda + \hat{Q} + \hat{q} \)} - \frac{1}{2} \log\(\lambda + \hat{Q} + \hat{q}\) \)\\
	&= \frac{1}{2}\(  \frac{\hat{q}+\rho_{\w^\star}\hat{m}^2}{\lambda + \hat{Q} + \hat{q}} - \log\( \lambda + \hat{Q} + \hat{q} \) \)
\end{aligned}
\label{appendix:annex:free_entropy:gaussian}
\end{align}

In the Bayes-optimal case for $\rho_{\w^\star}=1$, the computation is similar and is given by the above expression with $\lambda=1$, $\hat{Q}=0$, $\hat{m}=\hat{q}$:
\begin{align}
\begin{aligned}
	\Psi_{\w}^{\rm bayes}\(\hat{q}\) &=
	&= \frac{1}{2}\( \hat{q} - \log\( 1 +  \hat{q} \) \)
\end{aligned}
\label{appendix:annex:free_entropy:gaussian:bayes}
\end{align}

\paragraph{Binary prior}
\label{appendix:sec:annex:free_entropy:binary}

Let us compute the corresponding free entropy term with partition functions 
$\mZ_{\w^\star}$ for a binary prior $P_{\w^\star}(w^\star) = \frac{1}{2}\(\delta(w^\star-1) + \delta(w^\star+1) \)$ and $\mZ_{\w}^{\rL_2, \lambda}$ for a $\rL_2$ regularization respectively given by eq.~\eqref{appendix:definitions:application:binary} and eq.~\eqref{appendix:definitions:application:L2}:
\begin{align*}
	\begin{aligned}
		\mZ_{\w^\star}\( \gamma , \Lambda  \) &=e^{-\frac{\Lambda }{2}} \cosh (\gamma)\,, && \mZ_{\w}^{\rL_2, \lambda} \( \gamma , \Lambda  \) = \frac{e^{\frac{\gamma^2}{2\( \Lambda + \lambda \)}}}{\sqrt{\Lambda + \lambda }} \,.
	\end{aligned}
\end{align*}
The entropy term $\Psi_{\w}$ reads
\begin{align}
\begin{aligned}
	\Psi_{\w}\(\hat{Q},\hat{m},\hat{q}  \) &= \EE_{\xi} \[ \mZ_{\w^\star}\( \hat{m}  \hat{q}^{-1/2}  \xi , \hat{m} \hat{q}^{-1} \hat{m}  \) \log \mZ_{\w}^{\rL_2, \lambda} \(  \hat{q}^{1/2}\xi  , \hat{q} + \hat{Q} \)   \]\\
	&= \EE_{\xi} \[ \mZ_{\w^\star}\( \hat{m}  \hat{q}^{-1/2}  \xi , \hat{m} \hat{q}^{-1} \hat{m}  \) \(\frac{ \hat{q} \xi^2 }{2\(\lambda + \hat{Q} + \hat{q}  \)} - \frac{1}{2} \log\(\lambda + \hat{Q} + \hat{q} \) \) \]\\
	&= \int d\xi \frac{e^{-\frac{\xi^2}{2}}}{\sqrt{2\pi}} e^{-\frac{ \hat{m} \hat{q}^{-1} \hat{m}}{2}} \cosh \( \hat{m}  \hat{q}^{-1/2}  \xi \)  \(\frac{ \hat{q} \xi^2 }{2\(\lambda + \hat{Q} + \hat{q} \)} - \frac{1}{2} \log\(\lambda + \hat{Q} + \hat{q}\) \)\\
	&= \frac{1}{2}\(  \frac{\hat{q}+\hat{m}^2}{\lambda + \hat{Q} + \hat{q}} - \log\( \lambda + \hat{Q} + \hat{q} \) \)
\end{aligned}
\label{appendix:annex:free_entropy:binary}
\end{align}
We recover exactly the same free entropy term than for Gaussian prior teacher eq.~\eqref{appendix:annex:free_entropy:gaussian} for $\rho_{\w^\star}=1$.

\subsubsection{Updates derivatives}
Let's compute, in full generality, the derivative of the partition functions defined in Sec.~\ref{appendix:definitions:update_generic} and that will be useful to simplify the set \eqref{appendix:se_equations_generic_not_simplified}.
\begin{align}
	\begin{aligned}
	\partial_\gamma \mZ_{\w}\(\gamma, \Lambda\) &= \mZ_{\w}\(\gamma, \Lambda\) \times  \EE_{Q_{\w}}\[w\] = \mZ_{\w}\(\gamma, \Lambda\) f_{\w}\(\gamma, \Lambda\)\\
	\partial_\Lambda \mZ_{\w}\(\gamma, \Lambda\) &= -\frac{1}{2} \mZ_{\w}\(\gamma, \Lambda\) \times \EE_{Q_{\w}}\[w^2\] = -\frac{1}{2} \( \partial_\gamma f_{\w} (\gamma ,\Lambda) + f_{\w}^2(\gamma ,\Lambda)  \)\\
	\partial_\omega \mZ_{\out}\(y, \omega, V\)  &= \mZ_{\out}\(y, \omega, V\) \times  V^{-1} \EE_{Q_{\out}} \[ z - \omega\] \\
		&= \mZ_{\out}\(y, \omega, V\) f_{\out}\(y, \omega, V\) \\
	\partial_V \mZ_{\out}\(y, \omega, V\)  &= \frac{1}{2} \mZ_{\out}\(y, \omega, V\) \times \( \EE_{Q_{\out}}\[V^{-2}(z-\omega)^2\] - V^{-1}\)\\
		&= \frac{1}{2} \mZ_{\out}\(y, \omega, V\) \( \partial_\omega f_{\out}\(y, \omega, V\)  + f_{\out}^2\(y, \omega, V\)  \)
	\end{aligned}
	\label{appendix:annex:derivatives}
\end{align}

\subsubsection{Simplifications of the fixed point equations}
\label{appendix:annex:state_evolution:simplications}
We recall the set of fixed point equations eq.~\eqref{appendix:se_equations_generic_not_simplified}
\begin{align}
\begin{aligned}
	\hat{Q} &= -2 \alpha \partial_Q \Psi_{\out}\,, \hspace{2cm}
			&& Q = - 2  \partial_{\hat{Q}} \Psi_{\w}\spacecase
	\hat{q} &= -2 \alpha \partial_q \Psi_{\out}\,, 
			&& q =   -2\partial_{\hat{q}} \Psi_{\w} \,,\spacecase
	\hat{m} &= \alpha \partial_{m} \Psi_{\out}\,,
			&& m = \partial_{\hat{m}} \Psi_{\w} \,,
\end{aligned}
\end{align}
that can be simplified and formulated as functions of $\mZ_{\out^\star}$, $\mZ_{\w^\star}$,$f_{\out^\star}, f_{\w^\star}, f_\out$, and $f_{\w}$ defined in eq.~\eqref{appendix:definitions:update_functions_bayes} and eq.~\eqref{appendix:definitions:update_functions_map}, using the derivatives in \eqref{appendix:annex:derivatives}.

\paragraph{Equation over $\hat{q}$}
\begin{align*}
	\partial_q \Psi_{\out} &= \partial_q \EE_{y, \xi } \[  \mZ_{\out^\star} \( y,  m q^{-1/2}\xi, \rho_{\w^\star} - m q^{-1}m  \)  \log \mZ_{\out} \( y,  q^{1/2}\xi, Q - q \)  \] \\
	&= \EE_{y, \xi} \[  \partial_q \omega^\star \partial_\omega \mZ_{\out^\star} \log \mZ_{\out} + \partial_q V^\star \partial_V \mZ_{\out^\star} \log \mZ_{\out}\right. \\
	& \hhspace \left. + \frac{\mZ_{\out^\star}}{\mZ_{\out}} \(\partial_q \omega \partial_\omega \mZ_{\out} + \partial_q V \partial_V \mZ_{\out} \)    \] \\
	&= \EE_{y, \xi} \[ - \frac{m}{2}q^{-3/2} \xi f_{\out^\star} \mZ_{\out^\star} \log \mZ_{\out} + \frac{m^2 q^{-2}}{2}\( \partial_\omega f_{\out^\star} + f_{\out^\star}^2 \)  \mZ_{\out^\star} \log \mZ_{\out} \right. \\
	& \hhspace \left. +  \frac{\mZ_{\out^\star}}{\mZ_{\out}} \( \frac{1}{2}q^{-1/2} \xi f_{\out} \mZ_{\out} - \frac{1}{2} \(  \partial_\omega f_{\out} + f_{\out}^2  \)\mZ_{\out} \)    \] \\
	&= \frac{1}{2} \EE_{y, \xi} \[ - m^2q^{-2}  \partial_\xi \( f_{\out^\star} \mZ_{\out^\star} \log \mZ_{\out}\) + m^2q^{-2}\( \partial_\omega f_{\out^\star} + f_{\out^\star}^2 \) \mZ_{\out^\star} \log \mZ_{\out} \right. \\
	& \hhspace \left. +   \( \partial_\xi \( f_{\out} \mZ_{\out^\star}\)   - \(  \partial_\omega f_{\out} + f_{\out}^2  \) \mZ_{\out^\star} \)    \]  \tag{Stein lemma} \\
	&= \frac{1}{2} \EE_{y, \xi} \[ - m^2q^{-2} \( \partial_\omega f_{\out^\star} \log \mZ_{\out} + \mZ_{\out^\star} f_{\out^\star}^2 \log \mZ_{\out} \right. \right. \\
	& \hhspace\left. \left. - \( \partial_\omega f_{\out^\star} + f_{\out^\star}^2 \) \mZ_{\out^\star} \log \mZ_{\out}   \)   \] +  \frac{1}{2} \EE_{y, \xi} \[ - mq^{-1} \mZ_{\out^\star} f_{\out^\star} f_{\out}  \] \\
	& \hhspace + \frac{1}{2} \EE_{y, \xi} \[ \partial_\omega f_{\out}  \mZ_{\out} + mq^{-1}  \mZ_{\out^\star} f_{\out^\star} f_{\out}  -  \(  \partial_\omega f_{\out} + f_{\out}^2  \) \mZ_{\out^\star}  \]  \\
	&= - \frac{1}{2} \EE_{y, \xi} \[    \mZ_{\out^\star}\( y,  m q^{-1/2}\xi, \rho_{\w^\star} - m q^{-1}m  \)  f_{\out}^2\( y,  q^{1/2}\xi, Q - q \)       \]\,, \tag{Simplifications with \eqref{appendix:annex:derivatives}}
\end{align*}
that leads to 
\begin{align}
	\hat{q} &= -2 \alpha \partial_q \Psi_{\out} = \alpha \EE_{y, \xi } \[ \mZ_{\out^\star} \( y,  m q^{-1/2}\xi, \rho_{\w^\star} - m q^{-1}m  \)   f_{\out} \( y,  q^{1/2}\xi, Q - q \)^2    \]\,.
\end{align}

\paragraph{Equation over $\hat{m}$}
\begin{align*}
	\partial_m \Psi_{\out} &= \EE_{y, \xi } \[ \partial_m \mZ_{\out^\star} \( y,  m q^{-1/2}\xi, \rho_{\w^\star} - m q^{-1}m  \)  \log \mZ_{\out} \( y,  q^{1/2}\xi, Q - q \)  \] \\
	&= \EE_{y, \xi} \[ \( \partial_m \omega^\star \partial_\omega \mZ_{\out^\star} + \partial_m V^\star \partial_V \mZ_{\out^\star}  \) \log \mZ_{\out}  \] \\
	&= \EE_{y, \xi} \[ \( q^{-1/2}\xi f_{\out^\star} \mZ_{\out^\star} - m q^{-1} \( \partial_\omega f_{\out^\star} + f_{\out^\star}^2 \) \mZ_{\out^\star}  \) \log \mZ_{\out}  \] \\
	&= \EE_{y, \xi} \[ \partial_\xi  \(f_{\out^\star} \mZ_{\out^\star} \log \mZ_{\out} \) -  \( \partial_\omega f_{\out^\star} + f_{\out^\star}^2 \) \mZ_{\out^\star}  \log \mZ_{\out}  \] \tag{Stein Lemma} \\
	&= \EE_{y, \xi } \[ m q^{-1}\(\partial_\omega f_{\out^\star} \mZ_{\out^\star} \log \mZ_{\out} + f_{\out^\star} \partial_\omega\mZ_{\out^\star} \log \mZ_{\out} \right.\right. \\
	& \hhspace \left. \left. -  \( \partial_\omega f_{\out^\star} + f_{\out^\star}^2 \) \mZ_{\out^\star}  \) \log \mZ_{\out}    \] + \EE_{y, \xi } \[  \mZ_{\out^\star}  f_{\out^\star} f_{\out}  \]  \\ 
	&= \EE_{y, \xi} \[ \mZ_{\out^\star} \( ., ., . \) f_{\out^\star} \( y,  m q^{-1/2}\xi, \rho_{\w^\star} - m q^{-1}m  \)  f_{\out} \( y,  q^{1/2}\xi, Q - q \)  \] \tag{Simplifications with \eqref{appendix:annex:derivatives}} 
\end{align*}
that leads to 
\begin{align}
\begin{aligned}
	\hat{m} &=\alpha \partial_{m} \Psi_{\out} \\
	&= \alpha \EE_{y, \xi } \[ \mZ_{\out^\star} \(., ., . \) f_{\out^\star} \( y,  m q^{-1/2}\xi, \rho_{\w^\star} - m q^{-1}m  \)   f_{\out} \( y,  q^{1/2}\xi, Q - q \)    \]\,.
\end{aligned}
\end{align}

\paragraph{Equation over $\hat{Q}$}
\begin{align*}
\partial_Q \Psi_{\out} &= \EE_{y, \xi } \[ \mZ_{\out^\star} \( y,  m q^{-1/2}\xi, \rho_{\w^\star} - m q^{-1}m  \) \partial_Q \log \mZ_{\out} \( y,  q^{1/2}\xi, Q - q \)  \] \\
	&= \EE_{y, \xi } \[ \mZ_{\out^\star} \( y,  m q^{-1/2}\xi, \rho_{\w^\star} - m q^{-1}m  \) \partial_Q V \partial_V \log \mZ_{\out} \( y,  q^{1/2}\xi, Q-q \)  \] \\
	&= \frac{1}{2}  \EE_{y, \xi } \[ \mZ_{\out^\star} \( y,  m q^{-1/2}\xi, \rho_{\w^\star} - m q^{-1}m  \) \(  \partial_\omega f_{\out} + f_{\out}^2  \) \( y,  q^{1/2}\xi, Q-q \) \]
\end{align*}
leading to
\begin{align}
\begin{aligned}
	\hat{Q} &= -2 \alpha \partial_Q \Psi_{\out}	\\
	&= - \alpha \EE_{y, \xi } \[ \mZ_{\out^\star} \( y,  m q^{-1/2}\xi, \rho_{\w^\star} - m q^{-1}m  \)  \partial_\omega f_{\out} \( y,  q^{1/2}\xi, Q - q \)    \] - \hat{q} \,.
\end{aligned}
\end{align}

\paragraph{Equation over $q$}
\begin{align*}
	\partial_{\hat{q}} \Psi_{w} &= \partial_{\hat{q}} \EE_{\xi} \[  \mZ_{\w^\star}\( \hat{m}  \hat{q}^{-1/2}  \xi , \hat{m} \hat{q}^{-1} \hat{m} \)  \log \mZ_{\w}\(\hat{q}^{1/2}\xi,\hat{Q} + \hat{q}\)  \] \\
	&= \EE_{\xi} \[  \partial_{\hat{q}} \omega^\star \partial_\omega \mZ_{\w^\star} \log \mZ_{\w} + \partial_{\hat{q}} V^\star \partial_V \mZ_{\w^\star} \log \mZ_{\w} + \frac{\mZ_{\w^\star}}{\mZ_{\w}} \(\partial_{\hat{q}} \omega \partial_\omega \mZ_{\w} + \partial_{\hat{q}} V \partial_V \mZ_{\w} \)    \] \\
	&= \EE_{\xi} \[ - \frac{\hat{m}}{2}\hat{q}^{-3/2} \xi f_{\w^\star} \mZ_{\w^\star} \log \mZ_{\w} + \frac{\hat{m}^2 \hat{q}^{-2}}{2}\( \partial_\omega f_{\w^\star} + f_{\w^\star}^2 \)  \mZ_{\w^\star} \log \mZ_{\w} \right. \\
	&\left. +  \frac{\mZ_{\w^\star}}{\mZ_{\w}} \( \frac{1}{2}\hat{q}^{-1/2} \xi f_{\w} \mZ_{\w} - \frac{1}{2} \(  \partial_\omega f_{\w} + f_{\w}^2  \)\mZ_{\w} \)    \] \\
	&= \EE_{\xi} \[ - \frac{\hat{m}}{2}\hat{q}^{-3/2} \partial_\xi \( f_{\w^\star} \mZ_{\w^\star} \log \mZ_{\w}\) + \frac{\hat{m}^2 \hat{q}^{-2}}{2}\( \partial_\omega f_{\w^\star} + f_{\w^\star}^2 \)  \mZ_{\w^\star} \log \mZ_{\w} \right. \\
	&\left. +   \( \frac{1}{2}\hat{q}^{-1/2} \partial_\xi \( f_{\w} \mZ_{\w^\star}\) - \frac{1}{2} \(  \partial_\omega f_{\w} + f_{\w}^2  \)\mZ_{\w^\star} \)    \] \tag{Stein lemma} \\
	&= \frac{1}{2} \EE_{\xi} \[ - \hat{m}^2\hat{q}^{-2} \(  \partial_\omega f_{\w^\star} \mZ_{\w^\star} \log \mZ_{\w} + \mZ_{\w^\star} f_{\w^\star}^2 \log \mZ_{\w}  -  \( \partial_\omega f_{\w^\star} + f_{\w^\star}^2 \) \mZ_{\w^\star} \log \mZ_{\w} \)  \right. \\
	&\left. - \hat{m}\hat{q}^{-1} \mZ_{\w^\star} f_{\w^\star} f_{\w} +    \( \hat{m} \hat{q}^{-1}  \mZ_{\w^\star} f_{\w} f_{\w^\star}   + \mZ_{\w^\star} \partial_\omega f_{\w}   - \(  \partial_\omega f_{\w} + f_{\w}^2  \) \mZ_{\w^\star} \)    \] \\
	&= - \frac{1}{2} \EE_{\xi} \[    \mZ_{\w^\star}\( \hat{m}  \hat{q}^{-1/2}  \xi , \hat{m} \hat{q}^{-1} \hat{m} \)   f_{\w}\(\hat{q}^{1/2}\xi,\hat{Q} + \hat{q}\)^2       \] \tag{Simplifications with \eqref{appendix:annex:derivatives}}
\end{align*}
leading to
\begin{align}
	q &= - 2\partial_{\hat{q}} \Psi_{\w}= \EE_{\xi} \[ \mZ_{\w^\star}\( \hat{m}  \hat{q}^{-1/2}  \xi , \hat{m} \hat{q}^{-1} \hat{m}  \) f_{\w} \(  \hat{q}^{1/2}\xi  , \hat{q} + \hat{Q} \)^2   \]
\end{align}

\paragraph{Equation over $m$}
\begin{align*}
	\partial_{\hat{m}} \Psi_{w} &= \partial_m \EE_{\xi} \[  \mZ_{\w^\star}\( \hat{m}  \hat{q}^{-1/2}  \xi , \hat{m} \hat{q}^{-1} \hat{m} \)  \log \mZ_{\w}\(\hat{q}^{1/2}\xi,\hat{Q} + \hat{q}\)  \] \\
	&= \EE_{\xi} \[ \( \partial_{\hat{m}} \omega^\star \partial_\omega \mZ_{\w^\star} + \partial_{\hat{m}} V^\star \partial_V \mZ_{\w^\star}  \) \log \mZ_{\w}  \] \\
	&= \EE_{\xi} \[ \( \hat{q}^{-1/2}\xi f_{\w^\star} \mZ_{\w^\star} - \hat{m} \hat{q}^{-1} \( \partial_\omega f_{\w^\star} + f_{\w^\star}^2 \) \mZ_{\w^\star}  \) \log \mZ_{\w}  \]  \\
	&= \EE_{\xi} \[ \hat{m} \hat{q}^{-1} \partial_\xi  \(f_{\w^\star} \mZ_{\w^\star} \log \mZ_{\w} \) -  \( \partial_\omega f_{\w^\star} + f_{\w^\star}^2 \) \mZ_{\w^\star}  \log \mZ_{\w}  \] \tag{Stein Lemma} \\
	&= \EE_{\xi} \[ \hat{m} \hat{q}^{-1} \( \partial_\omega f_{\w^\star} \mZ_{\w^\star} \log \mZ_{\w} +  \mZ_{\w^\star} f_{\w^\star}^2  \log \mZ_{\w}   -  \( \partial_\omega f_{\w^\star} + f_{\w^\star}^2 \) \mZ_{\w^\star}  \log \mZ_{\w}    \) \right.\\
	& \left. \hspace{1cm}
	  +  \mZ_{\w^\star}  f_{\w^\star} f_{\w} \] \\
	&= \EE_{\xi } \[ \mZ_{\w^\star}\( \hat{m}  \hat{q}^{-1/2}  \xi , \hat{m} \hat{q}^{-1} \hat{m} \) f_{\w^\star}\( \hat{m}  \hat{q}^{-1/2}  \xi , \hat{m} \hat{q}^{-1} \hat{m} \)  f_{\w}\(\hat{q}^{1/2}\xi,\hat{Q} + \hat{q}\)  \] \tag{Simplifications with \eqref{appendix:annex:derivatives}}
\end{align*}
leading to
\begin{align}
	m &= 2\partial_{\hat{m}} \Psi_{\w} = \EE_{\xi} \[ \mZ_{\w^\star}\( \hat{m}  \hat{q}^{-1/2}  \xi , \hat{m} \hat{q}^{-1} \hat{m}  \) f_{\w^\star}\( \hat{m}  \hat{q}^{-1/2}  \xi , \hat{m} \hat{q}^{-1} \hat{m}  \) f_{\w} \(  \hat{q}^{1/2}\xi  , \hat{q} + \hat{Q} \)   \]
\end{align}
\paragraph{Equation over $Q$}
\begin{align*}
	\partial_{\hat{Q}} \Psi_w\(\hat{Q},\hat{m},\hat{q}  \) &= \partial_{\hat{Q}}  \EE_{\xi} \[ \mZ_{\w^\star}  \( \hat{m}  \hat{q}^{-1/2}  \xi , \hat{m} \hat{q}^{-1} \hat{m}  \) \log \mZ_{\w} \(  \hat{q}^{1/2}\xi  ,    \hat{Q} + \hat{q} \)   \] \\
	&= \EE_{\xi} \[ \mZ_{\w^\star}\( \hat{m}  \hat{q}^{-1/2}  \xi ,\hat{m} \hat{q}^{-1} \hat{m} \) \frac{1}{\mZ_w}  \partial_{\hat{Q}} \Lambda \partial_{\Lambda} \mZ_{\w} \(  \hat{q}^{1/2}\xi  ,    \hat{Q} + \hat{q} \)  \] \\
	&= - \frac{1}{2} \EE_{\xi} \[ \mZ_{\w^\star}\( \hat{m}  \hat{q}^{-1/2}  \xi ,\hat{m} \hat{q}^{-1} \hat{m} \)    \(\partial_\gamma f_w  + f_w^2  \)  \] \tag{with \eqref{appendix:annex:derivatives}}
\end{align*}
hence
\begin{align}
\begin{aligned}
	Q &= - 2  \partial_{\hat{Q}} \Psi_{\w} =  \EE_{\xi} \[ \mZ_{\w^\star}\( \hat{m}  \hat{q}^{-1/2}  \xi , \hat{m} \hat{q}^{-1} \hat{m}  \)  \partial_\gamma f_{\w} \(  \hat{q}^{1/2}\xi  , \hat{q} + \hat{Q} \) \] + q \,.
\end{aligned}
\end{align}

%% file: files/supplementary/applications.tex
\label{appendix:applications}

In this section, we provide details of the results presented in Sec.~\ref{sec:applications}. In particular as an illustration, we consider a Gaussian teacher ($\rho_{\w^\star}=1$) with a noiseless sign activation:
\begin{align}
	P_{\out^\star}(y|z) &= \delta\(y - \sign(z)\)\,, && P_{\w^\star}(w^\star) = \mN_{w^\star}\(0,\rho_{\w^\star}\)\,,
	\label{appendix:applications:teacher_channel_weights}
\end{align}
whose corresponding denoising functions are derived in eq.~\eqref{appendix:definitions:application:sign} and eq.~\eqref{appendix:definitions:application:gaussian}.\\

\begin{remark}
	Note that performances of ERM with $\rL_2$ regularization for a teacher with Gaussian weights $P_{\w^\star}(w) = \mN_w\(0,1\)$ or binary weights $P_{\w^\star}(w) =\frac{1}{2}\( \delta(w-1) + \delta(w+1)\)$, will be similar. Indeed free entropy terms $\Psi_{\w}$ eq.~\eqref{appendix:replicas:free_entropy_terms_non_bayes} for a Gaussian prior \eqref{appendix:annex:free_entropy:gaussian} and for binary weights \eqref{appendix:annex:free_entropy:binary} are equal in this setting, so do the set of fixed point equations.
\end{remark}

\subsection{Bayes-optimal estimation}
\label{appendix:applications:bayes}
Using expressions 
eq.~\eqref{appendix:definitions:application:sign} and eq.~\eqref{appendix:definitions:application:gaussian}, corresponding to the \emph{teacher} model eq.~\eqref{appendix:applications:teacher_channel_weights}, 
the prior equation eq.~\eqref{appendix:se_equations_generic:bayes} can be simplified while the channel one has no analytical expression. Hence the set of fixed point equations eqs.~\eqref{appendix:annex:free_entropy:gaussian:bayes} for the model eq.~\eqref{appendix:applications:teacher_channel_weights} read 
\begin{align}
	q_\bayes &=  \frac{\hat{q}_\bayes}{1+\hat{q}_\bayes}\,,&& \hat{q}_\bayes = \alpha \EE_{y, \xi } \[ \mZ_{\out^\star} \( y,  q_\bayes^{1/2}\xi, \rho_{\w^\star} - q_\bayes \)   f_{\out^\star} \( y,  q_\bayes^{1/2}\xi, \rho_{\w^\star} - q_\bayes \)^2  \]\,.
\label{appendix:application:se_equations_bayes}
\end{align}

\paragraph{Large $\alpha$ behaviour}
\label{appendix:applications:bayes_large_alpha}
Let us derive the large $\alpha$ behaviour of the Bayes-optimal generalization error eq.~\eqref{appendix:generalization_error:bayes} that depends only on the overlap $q_\bayes$ solution of eq.~\eqref{appendix:application:se_equations_bayes}.
$q_\bayes$ measures the correlation with the ground truth, so we expect that in the limit $\alpha \to \infty$, $q_\bayes \to 1$. Therefore, we need to extract the behaviour of $\hat{q}_\bayes$ in eq.~\eqref{appendix:application:se_equations_bayes}. Injecting expressions $\mZ_{\out^\star}$ and $f_{\out^\star}$ from eq.~\eqref{appendix:definitions:application:sign}, we obtain

\begin{align*}
	\hat{q}_\bayes &=  \alpha \EE_{y, \xi } \[ \mZ_{\out^\star} \( y,  q_\bayes^{1/2}\xi, 1 - q_\bayes \)   f_{\out^\star} \( y,  q_\bayes^{1/2}\xi, 1 - q_\bayes \)^2    \]\\
	&= 2 \alpha \int D\xi y^2 \frac{\mN_{\sqrt{q} \xi }(0,1-q_\bayes)^2}{\frac{1}{2} \(1 + \erf\(\frac{\sqrt{q_\bayes}\xi }{\sqrt{2(1-q_\bayes)}} \) \)} =\frac{2}{\pi} \frac{\alpha}{1-q_\bayes}   \int D\xi \frac{ e^{-\frac{q_\bayes \xi^2}{1-q_\bayes}} }{ \(1 + \erf\(\frac{\sqrt{q_\bayes}\xi }{\sqrt{2(1-q_\bayes)}} \) \)}\,,
\end{align*}
where the last integral can be computed in the limit $q_\bayes \to 1$:
\begin{align*}
	 &\int D\xi \frac{ e^{-\frac{q_\bayes \xi^2}{1-q_\bayes}} }{ \(1 + \erf\(\frac{\sqrt{q_\bayes}\xi }{\sqrt{2(1-q_\bayes)}} \) \)}  = \int \d\xi \frac{\frac{- e^{\frac{\xi ^2 (q_\bayes+1)}{2 (1-q_\bayes)}}}{\sqrt{2 \pi }}}{\(1 + \erf\(\frac{\sqrt{q_\bayes}\xi }{\sqrt{2(1-q_\bayes)}} \) \)}\\
	&\simeq \int \d\xi \frac{\frac{- e^{\frac{\xi ^2 }{1-q_\bayes}}}{\sqrt{2 \pi }}}{\(1 + \erf\(\frac{\xi }{\sqrt{2 (1-q_\bayes)}} \) \)} = \frac{\sqrt{1-q_\bayes}}{\sqrt{2\pi}} \int \d\eta \frac{e^{-\eta^2}}{1+\erf\(\frac{\eta}{\sqrt{2}}\)} = \frac{c_0}{\sqrt{2\pi}} \sqrt{1-q_\bayes}\,,
\end{align*}
with $c_0 \equiv \int \d\eta \frac{e^{-\eta^2}}{1+\erf\(\frac{\eta}{\sqrt{2}}\)} \simeq 2.83748$.
Finally, we obtain in the large $\alpha$ limit:
\begin{align*}
	\hat{q}_\bayes &= k \frac{\alpha}{\sqrt{1-q_\bayes}} \,, && q_\bayes = \frac{\hat{q}_\bayes}{1+\hat{q}_\bayes}\,,
\end{align*}
with $k\equiv \frac{2 c_0}{\pi \sqrt{2 \pi}} \simeq 0.720647 $. The above equations can be solved analytically and lead to:
\begin{align*}
	q_\bayes &= \frac{1}{2} \(\alpha  k \sqrt{\alpha ^2 k^2+4}-\alpha ^2 k^2\) \underset{\alpha \to \infty}{\simeq}  1-\frac{1}{\alpha ^2 k^2}\,, && \hat{q}_\bayes = k^2 \alpha^2\,,
\end{align*}
and therefore the Bayes-optimal asymptotic generalization error is given by
\begin{align}
	e_{\rm g}^{\rm bayes}(\alpha) = \frac{1}{\pi} \textrm{acos}\( \sqrt{q_\bayes} \)  \underset{\alpha \to \infty}{\simeq} \frac{1}{k \pi} \frac{1}{\alpha} \simeq \frac{0.4417}{\alpha} \,.
\end{align}

\subsection{Generalities on ERM with $\rL_2$ regularization}
\label{appendix:applications:L2_regularization}
Combining the teacher update for Gaussian weights eq.~\eqref{appendix:definitions:application:gaussian} with the update associated to the $\rL_2$ regularization eq.~\eqref{appendix:definitions:application:gaussian}, the free entropy term can be explicitly derived in \eqref{appendix:annex:free_entropy:gaussian}. Taking the corresponding derivatives, the fixed point equations for $m, q, \Sigma$ eq.~\eqref{appendix:se_equations_generic_not_simplified} are thus explicit and simply read
\begin{align}
\begin{aligned}
	\Sigma &= \frac{1}{\lambda + \hat{\Sigma}}\,,\hhspace && q =  \frac{\rho_{\w^\star}\hat{m}^2+\hat{q}}{(\lambda + \hat{\Sigma})^2}\,,\hhspace && m= \frac{\rho_{\w^\star}\hat{m}}{\lambda + \hat{\Sigma}}\,.
\end{aligned}
\label{appendix:application:se_equations:L2}
\end{align}
All the following examples have been performed with a $\rL_2$ regularization, so that the above equations \eqref{appendix:application:se_equations:L2} remain valid for the different losses considered in Sec.~\ref{sec:applications}. In the next subsections, we provide some details on the asymptotic performances of ERM with various losses with $\rL_2$ regularization and $\rho_{\w^\star} = 1$.

In general for a generic loss, the proximal eq.~\eqref{appendix:definitions:update_functions_map} has no analytical expression, just as the fixed point equations \eqref{appendix:se_equations_generic}.
The square loss is particular in the sense eqs.~\eqref{appendix:se_equations_generic} have a closed form solution. Also the Hinge loss has an analytical proximal. Apart from that, eqs.~\eqref{appendix:se_equations_generic} must be solved numerically. However it is useful to notice that the proximal can be easily found for a two times differentiable loss using eq.~\eqref{appendix:definitions:application:differentiable}. This is for example the case of the logistic loss.

\subsection{Ridge regression - Square loss with $\rL_2$ regularization}
\label{appendix:applications:ridge}
The prior equations over $m, q, \Sigma$ are already derived in eq.~\eqref{appendix:application:se_equations:L2} and remain valid. Combining eq.~\eqref{appendix:definitions:application:sign} for the considered sign channel with a potential additional Gaussian noise $\Delta^\star$ in \eqref{appendix:applications:teacher_channel_weights} and the square loss eq.~\eqref{appendix:definitions:application:square}, the channel fixed point equations for $\hat{q}, \hat{m}, \hat{\Sigma}$ eqs.~\eqref{appendix:se_equations_generic} lead to 
\begin{align}
\begin{aligned}
	\Sigma &= \frac{1}{\lambda + \hat{\Sigma}}\,,  \hspace{4cm} && \hat{\Sigma}  = \frac{\alpha}{\Sigma + 1}\,, \\
	q &=  \frac{\hat{m}^2+\hat{q}}{(\lambda + \hat{\Sigma})^2}\,, && \hat{q} = \alpha  \frac{( 1+q +\Delta^\star ) - 2 \sqrt{\frac{2 m^2}{\pi }} }{(\Sigma +1)^2\,,} \\
	m&= \frac{\hat{m}}{\lambda + \hat{\Sigma}}\,,
	&& \hat{m} = \frac{\alpha \sqrt{\frac{2}{\pi }}}{\Sigma + 1}\,.
\end{aligned}
\label{appendix:applications:se_equations:ridge}
\end{align}

\subsubsection{Pseudo-inverse estimator}
\label{appendix:applications:analytical_ridge_pseudo}
We analyze the fixed point equations eqs.~\eqref{appendix:applications:se_equations:ridge} for the \emph{pseudo-inverse} estimator, that is in the limit $\lambda \to 0$.

\paragraph{Solving $\Sigma$}
Combining the two first equations over $\Sigma$ and $\hat{\Sigma}$ in \eqref{appendix:applications:se_equations:ridge}, we obtain
\begin{align}
 	\Sigma &= \frac{\sqrt{(\alpha +\lambda -1)^2+4 \lambda }-\alpha -\lambda +1}{2 \lambda }  \underset{\lambda \to 0}{\simeq} \frac{1-\alpha + |\alpha -1| }{2 \lambda }+\frac{1}{2} \(\frac{\alpha +1}{|\alpha -1|}-1\)\,,
 	\label{appendix:applications:ridge:Sigma}
\end{align}
that exhibits two different behaviour depending if $\alpha < 1$ or $\alpha > 1$.

\paragraph{Regime $\alpha < 1$}
In this regime $\alpha < 1$, eq.~\eqref{appendix:applications:ridge:Sigma} becomes
\begin{align*}
 	\Sigma &= \frac{1 - \alpha}{\lambda} + \frac{\alpha }{1-\alpha}\,,
\end{align*}
that leads to the closed set of equations in the limit $\lambda \to 0$
\begin{align}
\begin{aligned}
	\Sigma &= \frac{(1-\alpha)^2 + \alpha \lambda}{\lambda\(1-\alpha\)} \underset{\lambda \to 0}{\simeq} \frac{1-\alpha}{\lambda}\,,   
	\hspace{2cm} && \hat{\Sigma}  = \frac{(1-\alpha) \alpha  \lambda }{(\alpha -1)^2+\lambda } \underset{\lambda \to 0}{\simeq} \frac{\lambda \alpha}{1-\alpha}\,, \\
	m&=\frac{\alpha(1-\alpha)}{\lambda + (1-\alpha)} \sqrt{\frac{2}{\pi}} \underset{\lambda \to 0}{\simeq} \alpha  \sqrt{\frac{2}{\pi}}\,,
	&& \hat{m} = \frac{\lambda \alpha \sqrt{\frac{2}{\pi }}}{\lambda + (1-\alpha)} \underset{\lambda \to 0}{\simeq} \frac{\lambda \alpha \sqrt{\frac{2}{\pi }}}{1-\alpha} \,,  \\
	q & \underset{\lambda \to 0}{\simeq} \frac{\alpha  ( \pi (1+\Delta^\star) -2 \alpha  )}{\pi  (1 -\alpha)}   \,,
	&& \hat{q} \underset{\lambda \to 0}{\simeq} \frac{\alpha  \lambda ^2 (2 (\alpha -2) \alpha +\pi  (\Delta^\star +1))}{\pi  (1- \alpha ) (1-\alpha +\lambda)^2} \,.
\end{aligned}
\end{align}
Hence we obtain for $\alpha < 1$:
\begin{align}
	m^{\rm pseudo} &=  \alpha  \sqrt{\frac{2}{\pi}}  &&	 q^{\rm pseudo}=  \frac{\alpha  ( \pi (1+\Delta^\star) -2 \alpha  )}{\pi  (1 -\alpha)} 
\end{align}
and the corresponding generalization error 
\begin{align}
	e_{\rm g}^{\rm pseudo} \(\alpha\) &= \frac{1}{\pi} \acos\( \sqrt{\frac{2\alpha(1-\alpha)}{\pi\(1+\Delta^\star\) - 2\alpha}} \)	 \textrm{ if } \alpha < 1\,.
\end{align}
Note in particular that $e_{\rm g}^{\rm pseudo} \(\alpha\) \underlim{\alpha}{1} 0.5$, meaning that the interpolation peak at $\alpha=1$ reaches the maximum generalization error.

\paragraph{Regime $\alpha > 1$}
Eq.~\eqref{appendix:applications:ridge:Sigma} becomes
\begin{align*}
 	\Sigma &= \frac{1}{2} \(\frac{\alpha +1}{\alpha -1}-1\) =  \frac{1}{2} \(\frac{\alpha +1}{\alpha -1}-1\) = \frac{1}{\alpha-1}\,.
\end{align*}
In the limit $\lambda \to 0$, the fixed point equations eqs.~\eqref{appendix:applications:se_equations:ridge} reduce to
\begin{align}
\begin{aligned}
	\Sigma + 1  &= \frac{\alpha}{\alpha - 1}\,,  \hspace{2cm} 
	&& \hat{\Sigma}  = \alpha - 1\,, \\
	q &=  \frac{\( \alpha - 1\)^2 \frac{2}{\pi } +\hat{q}}{\(\alpha - 1\)^2}\,,  \hspace{2cm} 
	&& \hat{q} = \frac{(\alpha-1)^2}{\alpha} \((1+q+\Delta^\star) - \frac{4}{\pi}  \)\,,  \\
	m&=  \sqrt{\frac{2}{\pi }}\,, \hspace{2cm} 
	&& \hat{m} =\( \alpha - 1\) \sqrt{\frac{2}{\pi }}\,.  \\
\end{aligned}
\label{main:se_equations:pseudoinverse}
\end{align}

In particular we obtain for $\alpha > 1$:
\begin{align}
	m^{\rm pseudo} &= \sqrt{\frac{2}{\pi}}\,, &&	 q^{\rm pseudo}= \frac{1}{\alpha - 1}\( 1 + \Delta^\star + \frac{2}{\pi}\(\alpha-2\) \)\,,
\end{align}
and the corresponding generalization error
\begin{align}
	e_{\rm g}^{\rm pseudo}\(\alpha\) &= \frac{1}{\pi} \acos\( \sqrt{\frac{\alpha-1}{\frac{\pi}{2}\(1 + \Delta^\star\) +\(\alpha-2\) }} \) \textrm{ if } \alpha > 1 \,.
\end{align}

\paragraph{Large $\alpha$ behaviour}
From this expression we easily obtain the large $\alpha$ behaviour of the pseudo-inverse estimator:
\begin{align*}
	e_{\rm g}^{\rm pseudo}(\alpha) &= \frac{1}{\pi} \acos\( \sqrt{\frac{\alpha-1}{\frac{\pi}{2}\(1 + \Delta^\star\) +\(\alpha-2\) }} \) = \frac{1}{\pi} \acos\( \( 1 + \frac{C}{\alpha-1} \)^{1/2} \) \underset{\alpha \to \infty}{\simeq}  \frac{c}{\sqrt{\alpha}}
\end{align*}
where $C=\frac{\pi}{2}\(1+\Delta^\star\)-1$ and $c= \frac{\sqrt{C}}{\pi}$. In particular for a noiseless teacher $\Delta^\star=0$, $c= \sqrt{\frac{\pi-2}{2\pi^2}} \simeq 0.240487$, leading to 
\begin{align}
	e_{\rm g}^{\rm pseudo}(\alpha)  \underset{\alpha \to \infty}{\simeq}  \frac{0.2405}{\sqrt{\alpha}}\,.
\end{align}

\subsubsection{Ridge at finite $\lambda$}
\label{appendix:applications:analytical_ridge_map}
Let us now consider the set of fixed point equation eq.~\eqref{appendix:applications:se_equations:ridge} for finite $\lambda \ne 0 $.  Defining
\begin{align*}
	t_0 & \equiv \sqrt{(\alpha +\lambda -1)^2+4 \lambda }\\
	t_1 &\equiv \(t_0+\alpha +\lambda +1\)^{-1}\\
	t_2 &\equiv \sqrt{2 (\alpha +1) \lambda +(\alpha -1)^2+\lambda ^2}\\
	t_3 &\equiv \(t_2+\alpha +\lambda +1\)^{-1}	 \\
	t_4 & \equiv \sqrt{\alpha ^2+2 \alpha  (\lambda -1)+(\lambda +1)^2}\,,
\end{align*}
the equations can be in fact fully solved analytically and read
\begin{align*}
	\Sigma &=  \frac{1}{2} \frac{t_0-\alpha -\lambda +1}{\lambda }\spacecase 
	\hat{\Sigma} &= \frac{1}{2} \big(t_0+\alpha -\lambda -1\big) \spacecase
	q &= \frac{2 \alpha  \(-8 \alpha ^2 t_1 + 2 \alpha +\pi  \Delta^\star +\pi \)}{\pi  \(\alpha ^2+\alpha  \(t_2+2 \lambda -2\)+(\lambda +1) \(t_2+\lambda +1\)\)}\,, \spacecase \spacecase
	\hat{q} &= \big( 4 \alpha  \lambda ^2 \big(\pi  (\Delta^\star +1) \big( t_4+(\alpha +\lambda ) \big(t_2+\alpha +\lambda \big)+2 \lambda +1\big)\\
	& \hspace{1cm} -8 \alpha t_3  \big(t_4+(\alpha +\lambda ) \big(\sqrt{2 (\alpha +1) \lambda 
	+(\alpha -1)^2+\lambda ^2}+\alpha +\lambda \big)+2 \lambda \big) -8 \alpha t_3+4 \alpha ^2\big) \big)\,,  \spacecase
	m&= \frac{2 \sqrt{\frac{2}{\pi }} \alpha }{t_2+\alpha +\lambda +1}\,, \spacecase
	\hat{m} &= \frac{2 \sqrt{\frac{2}{\pi }} \alpha  \lambda }{t_0-\alpha +\lambda +1}\,.
\end{align*}

\paragraph{Generalization error behaviour at large $\alpha$}
Expanding the ratio $\frac{m}{\sqrt{q}}$ in the large $\alpha$ limit, we obtain
\begin{align*}
	\frac{m}{\sqrt{q}} \simeq  1 - \frac{C}{2 \alpha } \textrm{ with } C = \frac{\pi}{2} \(1+ \Delta^\star\) - 1
\end{align*}
leading to
\begin{align}
	e_{\rm g}^{\rm ridge, \lambda}\(\alpha\) &= \frac{1}{\pi} \acos\( \frac{m}{\sqrt{q}} \) \underset{\alpha \to \infty}{\simeq}  \frac{c}{\sqrt{\alpha}}  \textrm{ with } c= \frac{\sqrt{C}}{\pi}\,.
\end{align}
Thus, the asymptotic generalization error for ridge regression with any regularization strength $\lambda \geq 0$ decrease as $\frac{0.2405}{\sqrt{\alpha}}$, similarly to the pseudo-inverse result. 
 
\paragraph{Optimal regularization}
The optimal value $\lambda^{\rm opt}(\alpha)$, introduced in Sec.~\ref{sec:applications}, which minimizes the generalization error at a given $\alpha$ can be found taking the derivative of $\frac{m}{\sqrt{q}}$ and is written as the root of the following functional
\begin{align*}
	F[\alpha, \lambda, \Delta^\star]&= \partial_\lambda \( \frac{m}{\sqrt{q}}\) =  \frac{a_1 a_2}{ a_3 a_4^2}\,,\\
	\text{with}&\\
	a_1 &= - 4 \alpha  \sqrt{\frac{a_4}{\alpha ^2 + \alpha \(t_2+2 \lambda -2\) +(\lambda +1) \(t_2+\lambda +1\)}} \,,\\ 
	a_2 &= 2 \(\alpha ^2 t_3 + \alpha  \(2 \lambda  t_3 + \(t_2+2\) t_3-1 \)+(\lambda +1) \(t_2+\lambda +1\) t_3 \) -\pi  (1+\Delta^\star) \,, \\
	a_3 &= \frac{t_0}{t_1}\,,\\
	a_4 &= \alpha\( 2 -8  t_1\)  + \pi\( 1 + \Delta^\star\)  \,.
\end{align*}
Unfortunately, this functional cannot be analyzed analytically. Instead we plot its value for a wide range of $\alpha$ as a function of $\lambda$ (for $\Delta^\star=0$) and we observe in particular that there exists a unique value $\lambda^{\rm opt}\simeq 0.570796$ as illustrated in Fig.~\ref{fig:ridge_opt_lambda} (\textbf{left}) that is independent of $\alpha$. As an illustration, we show the generalization error of ridge regression with the optimal regularization $\lambda^{\rm opt}=0.5708$ compared to the Bayes-optimal performances in Fig.~\ref{fig:ridge_opt_lambda} (\textbf{right}).
\begin{figure}
	\centering
	\includegraphics[scale=0.5]{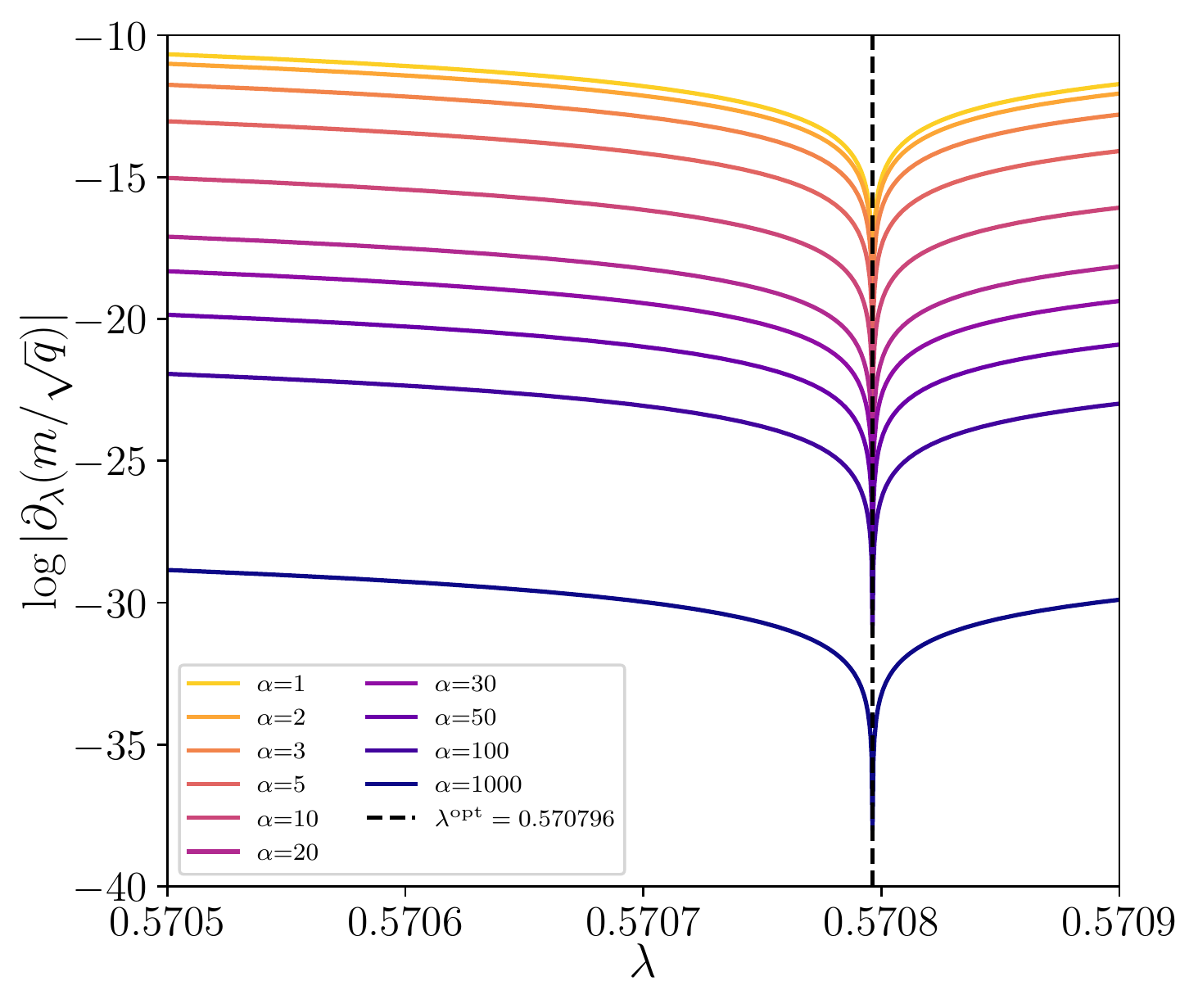}
	\hfill
	\includegraphics[scale=0.5]{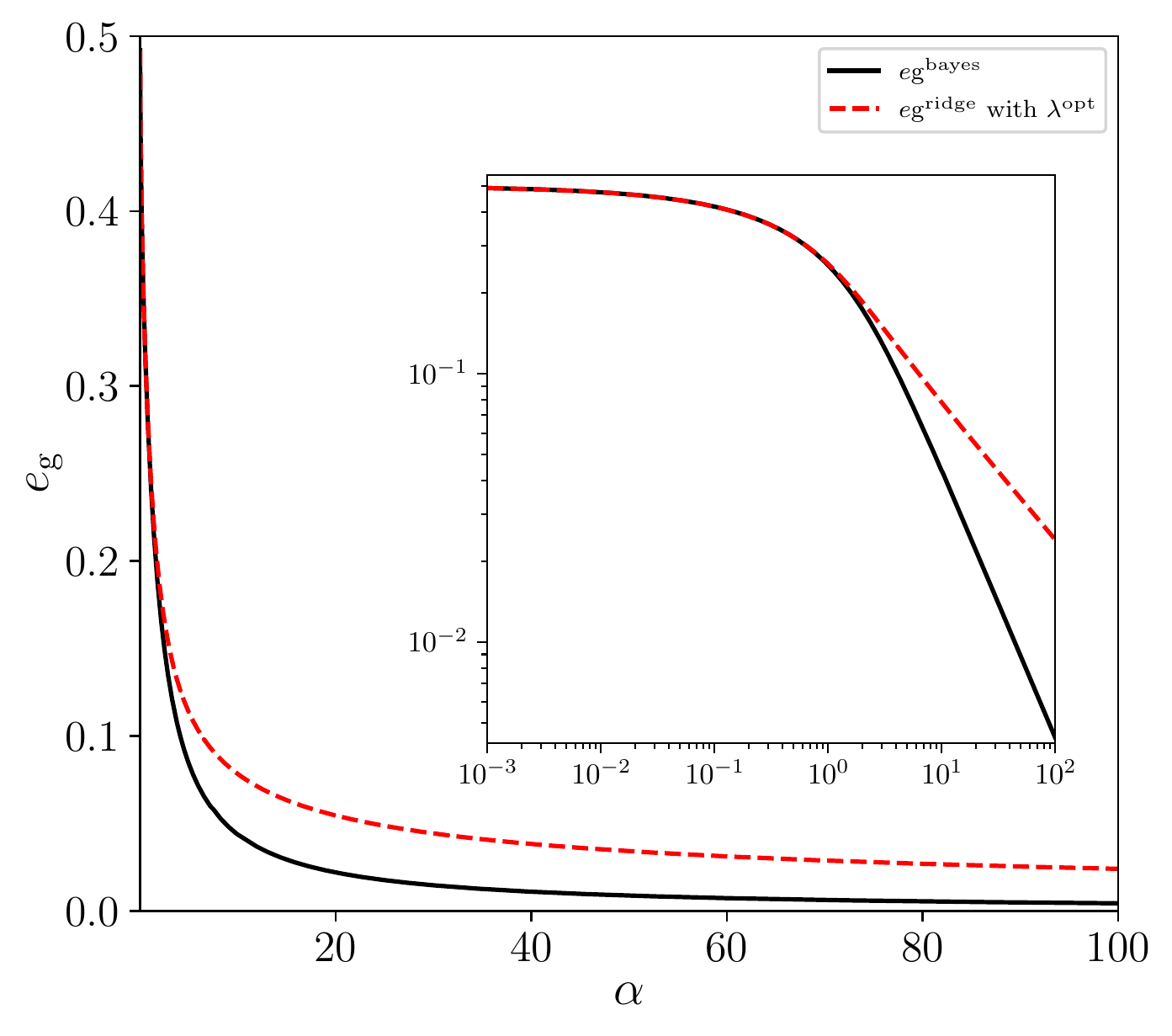}
	\caption{(\textbf{Left})  Absolute value of the derivative of $m/\sqrt{q}$ with respect to $\lambda$ plotted in a logarithmic scale. $\lambda^{\rm opt}$ is reached at the root of the functional $F[\alpha, \lambda]$ that corresponds to the divergence in the logarithmic scale. Plotted for a wide range of $\alpha$, the optimal value is clearly constant and independent of $\alpha$. Its value is approximately $\lambda^{\rm opt}\simeq 0.570796$. (\textbf{Right}) Bayes-optimal (black) vs ridge regression  (dashed red) generalization errors with optimal $\rL_2$ regularization $\lambda^{\rm opt}\simeq 0.570796$.}	
	\label{fig:ridge_opt_lambda}
\end{figure}

\newpage
\subsection{Hinge regression / SVM - Hinge loss with $\rL_2$ regularization}
\label{appendix:applications:hinge}
The hinge loss $l^{\rm hinge}(y,z) = \max\(0, 1- yz\)$ is linear by part and is therefore another simple example of analytical loss to analyze.
In particular its proximal map can computed in eq.~\eqref{appendix:definitions:application:hinge} and the corresponding denoising functions read:
\begin{align}
\begin{aligned}
	f_{\out} \( y,  q^{1/2}\xi, \Sigma \) &=\begin{cases}
		 y  \textrm{ if } \xi y < \frac{1 -\Sigma}{\sqrt{q}} \spacecase 
		\frac{y  -\sqrt{q} \xi }{\Sigma}   \textrm{ if }  \frac{1 -\Sigma}{\sqrt{q}} < \xi y < \frac{1}{\sqrt{q}} \spacecase 
		0 \textrm{ otherwise }
		\end{cases}\,,	\spacecase
	\partial_\omega f_{\out} \( y,  q^{1/2}\xi, \Sigma \)&= \begin{cases}
		-\frac{1}{\Sigma}   \textrm{ if }  \frac{1 -\Sigma}{\sqrt{q}} < \xi y < \frac{1}{\sqrt{q}} \spacecase 
		0 \textrm{ otherwise } 
		\end{cases}\,.
\end{aligned}
\label{appendix:application:hinge:denoising}
\end{align}

The fixed point equations eq.~\eqref{appendix:se_equations_generic} have unfortunately no closed form and need to be solved numerically.

\subsubsection{Max-margin estimator}
As proven in \cite{Rosset04} both the hinge and logistic estimators converge to the \emph{max-margin} solution in the limit $\lambda\to0$ as soon as the data are linearly separable. We will start with the fixed point equations for hinge, whose denoising functions \eqref{appendix:application:hinge:denoising} are analytical. Taking the $\lambda \to 0$ limit is non-trivial and we need therefore to introduce some rescaled variables to obtain a closed set of equations. Numerical evidences at finite $\alpha$ show that we shall use the following rescaled variables: 
\begin{align*}
	\hat{m} &= \Theta\(\lambda\),  &&\hat{q}= \Theta\(\lambda^2\), &&\hat{\Sigma} = \Theta\(\lambda\), && m=\Theta(1), && q= \Theta(1), &&\Sigma = \Theta\(\lambda^{—1}\)\,.
\end{align*}
The fixed point equations eq.~\eqref{appendix:se_equations_generic} simplify and become
\begin{align}
	\begin{aligned}
	m &= \frac{\hat{m}}{1+  \hat{\Sigma}}\,, && q = \frac{\hat{m}^2+\hat{q}}{(1 + \hat{\Sigma})^2}\,, &&\Sigma = \frac{1}{1 + \hat{\Sigma}}\,, \\
	\hat{m} &= \frac{2\alpha}{\Sigma} \mI_{\hat{m}}(q,\eta)\,, && \hat{q}= \frac{2\alpha}{\Sigma^2} \mI_{\hat{q}}(q,\eta) \,, && \hat{\Sigma} = \frac{2\alpha}{\Sigma} \mI_{\hat{\Sigma}}(q,\eta)\,,
	\end{aligned}
	\label{appendix:applications:hinge:max_margin:state_evollution}
\end{align}

with 
\begin{align}
	\begin{aligned}	
		\mI_{\hat{m}}(q,\eta) &\equiv \int_{-\infty}^{\frac{1}{\sqrt{q}}} \d \xi \mN_{\xi}(0,1) \mN_{\xi}\(0,\frac{1-\eta}{\sqrt{\eta}}\)  \(1-\sqrt{q}\xi\)\,, \\ 
		&=\frac{\sqrt{2 \pi } \(\text{erf}\(\frac{1}{\sqrt{2} \sqrt{q(1-\eta)}}\)+1\)+2 e^{-\frac{1}{2 q(1-\eta )}} \sqrt{q(1-\eta )}}{4 \pi }\\
		\mI_{\hat{q}}(q,\eta) &\equiv \int_{-\infty}^{\frac{1}{\sqrt{q}}} \d \xi \mN_{\xi}(0,1)  \frac{1}{2} \(1 + \erf\(\frac{\sqrt{\eta}\xi}{\sqrt{2(1-\eta)}} \) \) \(1-\sqrt{q}\xi\)^2 \,, \\
		\mI_{\hat{\Sigma}}(q,\eta) &\equiv \int_{-\infty}^{\frac{1}{\sqrt{q}}} \d \xi \mN_{\xi}(0,1)  \frac{1}{2} \(1 + \erf\(\frac{\sqrt{\eta}\xi}{\sqrt{2(1-\eta)}} \) \)\,.
	\end{aligned}
\end{align}

\paragraph{Large $\alpha$ expansion}
Numerically at large $\alpha$ (and $\lambda \to 0$), we obtain the following scalings
\begin{align}
\begin{aligned}
	q &= \Theta(\alpha^2)\,, && m= \Theta(\alpha)\,, && \Sigma= \Theta(1)\,, && \hat{q} = \Theta(1)\,, && \hat{m}= \Theta(\alpha)\,, && \hat{\Sigma}= \Theta(1)\,.
\end{aligned}
\end{align}
Therefore, in order to close the equations, we introduce new variables $(c_q, c_\eta)$ such that
\begin{align}
	q &\underset{\alpha \to \infty}{=} c_q \alpha^2\,, && \eta = 1 - \frac{c_\eta}{\alpha^2}\,.
	\label{appendix:applications:hinge:max_margin:scaling_assumption}
\end{align}
In this limit, we can extract the large $\alpha$ behaviours of integrals $\mI_{\hat{m}}, \mI_{\hat{q}}, \mI_{\hat{\Sigma}}$:
\begin{align}
	\begin{aligned}	
		\mI_{\hat{m}}(q,\eta) &= \mI_{\hat{m}}^\infty(c_q,c_\eta)\,, && \mI_{\hat{q}}(q,\eta) = \frac{\mI_{\hat{q}}^\infty(c_q,c_\eta)}{\alpha}\,, &&\mI_{\hat{\Sigma}}(q,\eta) = \frac{\mI_{\hat{\Sigma}}^\infty(c_q,c_\eta)}{\alpha}\,,\\
	\end{aligned}
\end{align}
where $\mI_{\hat{m}}^\infty, \mI_{\hat{q}}^\infty, \mI_{\hat{\Sigma}}^\infty$ are $\Theta(1)$ and read
\begin{align}
	\begin{aligned}
		\mI_{\hat{m}}^\infty(c_q,c_\eta)&\equiv \frac{\sqrt{2 \pi } \(\text{erf}\(\frac{1}{\sqrt{2} \sqrt{c_\eta c_q}}\)+1\)+2 e^{-\frac{1}{2 c_\eta c_q}} \sqrt{c_\eta c_q}}{4 \pi }\,, \\
		\mI_{\hat{q}}^\infty(c_q,c_\eta)&\equiv  \frac{e^{-\frac{1}{2 c_\eta c_q}} \(\sqrt{2 \pi } (3 c_\eta c_q+1) e^{\frac{1}{2 c_\eta c_q}} \(\text{erf}\(\frac{1}{\sqrt{2} \sqrt{c_\eta c_q}}\)+1\)+4 (c_\eta c_q)^{3/2}+2 \sqrt{c_\eta c_q}\)}{12 \pi   \sqrt{c_q}}\,, \\
		\mI_{\hat{\Sigma}}^\infty(c_q,c_\eta)&\equiv  \frac{\sqrt{2 \pi } \(\text{erf}\(\frac{1}{\sqrt{2} \sqrt{c_\eta c_q}}\)+1\)+2 e^{-\frac{1}{2 c_\eta c_q}} \sqrt{c_\eta c_q}}{4 \pi   \sqrt{c_q}}\,.
	\end{aligned}
\end{align}
Hence the set of fixed-point equations eq.~\eqref{appendix:applications:hinge:max_margin:state_evollution} simplifies to:
\begin{align}
	\begin{aligned}
		\hat{\Sigma} &= \frac{2\mI_{\hat{\Sigma}}^\infty(c_q,c_\eta)}{1 - 2\mI_{\hat{\Sigma}}^\infty(c_q,c_\eta)}\,, && \Sigma = 1 - 2\mI_{\hat{\Sigma}}^\infty(c_q,c_\eta)\\
		\hat{m} &= \frac{2 \alpha \mI_{\hat{m}}^\infty(c_q,c_\eta)}{ 1 - 2\mI_{\hat{\Sigma}}^\infty(c_q,c_\eta)} \,, && m =   2 \alpha \mI_{\hat{m}}^\infty(c_q,c_\eta)\\
		\hat{q} &= \frac{2 \mI_{\hat{q}}^\infty(c_q,c_\eta)}{\(1 - 2\mI_{\hat{\Sigma}}^\infty(c_q,c_\eta)\)^2}\,, && q = 4 \alpha^2 \(\mI_{\hat{m}}^\infty(c_q,c_\eta)\)^2 + 2 \mI_{\hat{q}}^\infty(c_q,c_\eta)\,,
	\end{aligned}	
\end{align}
which can be closed by rewriting the equations eqs.~\eqref{appendix:applications:hinge:max_margin:scaling_assumption}:
\begin{align}
	\begin{aligned}
	\eta &= \frac{m^2}{q} \equiv 1 - \frac{c_\eta}{\alpha^2} = 1 - \frac{\mI_{\hat{q}}^\infty(c_q,c_\eta)}{2 \(\mI_{\hat{m}}^\infty(c_q,c_\eta)\)^2}\frac{1}{\alpha^2}\,, \spacecase
	q &= c_q \alpha^2 \simeq 4 \alpha^2 \(\mI_{\hat{m}}^\infty(c_q,c_\eta)\)^2\,.
	\end{aligned}
\end{align}
Equivalently $(c_q^\star, c_\eta^\star)$ is the root of the set of non-linear fixed point equations $(F_\eta(c_q, c_\eta), F_q(c_q, c_\eta))$:
\begin{align}
	F_\eta(c_q, c_\eta) &\equiv  \frac{\mI_{\hat{q}}^\infty(c_q,c_\eta)}{2 \(\mI_{\hat{m}}^\infty(c_q,c_\eta)\)^2} - c_\eta \,, && F_q(c_q, c_\eta) \equiv 4 \(\mI_{\hat{m}}^\infty(c_q,c_\eta)\)^2  -c_q \,,
\end{align}
that cannot be solved analytically. However a unique numerical solution is found and lead to $(c_q^\star, c_\eta^\star) = (0.9911, 2.4722)$. Therefore the generalization error of the max-margin estimator in the large $\alpha$ regime is given by
\begin{align}
	e_{\rm g}^{\rm max-margin} (\alpha) = \frac{1}{\pi} \arccos \( \frac{m}{\sqrt{q}}\) \underset{\alpha \to \infty}{\simeq}  \frac{1}{\pi} \arccos \( 1 - \frac{c_\eta^\star}{\alpha^2} \) \underset{\alpha \to \infty}{\simeq}  \frac{K}{\alpha}\,,
\end{align}
with $K=\frac{\sqrt{c_\eta^\star}}{\pi} \simeq 0.5005$, leading to
\begin{align}
	e_{\rm g}^{\rm max-margin} (\alpha) \underset{\alpha \to \infty}{\simeq} \frac{0.5005}{\alpha}\,.
\end{align}

\subsection{Logistic regression}
\label{appendix:applications:logistic}
The logistic loss is a combination of the cross entropy loss  $l (y, z) = -y \log(\sigma(z)) -(1-y) \log(1-\sigma(z)) $ with as sigmoid activation function $\sigma$, that simplifies for binary labels $y\pm1$ to $l^{\rm logistic} (y,z) = \log(1+\exp(-y z))$ with the two first derivatives given by
\begin{align*}
	\partial_z l^{\rm logistic}(y,z) &= -\frac{y}{e^{z y}+1}\,, 
	&& \partial^2_{z} l^{\rm logistic}(y,z) = \frac{y^2}{2 (1 +  \textrm{cosh}\(z y\))} = \frac{y^2}{4 \textrm{cosh}\(\frac{yz}{2}\)} \,.
\end{align*}
Its proximal is not analytical, but it can be written as the solution of the implicit equation \eqref{appendix:definitions:application:proximal_differentiable}
providing the corresponding denoising functions \eqref{appendix:definitions:application:differentiable}. Solving the fixed point equations \eqref{appendix:se_equations_generic}, we obtain performances that approach closely the Bayes-optimal baseline as illustrated in Fig.~\ref{fig:appendix:gen_error_logistic} (\textbf{left}).
\begin{figure}[htb!]
	\centering
	\includegraphics[scale=0.5]{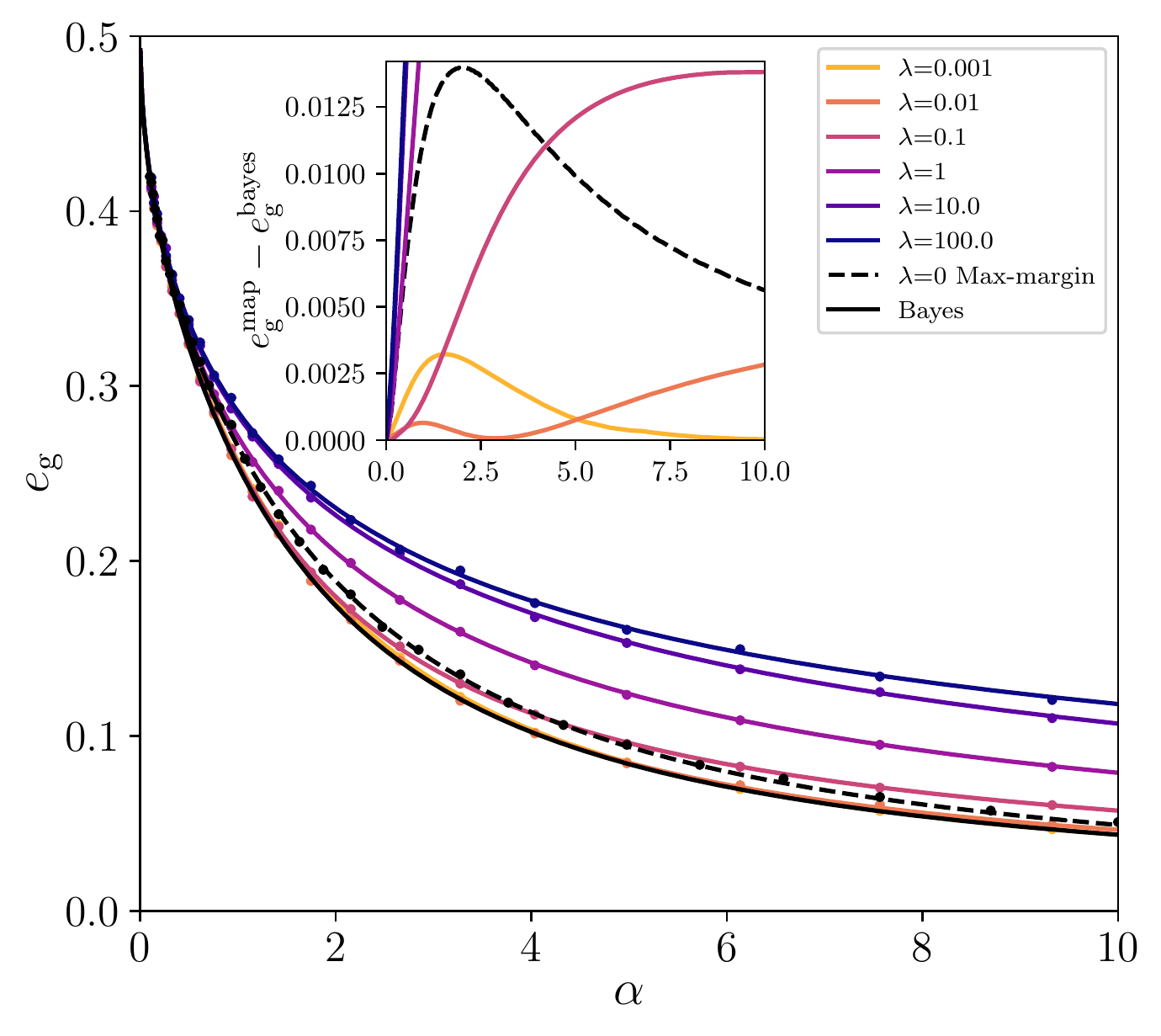}
	\hfil
	\includegraphics[scale=0.5]{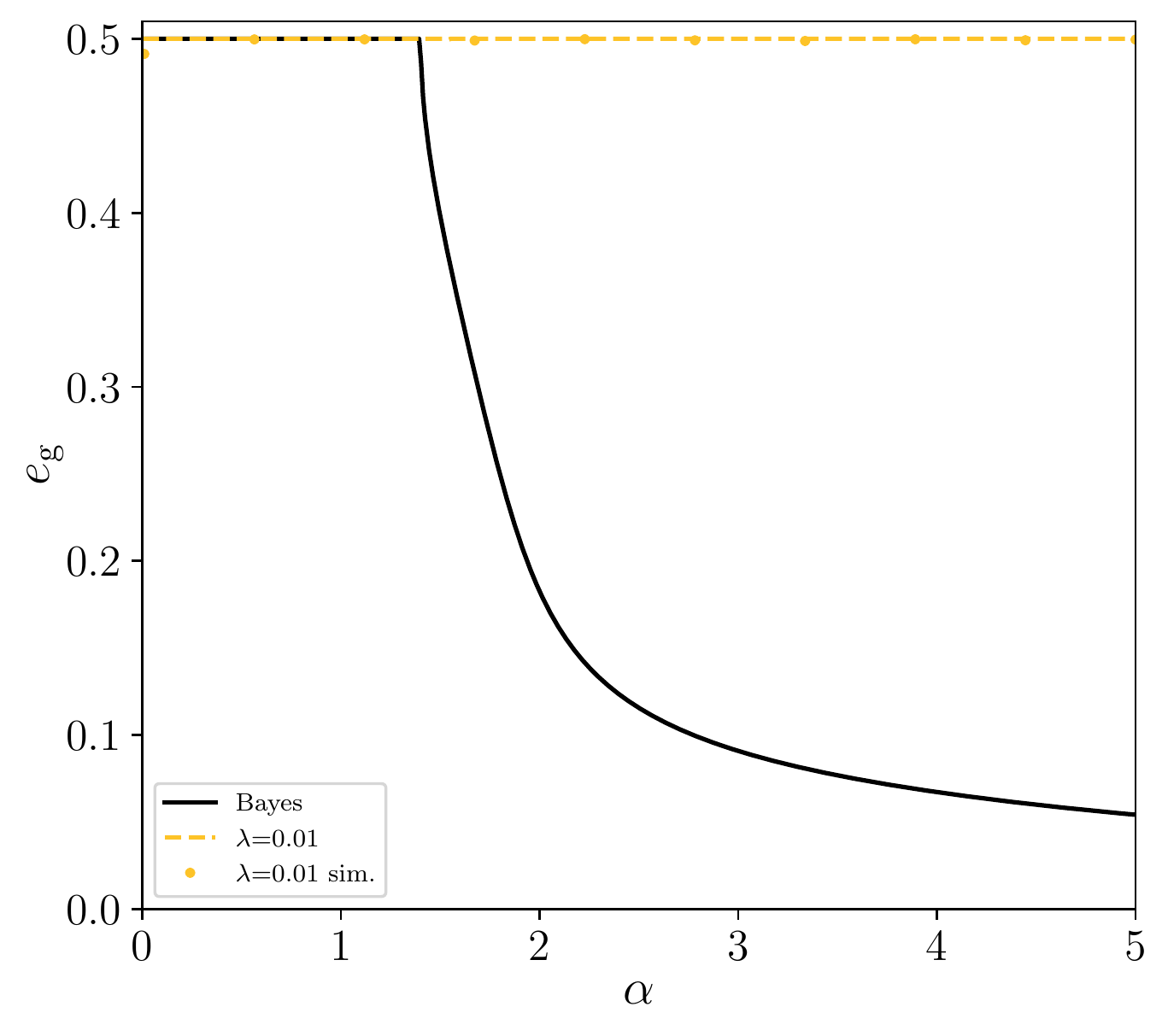}
	\caption{	
	 (\textbf{Left}) Logistic regression - Generalization error as a function of $\alpha$ for different regularizations strength $\lambda$. Decreasing $\lambda$, the generalization error approaches very closely the Bayes-optimal error (black line). The difference with the Bayes error is shown as an inset. Logistic flirts with Bayes error  but never achieves it exactly. The asymptotic behaviour is compared to numerical logistic regression with $\ndim=10^{3}$ and averaged over $n_s=20$ samples, performed with the default method \textit{LogisticRegression} of the \texttt{scikit-learn} package \cite{scikit-learn}.	 
	 (\textbf{Right}) Rectangle door teacher with $\kappa=0.6745$ - Bayes-optimal generalization error (black) compared to asymptotic generalization performances of $\rL_2$ logistic regression (dashed yellow line) and numerical ERM (crosses).
	 }
	\label{fig:appendix:gen_error_logistic}	
\end{figure}

\subsection{Logistic with non-linearly separable data - A rectangle door teacher}
\label{appendix:applications:rectangle}

The analysis of ERM for the linearly separable dataset generated by \eqref{appendix:applications:teacher_channel_weights} reveals that logistic regression with $\rL_2$ regularization was able to approach very closely Bayes-optimal error. Therefore it seems us very interesting to investigate if logistic regression could perform as well on a more complicated non-linearly separable dataset obtained by a \emph{rectangle door} channel
\begin{align}
\vec{y} = \sign\(\left|\frac{1}{\sqrt{\ndim}} \mat{X}\vec{w}^\star\right| -\kappa\).
\label{appendix:teacher_rectangle}	
\end{align}
This channel has been already considered in \cite{Barbier5451} and we fix the width of the door to $\kappa= 0.6745$ to obtain labels $\pm 1$ with probability $0.5$. 
We then compare the ERM performances of logistic regression with $\rL_2$ regularization to the Bayes-optimal performances given by \eqref{appendix:application:se_equations_bayes} with denoising functions derived in eq.~\eqref{appendix:definitions:application:door}. We show in Fig.~\ref{fig:appendix:gen_error_logistic} (\textbf{right}) the comparison only for an arbitrary hyper-parameter $\lambda=10^{-2}$, as results are similar for any regularization.
As we might expect, the logistic regression is not able to reach the Bayes-optimal generalization error. Both Bayes-optimal and ERM performances are stuck in the symmetric fixed point $m=0$ up to $\alpha_{\rm it} \simeq 1.393$. Above this threshold it becomes unstable and Bayes error decreases to zero in the $\alpha \to 0$ limit, while the logistic regression with arbitrary $\lambda$ remains stuck to its maximal generalization error, meaning that in this non-linearly separable case, the logistic regression largely underperforms Bayes-optimal performances.

%% file: files/supplementary/optimality.tex
\label{appendix:optimal_loss_reg}
In this section, we propose a derivation inspired by 
\cite{Kinouchi1996, Opper1991, Advani2016, Bean2013, Donoho2016, gribonval2011should,NIPS2013_4868,gribonval2018characterization,gribonval2019bayesian, Advani2016b} of the fine-tuned loss and regularizer \eqref{main:opt_loss_reg} discussed in Sec.~\ref{sec:optimality}.  
We assume that the dataset is generated by a teacher \eqref{appendix:teacher_model} such that $\mZ_{\out^\star}(., \omega,.)$ and $\mZ_{\w^\star}(\gamma,.)$ are respectively log-concave in $\omega$ and $\gamma$.
The derivation is based on the GAMP algorithm introduced in \cite{Rangan2010} for the model eq.~\eqref{main:teacher_model}, that we start by recalling.

\subsection{Generalized Approximate Message Passing (GAMP) algorithm}
\label{appendix:optimal_loss_reg:amp}
The GAMP algorithm can be written as the following set of iterative equations that depend on the update functions \eqref{appendix:update_functions_generic}:
\begin{equation}
\begin{cases}
	\hat{\bw}^{t+1} = f_{\w}({\bgamma}^t , \Lambda^t)   \spacecase
	\hat{\bc}_{\w}^{t+1} = \partial_\gamma f_{\w}({\bgamma}^t , \Lambda^t)\spacecase
	\vec{f}_{\out}^{t} = f_{\rm out}\(y, {\bomega^{t}} ,V^{t} \)
	\end{cases}
	\andcase
	\begin{cases}
	\Lambda_i^t &= - \frac{1}{\ndim} \sum_{\mu =1}^\nsamples \mat{X}_{\mu
      i}^2  \partial_\omega f_{\out,\mu}^t \spacecase
	{\bgamma}_i^t &=  \frac{1}{\sqrt{\ndim}} \sum_{\mu =1}^\nsamples \mat{X}_{\mu
      i} f_{\out,\mu}^t  + \Lambda_i^t \hat{w}_{i}^t  \spacecase
	V_{\mu}^t &= \frac{1}{\ndim} \sum_{i=1}^\ndim \mat{X}_{\mu i}^2 \hat{c}_{w,i}^t \spacecase
	\omega_{\mu}^t &= \frac{1}{\sqrt{\ndim}} \sum_{i=1}^\ndim \mat{X}_{\mu
      i} \hat{w}_{i}^t -  V_{\mu}^t f_{\out,\mu}^{t-1}
\end{cases}\,.
	\label{appendix:amp:glm_}
\end{equation}
It has been proven in \cite{Barbier2017b} that the GAMP algorithm with Bayes-optimal update functions $f_{\w}=f_{\w^\star}$ and $f_{\out} = f_{\out^\star}$ \eqref{appendix:definitions:update_functions_bayes} converges to the Bayes-optimal performances in the large size limit.
Yet the GAMP denoising functions are generic and can be chosen as will depending on the statistical estimation method. In particular we may choose the denoising functions for Bayes-optimal estimation \eqref{appendix:definitions:update_functions_bayes} or the ones corresponding to ERM estimation \eqref{appendix:definitions:update_functions_map}
\begin{align}
\begin{aligned}
	& f_{\w}^{\rm bayes}(\gamma ,\Lambda) = \partial_\gamma \log\(\mZ_{\w^\star}\) \,, \spacecase
	& f_{\out}^{\rm bayes} (y,\omega,V) = \partial_\omega \log \( \mZ_{\out^\star} \) \,,\spacecase
	& f_{\rm w}^{{\rm erm}, r}(\gamma, \Lambda) =  \Lambda^{-1}\gamma - \Lambda^{-1} \partial_{\Lambda^{-1}\gamma}\mM_{\Lambda^{-1}}\[ r(.) \] (\Lambda^{-1}\gamma)\,, \spacecase
	& f_{\out}^{{\rm erm}, l} (y, \omega, V) =  - \partial_{\omega} \mM_{V}[l(y,.)](\omega) \,,
\end{aligned}
\end{align}
whose corresponding GAMP algorithms \eqref{appendix:amp:glm_} will achieve potentially different fixed points and thus different performances. 
As it is proven that GAMP with Bayes-optimal updates lead to the optimal generalization error, so that ERM matches the same performances it is sufficient to enforce that at each time step $t$ the Bayes-optimal and ERM denoising functions are equal $f^{\rm bayes}=f^{\rm erm}$. Enforcing these two constraints will lead to the expressions for the optimal loss $l^{\rm opt}$ and regularizer $r^{\rm opt}$, so that ERM matches Bayes-optimal performances.

\subsection{Matching Bayes-optimal and ERM performances}
Imposing the equality on the channel updates we obtain
 \begin{align*}
 	f_{\out}^{\rm bayes}\(y, \omega, V\) &= f_{\out}^{{\rm erm}, l}\(y, \omega, V\)
 	\Leftrightarrow
\partial_\omega \log \( \mZ_{\out^\star} \)\(y, \omega, V\) = - \partial_\omega \mM_V\[l^{\rm opt}\(y,.\)\] (\omega)\,.
\end{align*}
Integrating, leaving aside the constant that will not influence the final result, and taking the Moreau-Yosida regularization on both sides, we obtain:
\begin{align*}
	\mM_V\[\log  \mZ_{\out^\star} \(y, ., V\)\]\(\omega\) = \mM_V\[- \mM_V\[l^{\rm opt}\(y,.\)\] (\omega)\] = - l^{\rm opt}\(y,\omega\)\,,
\end{align*}
where we invert the Moreau-Yosida regularization in the last equality that is valid as long as $\mZ_{\out^\star}(y,\omega,V)$ is assumed to be log-concave in $\omega$, (see \cite{Advani2016} for a derivation).
We finally obtain
\begin{align}
		l^{\rm opt}\(y,z\) &= - \mM_V\[\log \( \mZ_{\out^\star} \)\(y, ., V\)\]\(z\) = - \min_\omega \( \frac{(z-\omega)^2}{2 V} + \log \mZ_{\out^\star} \(y,\omega,V\) \) \,.
	\label{appendix:optimality:loss}
\end{align}
Let us perform the same computation for the prior updates. First we introduce a rescaled denoising distribution:
\begin{align}
\begin{aligned}
		&\td{Q}_{\w^\star}(w; \gamma,\Lambda) \equiv \displaystyle \frac{1}{\td{\mZ}_{\w^\star} (\gamma,\Lambda)} P_{\w^\star}(w) e^{ - \frac{1}{2} \Lambda\(w - \Lambda^{-1} \gamma  \)^2  }\,, \\ 
		&\log\(\td{\mZ}_{\w^\star} (\gamma,\Lambda)\) = \log\(\mZ_{\w^\star} (\gamma,\Lambda)\) - \frac{1}{2}\Lambda^{-1}\gamma^2 \,,
\end{aligned}
\end{align}
so that the the prior updates read
\begin{align}
\begin{aligned}
 	f_{\w}^{\rm bayes}\( \gamma, \Lambda \) &= \partial_\gamma \log\(\mZ_{\w^\star}\) = \Lambda^{-1}\gamma + \Lambda^{-1} \partial_{\Lambda^{-1}\gamma} \log\(\td{\mZ}_{\w^\star}\)\,, \spacecase
 	f_{\w}^{{\rm erm},r}\( \gamma, \Lambda \) &= \mP_{\Lambda^{-1}}\[ r \](\Lambda^{-1}\gamma) = \Lambda^{-1}\gamma - \Lambda^{-1} \partial_{\Lambda^{-1}\gamma}\mM_{\Lambda^{-1}}\[ r \] (\Lambda^{-1}\gamma) \,.
\end{aligned}
\end{align}
Imposing the equivalence of the Bayes-optimal and ERM prior update, 
\begin{align}
		f_{\w}^{\rm bayes}\( \gamma, \Lambda \) =  f_{\w}^{{\rm erm}, r}\( \gamma, \Lambda \) \Leftrightarrow  \partial_{\Lambda^{-1}\gamma} \log\(\td{\mZ}_{\w^\star}\) = -\partial_{\Lambda^{-1}\gamma}\mM_{\Lambda^{-1}}\[ r^{\rm opt} \] (\Lambda^{-1}\gamma) \,,
\end{align}
and assuming that $\mZ_{\rm w}(\gamma, \Lambda)$ is log-concave in $\gamma$, we may invert the Moreau-Yosida regularization, that leads to:
\begin{align}
&	r^{\rm opt}\(\Lambda^{-1} \gamma\) = - \mM_{\Lambda^{-1}}\[ \log\(\td{\mZ}_{\w^\star}\)\(.,\Lambda^{-1}\) \]\(w\) \label{appendix:optimality:reg} \\
	&= - \min_{\Lambda^{-1}\gamma} \( \frac{(w-\Lambda^{-1}\gamma)^2}{2 \Lambda^{-1}} + \log \td{\mZ}_{\w^\star}\(\gamma, \Lambda\) \) = - \min_{\gamma} \( \frac{1}{2}\Lambda w^2 - \gamma w + \log \mZ_{\w^\star}\(\gamma, \Lambda\) \) \,. \nonumber
\end{align}

The last step, is to characterize the variances $V$ and $\Lambda$ involved in \eqref{appendix:optimality:loss} and \eqref{appendix:optimality:reg} that are so far undetermined. To achieve the Bayes-optimal performances, we therefore need to use the variances $V$ and $\Lambda$ solutions of the Bayes-optimal GAMP algorithm \eqref{appendix:amp:glm_}. In the large size limit, these quantities concentrate and are given by the State Evolution (SE) of the GAMP algorithm, that we recall herein.

\paragraph{State evolution of GAMP}
In the large size limit, the expectation of the parameter $V$ and $\Lambda$ over the ground truth $\vec{w}^\star$ and the input data $\mat{X}$ lead to \cite{Barbier2017b}: 
\begin{align}
\EE_{\vec{w}^\star, \mat{X}}\[\ V \] &= \rho_{\w^\star} - q_\bayes\,, &&\EE_{\vec{w}^\star, \mat{X}}\[\ \Lambda \] = \hat{q}_\bayes\,,
\end{align}
where $q_\bayes$ and $\hat{q}_\bayes$ are solutions of the Bayes-optimal set of fixed point equations eq.~\eqref{main:fixed_point_equations_bayes}.

\subsection{Summary and numerical evidences}
Choosing the fine-tuned (potentially non-convex depending on $\mZ_{\out^\star}$ and $\mZ_{\w^\star}$) loss and regularizer
\begin{align}
\begin{aligned}
		l^{\rm opt}\(y,z\) &=- \min_\omega \( \frac{(z-\omega)^2}{2 (\rho_{\w^\star}-q_\bayes)} + \log \mZ_{\out^\star}\(y,\omega,\rho_{\w^\star}-q_\bayes\) \) \spacecase
		r^{\rm opt}\(w\) &= - \min_{\gamma} \( \frac{1}{2} \hat{q}_\bayes w^2 - \gamma w + \log \mZ_{\w^\star}\(\gamma, \hat{q}_\bayes\) \)
	\label{appendix:opt_loss_reg}
\end{aligned}
\end{align}
with $q_\bayes$ and $\hat{q}_\bayes$ are solutions of the Bayes-optimal set of fixed point equations eq.~\eqref{main:fixed_point_equations_bayes}, we showed that ERM can provably match the Bayes-optimal performances.
In particular we illustrated the behaviour of the optimal loss and regularizer $\lambda^{\rm opt}$ and $r^{\rm opt}$ for the model \eqref{main:teacher_sign} in Fig.~\ref{fig:opt_loss_reg_sign_gaussian} of the main text. Note in particular that even though the loss $l^{\rm opt}$ is not convex (but seems quasi-convex), numerical simulations of ERM with (\ref{appendix:opt_loss_reg}) (black dots) presented in Fig.~\ref{fig:opt_loss_reg_sign_gaussian_numerics} show that ERM achieves indeed the Bayes-optimal performances (black line) even at finite dimension. 
\begin{figure}[!htb]
\centering
        \includegraphics[scale=0.5]{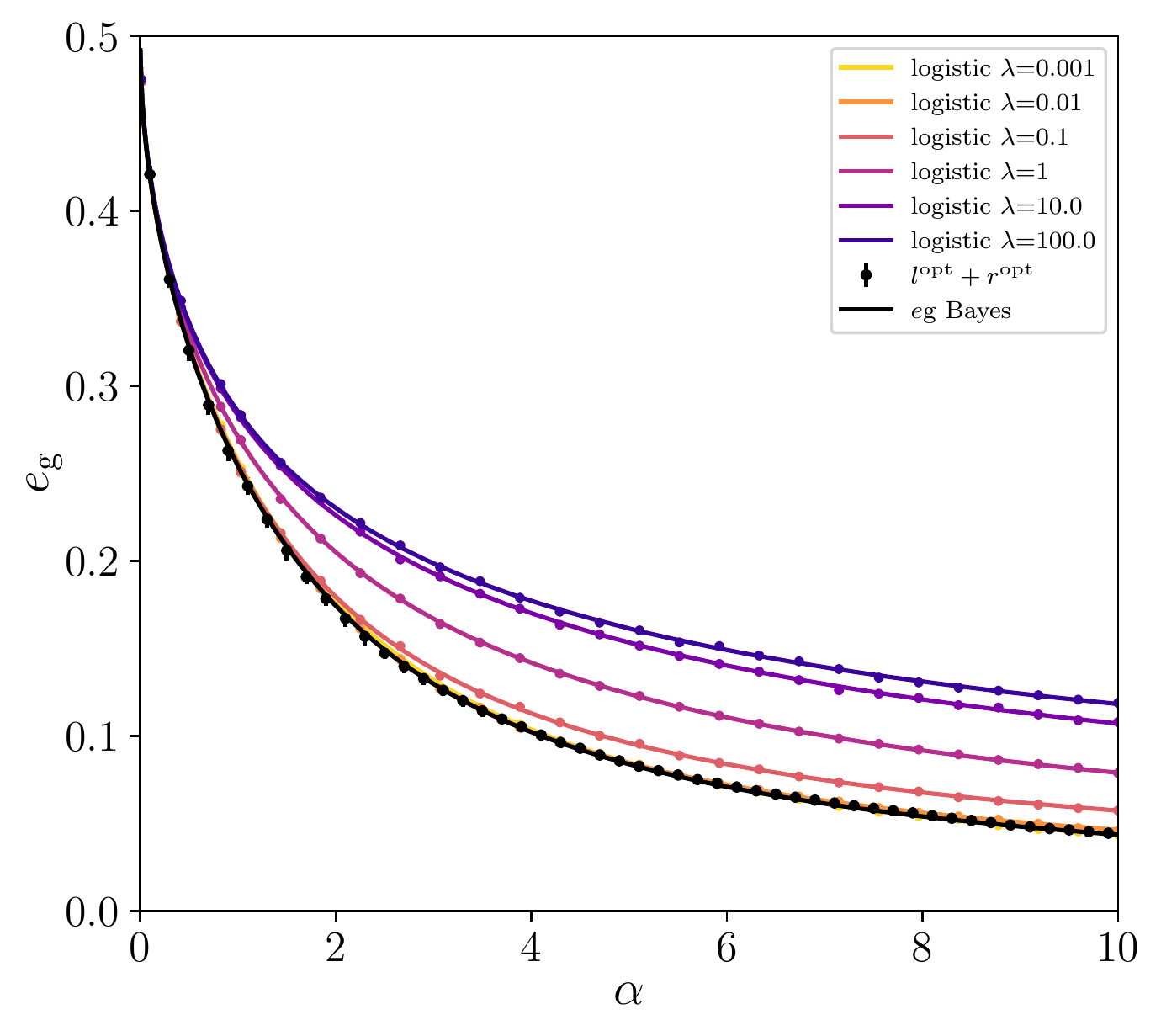}
        \caption{Generalization error obtained by optimization of the optimal loss $l^{\rm opt}$ and $r^{\rm opt}$ for the model \eqref{main:teacher_sign}, compared to $\rL_2$ logistic regression and Bayes-optimal performances. Numerics has been performed with \texttt{scipy.optimize.minimize} with the \texttt{L-BFGS-B} solver for $\ndim=10^3$ and averaged over $n_s=10$ instances. The error bars are barely visible.}
        \label{fig:opt_loss_reg_sign_gaussian_numerics}
\end{figure}